\documentclass[a4paper,fleqn]{cas-dc}
\usepackage[authoryear,longnamesfirst]{natbib}
\usepackage{amsmath,amssymb,amsthm}
\usepackage{graphicx}
\usepackage{booktabs}
\usepackage{xcolor}
\usepackage{url}
\usepackage{algorithm}
\usepackage{algpseudocode}
\usepackage{graphicx}
\usepackage{placeins}
\usepackage{balance}

\usepackage{tikz}
\usetikzlibrary{positioning}
\usepackage{hyperref}

\ExplSyntaxOn
\RenewDocumentCommand \printorcid { }
{
  \seq_if_empty:NF \g_stm_orcid_seq
  {
    \group_begin:
      \tex_let:D \thefootnote \relax \footnotetext
      {
        \raggedright
        \textsc{orcid}(s): \c_space_token
        \seq_use:Nn \g_stm_orcid_seq { ;~ }
      }
    \group_end:
  }
}
\ExplSyntaxOff

\AtBeginDocument{%
\setcounter{topnumber}{4}%
\setcounter{dbltopnumber}{3}%
\setcounter{totalnumber}{6}%
}

\AtBeginDocument{%
\setlength{\dblfloatsep}{16pt plus 2pt minus 2pt}%
\setlength{\dbltextfloatsep}{18pt plus 2pt minus 3pt}%
\setlength{\belowcaptionskip}{8pt}%
}

\makeatletter
\AtBeginDocument{%
\setlength{\@fptop}{0pt}%
\setlength{\@fpsep}{0pt}%
\setlength{\@fpbot}{0pt}%
}
\makeatother

\theoremstyle{plain}
\newtheorem{theorem}{Theorem}
\newtheorem{lemma}[theorem]{Lemma}
\theoremstyle{definition}
\newtheorem{definition}{Definition}
\newtheorem{remark}{Remark}

\begin{document}
\let\WriteBookmarks\relax
\def\floatpagepagefraction{1}
\def\textpagefraction{.001}
\raggedbottom % avoid huge vertical stretches around floats

% Short title / short authors
\shorttitle{Decoupled Prune-and-Reconstruct with Double-Layer Sparsification}
\shortauthors{Gu et al.}

% =============================================================
% Front matter (CAS: NO \begin{frontmatter} environment -- that is
% elsarticle syntax. In cas-dc these elements are placed directly
% inside \begin{document}, before \maketitle.)
% =============================================================

\title[mode=title]{Decoupled Prune-and-Reconstruct: Training on a Compact Core
with Double-Layer Sparsification}

\author[2]{Siqi Gu}
\ead{25521840825@nynu.edu.cn}
\orcidauthor{0009-0008-7645-8081}{Siqi Gu}

% Credit authorship (ESWA / CAS requires these)
\credit{Conceptualization of this study, Methodology, Software, Formal analysis (Appendix A.3--A.4), Writing -- original draft}

\author[1]{Gaozhen Ma}
\ead{u202411602@hust.edu.cn}
\orcidauthor{0009-0004-1230-2046}{Gaozhen Ma}
\credit{Methodology, Validation, Writing -- review and editing}

\author[3]{Liheng Huang}
\ead{u202415012@hust.edu.cn}
\orcidauthor{0009-0009-4024-0109}{Liheng Huang}
\credit{Investigation, Validation}

\author[1]{Weiqi He}
\ead{u202411845@hust.edu.cn}
\orcidauthor{0009-0001-4647-9681}{Weiqi He}
\credit{Investigation, Validation}

% Fifth author
\author[4]{Jiawen Luo}
\ead{u202615725@hust.edu.cn}
\orcidauthor{0009-0004-6542-5293}{Jiawen Luo}
\credit{Writing -- review and editing}

% Corresponding author
\author[1]{Jiaming Wu}%[<options>]
\cormark[1]
\ead{d202480622@hust.edu.cn}
\orcidauthor{0009-0002-1451-7177}{Jiaming Wu}
\credit{Supervision, Project administration, Writing -- review and editing}

% Second corresponding author (last position)
\author[1]{Lixin Zheng}
\cormark[1]
\ead{zhenglx@hust.edu.cn}
\orcidauthor{0009-0004-6829-1016}{Lixin Zheng}
\credit{Supervision, Project administration, Writing -- review and editing}

% Address / affiliation
\affiliation[1]{organization={School of Integrated Circuits, Huazhong University of Science and Technology},
            addressline={1037 Luoyu Road, Hongshan District},
            city={Wuhan},
            postcode={430074},
            state={Hubei},
            country={China}}

\affiliation[2]{organization={School of Physics and Electronic Engineering, Nanyang Normal University},
            addressline={Wolong District},
            city={Nanyang},
            postcode={473061},
            state={Henan},
            country={China}}

\affiliation[3]{organization={School of Architecture and Urban Planning, Huazhong University of Science and Technology},
            addressline={1037 Luoyu Road, Hongshan District},
            city={Wuhan},
            postcode={430074},
            state={Hubei},
            country={China}}

\affiliation[4]{organization={School of Mechanical Science and Engineering, Huazhong University of Science and Technology},
            addressline={1037 Luoyu Road, Hongshan District},
            city={Wuhan},
            postcode={430074},
            state={Hubei},
            country={China}}

\cortext[1]{Corresponding authors: d202480622@hust.edu.cn (J. Wu); zhenglx@hust.edu.cn (L. Zheng).}

\begin{abstract}
On million-node graphs, the memory wall arrives before the compute wall: a GNN that never fits in VRAM cannot deliver the all-node inference its accuracy promises. The standard remedy, pruning, has always been an accuracy tax --- paid once and never refunded. DPR refuses this framing. Its central insight is that \emph{the mask that shrinks training can also scaffold recovery}: one compact node set bounds both computations of graph learning, serving as the encoder's training graph and, at inference, as the anchor from which a sparse structure-aware decoder rebuilds every pruned node (\emph{double-layer sparsification}). The order is inverted --- \emph{prune for efficiency, then reconstruct for accuracy} --- and the tax is repaid: the deployed $30\%$-budget model trains in $1{,}450$\,MiB of peak memory on one $6$\,GB GPU and recovers $96.9\%$ of its dense-reconstruction reference, beating a same-GPU rerun of the edge-sparsification baseline DSpar by $+1.36$\,pp at $50.1\%$ lower peak training memory; against the same-protocol full-batch no-prune encoder the deployed point pays a measured $2.64$\,pp accuracy tax for a $2.8\times$ smaller training footprint. Pruning without reconstruction leaves the encoder at $0.378$ within the same $40$-epoch window (Fig.~\ref{fig:curves}), and all-node serving runs $46\times$ faster than sampling-based GraphSAGE, giving up $\approx$$5.4$\,pp of test accuracy. The framework scales to a $2.45$M-node production graph at $6{,}285$\,MiB of peak training memory; extended variants (long-range context, external text features, a CAN-bus application) are reported as extensions and do not compete with the deployed method, and the $41.8\%$-budget M4 point, scored on a coarsened super-node graph with its own split, is reported separately rather than compared with the original-node baselines. DPR thus turns aggressive node pruning from an accuracy tax into a safe training accelerator for VRAM-limited all-node inference.
\end{abstract}

% Research highlights (ESWA requires 3--5 concise bullets, <= ~85 chars each)
\begin{highlights}
\item DPR-S retains $96.9\%$ of dense-reference accuracy on ogbn-arxiv at $30\%$ core.
\item Double-layer sparsification cuts main and reconstruction graphs.
\item Matched dense and MLP (extended-variant) controls reduce capacity confounds.
\item Training on a $6$\,GB card: $1450$\,MiB VRAM, $3$\,k samples/s, ${\sim}4.8$\,ms/batch CPU--GPU xfer.
\item Products $2.45$M: extended M4 at $6{,}285$\,MiB; DPR-D $16{,}100$ vs $23{,}079$\,MiB no-prune.
\end{highlights}

% Keywords (separated by \sep)
\begin{keywords}
Graph pruning \sep Graph neural networks \sep Reconstruction \sep Node pruning \sep Double-layer sparsification \sep VRAM-limited inference
\end{keywords}

% cas-dc bugfix: \MaketitleBox displays the collected keywords box via
% \hbox_to_wd:nn{\z@}{\box\g_stm_key_box} (cas-common.sty, Abstract env).
% The keywords box is a \vtop of width .25\textwidth = 123.62721pt, so that
% \hbox-to-0pt ALWAYS reports a spurious "Overfull \hbox (123.62721pt)"
% (\hfuzz does not suppress it). Re-wrap the collected box to width 0pt with
% trailing \hss: visually identical (left edge fixed, content extends right,
% vertical height/depth unchanged), but the display box now sets exactly.
\ExplSyntaxOn
\box_new:N \l_dpr_keyfix_box
\tex_global:D \tex_setbox:D \l_dpr_keyfix_box ~ = ~ \box_use:N \g_stm_key_box
\tex_global:D \tex_setbox:D \g_stm_key_box ~ = ~
  \tex_hbox:D to \c_zero_dim { \box_use:N \l_dpr_keyfix_box \tex_hss:D }
\ExplSyntaxOff

% Float-page pinning (review 2026-09-08: the embedded Graphical Abstract page
% was removed from the main PDF -- ESWA collects the GA as a separate file,
% figs/graphical-abstract.pdf, generated by _gen_ga.py). The pinning below is
% kept because it also governs the appendix float pages: top-aligned floats
% with bottom stretch.
\makeatletter
\setlength{\@fptop}{0pt}
\setlength{\@fpsep}{12pt}
\setlength{\@fpbot}{0pt plus 1fil}
\makeatother

\maketitle

% =============================================================
\section{Introduction}
\label{sec:intro}
% =============================================================
Graph Neural Networks (GNNs) \citep{gcn2017,graphsage2017,gat2018} are the de-facto backbone for learning on graph-structured data, but their cost grows with graph size: full-batch message passing over million-node graphs exceeds the memory of a single commodity GPU. \emph{Graph pruning / sparsification} is the standard remedy: shrink the graph so that training fits. Yet there is a hard trade-off: a standard GNN trained on an \emph{aggressively} pruned graph collapses --- pruned nodes lose their input features, and pruning destroys the connectivity that message passing relies on. On ogbn-arxiv, a plain GNN on such a pruned graph falls far below the fair no-prune baseline (Table~\ref{tab:main}, Fig.~\ref{fig:curves}), and this \emph{zero-fill, no-reconstruction floor} is not a classifier artifact, since the classifier still scores the full node set. Nearly all prior work asks how to prune so that accuracy does not drop, which pushes methods toward conservative selection that leaves most of the graph intact. Our central observation is the opposite: the collapse is not a fate but a \emph{deferred repair problem} --- what aggressive pruning destroys can be reconstructed at inference time, and the repair itself can be made cheap.

DPR therefore inverts the usual order: \emph{prune for efficiency, then reconstruct for accuracy}. A fixed-magnitude core mask compresses the graph into a compact core (a small minority of the nodes, an even smaller share of the edges); a GraphSAGE encoder trains only on this core, at a small fraction of the full-graph aggregation cost; and a sparse structure-aware decoder (DPR-S) reconstructs the pruned nodes' features from the original topology, restricted to the core's 1-hop neighborhood. The design's signature property, formalized in \S\ref{sec:double}, is \emph{double-layer sparsification}: the same core mask that shrinks the training graph also shrinks the reconstruction graph, so the second-layer sparsification is inherited from the pruning already performed rather than paid for separately. The resulting resource--accuracy tradeoff is quantifiable and measured end-to-end on one $6$\,GB GPU: at a $30\%$ node budget the deployed method trains in $1{,}450$\,MiB of peak memory and recovers $96.9\%$ of the dense-reconstruction reference, exceeding a same-card rerun of the edge sparsifier DSpar by $+1.36$\,pp at $50.1\%$ lower peak training memory (Fig.~\ref{fig:teaser}; Table~\ref{tab:main}) --- while pruning without reconstruction leaves the same encoder at $0.378$ over the same $40$-epoch window.

Throughout, \emph{the submitted method} means \textbf{DPR-S}; \textbf{DPR-D} is the dense reconstruction \emph{reference} (the evaluation yardstick); \textbf{DPR-S-M4} is an \emph{extension} that adds a precomputed long-range context row to the same decoder; and every other reported variant (the learnable pruner, the positional encoding, the dual-pruner research variant) is an \emph{optional variant studied in a controlled ablation}, not part of the deployed method. Seeds, hardware, and per-number provenance are tabulated in Appendix~\ref{app:data}.

% Figure roles: Fig. 1 = system view (page-3 top, above Sec. Method),
% Fig. 2 = two-stage mechanism (Method), Fig. 3 = efficiency teaser (Experiments).
% Declared here (inside the p.2 source flow) because a figure* in twocolumn
% mode can land no earlier than the page AFTER its declaration point.
\begin{figure*}[!t]
  \centering
  \includegraphics[width=\textwidth]{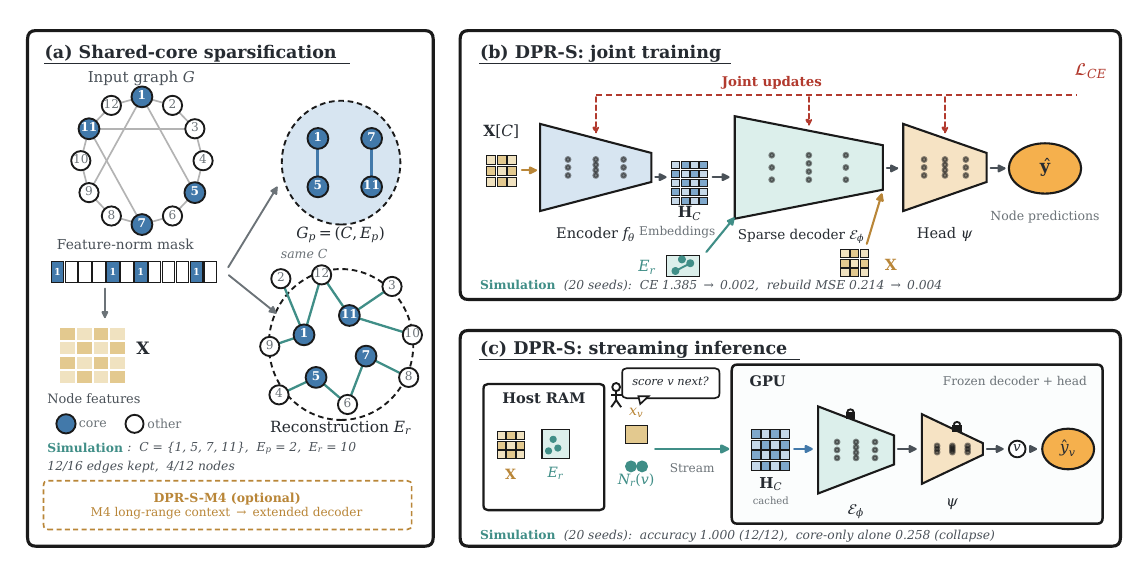}
  \caption{System overview. (a)~One feature-norm mask $C$ induces both
  sparsified graphs: the encoder graph $G_p=(C,E_p)$ keeps only core-internal
  edges, and the reconstruction graph $G_r=(V,E_r)$ with
  $E_r=E_{\text{core-1hop}}$ adds the core-to-1-hop links needed to rebuild
  deleted nodes. (b)~DPR-S joint training: the encoder $f_\theta$ runs on the
  core only ($\mathbf{X}[C]$ with $E_p$; a $30\%$-node, $8\%$-edge core on
  which a 3-layer GraphSAGE aggregates $\approx$$1/12$ of full); the sparse
  decoder $\mathcal{E}_\phi$ reconstructs pruned-node
  features within the core's 1-hop from $\mathbf{X}$ and $E_r$, and the
  reconstructed features feed the head; the CE loss
  over training nodes jointly updates $f_\theta,\mathcal{E}_\phi,\psi$ ---
  ``decoupled'' refers to a division of functional roles within one
  pipeline (\S\ref{sec:exp-abl}). The dashed box in (a) marks the optional extension
  DPR-S-M4, which feeds an $\mathrm{m4}_v$ long-range context row to an
  extended decoder. (c)~Streaming inference: $\mathbf{H}_C$ stays
  GPU-resident (cached after training), $\mathbf{X}$ and $G_r$ stay in host
  RAM, and each pruned node $v$ is scored by streaming its row
  $\mathbf{x}_v$ with its core-1-hop neighbours $N_r(v)$ through the frozen
  decoder and head to produce $\hat{y}_v$; the encoder is never
  re-run on the full graph.}
  \label{fig:architecture}
\end{figure*}

\noindent\textbf{Storyline.} The paper follows a single line --- \emph{aggressive pruning} $\rightarrow$ \emph{shared-core double-layer sparsification} $\rightarrow$ \emph{accurate reconstruction} $\rightarrow$ \emph{memory-efficient inference} --- and the main contributions are:
\begin{itemize}
  \item \textbf{Core method: DPR-S --- sparse structure-aware reconstruction under double-layer sparsification.} One mask sparsifies both layers: the core that compresses the training graph also bounds the reconstruction graph (\S\ref{sec:double}), so the second-layer sparsification needs no extra machinery and the reconstructing decoder runs $1.16$--$1.30\times$ faster. At the pure core-1hop operating point, DPR-S recovers $96.9\%$ of the dense reference while cutting reconstruction edges under 1-hop core localisation ($31.2\%$ on ogbn-arxiv, $39$--$45\%$ on the products probes; \S\ref{sec:exp-main}) and trains $2.7\times$ faster.
  \item \textbf{A decoupled prune-and-reconstruct framework with a pluggable core selector.} Its value is \emph{largely independent of core-selection sophistication}: a fixed-magnitude mask proves essentially equivalent to a learnable hyperbolic one (\S\ref{sec:exp-abl}), so the method inherits the robustness of a rule rather than the fragility of a learned selector.
  \item \textbf{Extended variants and application-scale validation (extensions, not part of the deployed method).} DPR-S-M4 reaches $97.8\%$ recovery on arxiv and serves as the deployment configuration of the Car-Hacking stress study; the framework further scales to a production ogbn-products graph whose coarsened training graph is trainable on a single commodity GPU via the streaming design of Algorithm~\ref{alg:dprs-final}. Every extended configuration is flagged where it appears.
\end{itemize}

\begin{remark}[Scope.]
Claims are separated by training regime. On graphs where the no-prune
encoder is trainable (arxiv-scale), we \emph{approach} same-backbone
full-graph accuracy, with a measured $2.64$\,pp gap rather than parity; on
VRAM-constrained large graphs, where no runnable no-prune baseline exists by
construction, the framework provides one practical route to full-graph-level
accuracy. The deployed method is therefore presented as a resource--accuracy
tradeoff --- a large VRAM reduction bought with a small, per-dataset quantified
gap --- rather than as accuracy parity. Efficiency numbers are reported as controlled, same-card
comparisons against rerun baselines (\S\ref{sec:related},
\S\ref{app:eff-attn}), and the deployed method places no architectural
constraint on the encoder, so the reconstruction layer --- not a
particular backbone --- carries the claim. See \S\ref{sec:conclusion}.
\end{remark}

% =============================================================
\section{Related work}
\label{sec:related}
% =============================================================
% Restored as a standalone section (ESWA convention: Sec. 2 "Related works"
% immediately after the Introduction, cf. published samples e.g.
% ESWA 223 (2023) 119872). \label{sec:related} kept so all
% \ref{sec:related} pointers resolve here.

\noindent\textbf{Edge-level sparsification.}
Industrial graph platforms routinely operate on graphs with millions to
billions of nodes (AliGraph~\citep{aligraph2019}, PinSage~\citep{pinsage2018}),
where exhaustive all-node scoring under a single-GPU memory ceiling is exactly
the regime that motivates pruning --- and in such deployments, edge-level
relief is often the first lever attempted, since edges dominate the
aggregation cost of message passing \citep{gilmer2017mpnn}. One line of work removes \emph{edges}
while keeping all nodes: DropEdge~\citep{dropedge2019} drops random edges as a
regularizer, spectral graph coarsening~\citep{coarsen2020} merges nodes into
super-nodes for cheaper training with an explicit eigenvalue error bound on
what the restriction forgets, and SGS-GNN~\citep{sgsgnn2025} learns an
edge-sampling distribution to shrink the graph. The closest method is
\textbf{DSpar}~\citep{dspar2023}, a degree-based edge sparsifier: it deletes
$27$--$95\%$ of edges ($26.7\%$ on ogbn-arxiv), retains \emph{every} node, and
discards the removed structure with no reconstruction, targeting faster GNN
\emph{training}. None of these methods shrinks the dominant $O(N^2)$
node-count term of an attention backbone, because the node set is untouched;
our node-level coarsening acts on that term directly (\S\ref{app:eff-attn}),
so the two cost regimes are complementary rather than competing.

\noindent\textbf{Node-level pruning without recovery.}
A second line prunes or subsamples \emph{nodes}. MVP~\citep{mvp2025} and
NCP~\citep{ncp2024} prune nodes from multiple views and use
reconstruction-like signals only to \emph{score} which nodes to discard; the
pruned nodes are then permanently dropped and never rebuilt at inference.
The distinction matters: in these pipelines the reconstruction objective is
an optimization device that shapes the pruning mask, and once the mask is
fixed the signal is discarded --- whereas in DPR the mask's own induced
structure is what makes inference-time recovery possible. SubInfer~\citep{subinfer2024}
trains and infers on node neighborhoods rather than the full graph, sharing
our ``small subgraph at inference'' spirit but without reconstructing the
features of explicitly pruned nodes. PruneGNN~\citep{prunegnn2024} compresses
the graph or its activations but stops at the main graph. Decomposition- and
expander-based accelerators are related in spirit: SDGNN~\citep{sdgnn2024}
approximates a trained GNN with a sparse surrogate for efficient inference,
FIT-GNN~\citep{fitgnn2025} trains on coarsened subgraphs, and
ExpanderGNN~\citep{egp2022} propagates messages over dense synthetic expander
substrates --- their ``expander'' is a propagation substrate, whereas our
decoder rebuilds pruned-node features from the original topology. All of
these leave pruning a \emph{lossy discard}: no method in this family recovers
pruned-node representations at inference, and none sparsifies a recovery
graph. The price of that discard is not hypothetical: naive subgraph
training on a pruned graph collapses (\S\ref{sec:exp-main}), which is
precisely the failure mode this family absorbs as permanent accuracy loss.

\noindent\textbf{Condensation and coarsening.}
Graph condensation compresses a large graph into a small \emph{synthetic} one
(GCond~\citep{jin2021gcond}) so that a model
can be trained cheaply; AGCond~\citep{agcond2026} adapts the
condensation ratio with structure-aware initialization. Crucially, the
condensed model is still \emph{evaluated on the full original graph}, so
condensation targets \emph{training} cost and leaves the deployment graph
intact; its bi-level optimization also does not scale to the million-node
graphs we target, since each inner loop back-propagates through a model
refreshed against the full original graph. A condensed core could in
principle replace our selector, but it neither recovers pruned nodes nor
sparsifies a reconstruction graph, so combining the two is a natural
extension. Coarformer~\citep{coarformer2022} runs a Transformer over a
coarsened graph without pruning or reconstructing nodes; structural position
encodings~\citep{psgnn2024} are orthogonal --- our PE is a pruning
by-product that the main method drops (App.~\ref{app:pe-control}).

\noindent\textbf{Feature recovery from structure.} Our premise --- that
pruned-node features can be rebuilt from the surviving topology, a premise
with classical roots in structure-only embeddings
\citep{perozzi2014deepwalk,grover2016node2vec} --- is
supported by~\citep{wang2023recover} (ICML 2023), which establishes that node
features are recoverable from graph connectivity alone, up to an explicit
error bound tying the reconstruction quality to the homophily of the
surrounding topology. That line studies inherently missing features; our
setting adds the \emph{deliberate} pruning of nodes and, crucially, the
sparsification of the reconstruction graph itself, both absent there. The
practical consequence of sparsifying the recovery graph is that recovery
quality becomes controllable: the same knob $c_r$ that sets the training
budget also bounds which neighborhoods the decoder may consult
(\S\ref{sec:double}), a coupling no condensation or missing-feature method
exhibits.

\noindent\textbf{Positioning.}
We instead prune \emph{nodes into a compact core} and \emph{reconstruct}
pruned-node features from the original topology at inference. Within the
literature we surveyed, DPR is the first framework that combines (i) the
train-prune / infer-reconstruct \emph{decoupling} --- a training-time
accelerator and an inference-time restorer sharing one trained model, so the
accuracy repair moves out of the pruning stage entirely; (ii)
\emph{recoverable} pruning --- pruned nodes are rebuilt rather than
discarded; and (iii) \emph{double-layer sparsification} (\S\ref{sec:double})
--- the reconstruction graph is itself sparsified by the same core mask,
whereas methods that reconstruct (EGP~\citep{egp2022},
SDGNN~\citep{sdgnn2024}) diffuse over the full or compressed topology and
methods that sparsify (DSpar~\citep{dspar2023}, PruneGNN~\citep{prunegnn2024})
stop at the main graph. Naive subgraph training on the pruned graph collapses
(\S\ref{sec:exp-main}), which is why recovery, not pruning alone, is the
load-bearing design.

% =============================================================
\section{Method}
\label{sec:method}
% =============================================================
We formalize the decoupled pipeline, then present the two stages, the
double-layer sparsification, and why reconstruction recovers accuracy;
Fig.~\ref{fig:architecture} gives the system view and the stage-by-stage
walkthrough. \emph{Reading guide (the method is one short path).} The
deployed method is a single pipeline --- fixed-magnitude mask $\to$ GNN
encoder on the compact core $\to$ sparse core-1hop decoder (DPR-S) ---
carrying one insight: \emph{prune hard at training time, recover at
inference time, and let the same mask sparsify both graphs}. The learnable
pruner, the positional encoding, and the long-range extension DPR-S-M4 are
\emph{optional variants}, each examined in one controlled ablation
(\S\ref{sec:exp-abl}); none is needed for the headline results.

\subsection{The Core Insight: One Node Set Bounds Both Layers}
\label{sec:insight}
Before the formal setup, we state the insight that organizes everything that
follows. DPR uses the \emph{one} core mask $C\subset V$ (formalized in
\S\ref{sec:setup}) to bound two computations at once. \emph{Encoding}: the
encoder trains on $G_p=(C,E_p)$, the graph sparsified to $C$.
\emph{Reconstruction}: the decoder rebuilds each pruned node only within the
core's 1-hop neighborhood, $E_r=E_{\text{core-1hop}}$ with
$E_p\subseteq E_r\subseteq E$. One mask, two sparsification layers; the
second layer is \emph{defined relative to the first} rather than cut
independently, so a single budget knob $cr$ sets both costs and the recovery
rate. \S\ref{sec:double} develops this into the paper's central claim; the
two stages that realize it follow.

\begin{figure*}[!htbp]
  \centering
  \includegraphics[width=0.95\textwidth]{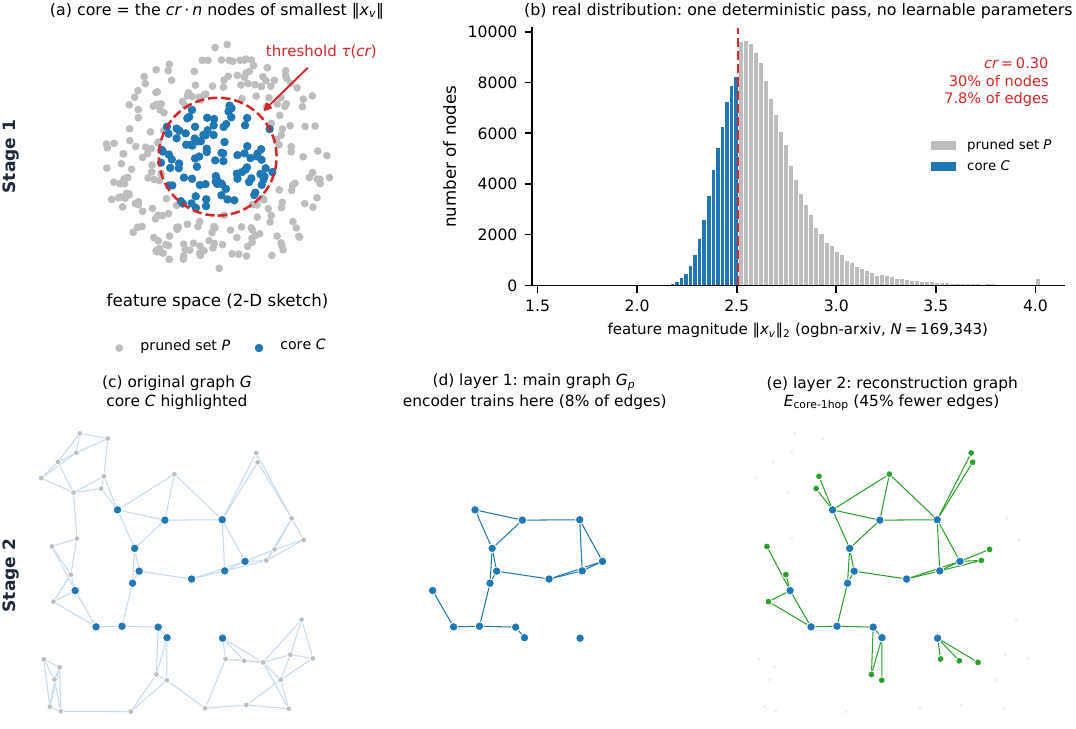}
  \caption{The two pipeline stages. \textbf{Stage 1 (a--b): core selection by
  feature magnitude.} (a) A $cr\cdot n$ fraction of nodes with the smallest
  Euclidean feature norm $\|\mathbf{x}_v\|$ is chosen as the core $C$ (2-D
  sketch); the dashed circle is the deterministic threshold $\tau(cr)$.
  (b) On a real feature distribution the histogram of $\|\mathbf{x}_v\|$
  separates cleanly into the $cr{=}0.30$ core and the pruned set $P$; the cut
  is a single deterministic pass with zero learnable parameters. We retain the
  learnable hyperbolic variant only as an ablation (\S\ref{sec:exp-abl}).
  \textbf{Stage 2 (c--e): double-layer sparsification, induced by the same
  core mask $C$.} (c) The original graph $G$ with core $C$ highlighted in
  blue; (d) the main graph $G_p$ (layer 1) on which the encoder trains,
  keeping the core-internal edges ($\approx$$8\%$ of $E$); (e) the
  reconstruction graph $E_{\text{core-1hop}}$ (layer 2) used by the decoder,
  which adds the core-to-1-hop edges needed to rebuild deleted nodes; it has
  $39$--$45\%$ fewer edges than the full reconstruction graph on the
  products probes (\S\ref{sec:exp-main}). Both layers are
  driven by the same mask $C$, so the second sparsification is endogenous to
  the decoder rather than a post-hoc cut.}
  \label{fig:stage12}
\end{figure*}

\subsection{Problem Setup}
\label{sec:setup}
Let $G=(V,E,\mathbf{X})$ be an attributed graph with $n=|V|$ nodes,
$m=|E|$ edges, and node features $\mathbf{X}\in\mathbb{R}^{n\times d}$. A
\emph{core ratio} $cr\in(0,1]$ selects a core mask $C\subset V$ with
$|C|=cr\cdot n$. The \emph{pruned set} is $P=V\setminus C$. The \emph{pruned
graph} $G_p=(C,E_p)$ keeps only edges whose both endpoints lie in $C$, so
$|E_p|\ll|E|$. The \emph{reconstruction graph} $G_r=(V,E_r)$ is the topology the
decoder diffuses over at inference; we consider two cases: $E_r=E$ (dense upper
bound / control) and $E_r=E_{\text{core-1hop}}$ (the compressed \emph{final}
method). A GNN encoder $f_\theta$ computes core representations
$\mathbf{H}_C=f_\theta(\mathbf{X}[C];E_p)$; the decoder
$\mathcal{E}_\phi$ reconstructs $\hat{\mathbf{X}}_P$ for $P$ and fuses
$\hat{\mathbf{X}}_C$, and a head $\mathrm{Head}_\psi$ predicts from
$\hat{\mathbf{X}}$.

Three properties of this setup carry the design. First, the two sparsified
graphs are \emph{induced by one mask}: $E_p$ is the restriction of $E$ to
$C$, and $E_{\text{core-1hop}}$ is the set of edges that cross from $C$ into
its immediate neighborhood, so neither graph requires a second selection
procedure (\S\ref{sec:double}). Second, the split is \emph{asymmetric by
construction}: training touches only $(\mathbf{X}[C],E_p)$, so the encoder's
cost scales with $cr$, while inference touches $(\mathbf{X},E_r)$ row by row
and never re-runs the encoder on the full graph. Third, the head is
\emph{shared across regimes}: the same $\mathrm{Head}_\psi$ scores core and
reconstructed nodes, so accuracy comparisons against the no-prune encoder
are matched by construction rather than by recalibration. Throughout, we
write $c_r$ for the deployment knob and reserve ``DPR-S'' for the variant
with $E_r=E_{\text{core-1hop}}$; the dense-reconstruction control is DPR-D.

\subsection{Stage 1: Core Selection (Prune)}
\label{sec:prune}
A \textbf{core selector} $\mathcal{S}(G;cr)\to C$ is a pluggable function.
\emph{Motivation for keeping it trivial.} In this framework accuracy is
restored in Stage~2, not at selection time: the decoder, not the mask,
carries the recovery signal. It follows that the framework's value should be
\emph{largely independent of selector sophistication} --- and this is a testable
design principle, not an assumption. \S\ref{sec:exp-abl} verifies it
directly: a fixed-magnitude mask proves essentially equivalent to a learnable
hyperbolic selector; in every configuration we tested, the learnable selector
shows no consistent advantage that would justify its overhead, so a
fixed-rule mask suffices. We therefore deploy the
simplest selector by default and demote the learned one to an ablation:
\begin{itemize}
  \item \textbf{Fixed-magnitude mask} (default): keep the $cr\cdot n$ nodes with
    the \emph{smallest} Euclidean feature norm as the core
    (Fig.~\ref{fig:stage12}a--b). A single
    deterministic pass, zero learnable parameters, excludes the rest (including
    high-degree hubs) to the reconstruct stage --- a deliberately low-overhead,
    reproducible default rule rather than a claim of universal optimality.
  \item \textbf{Learnable hyperbolic pruner} (HSP): a learnable
    magnitude/angle threshold in the Poincar\'e disk selects the core. More
    expressive, but it matches rather than beats the fixed mask
    ($0.7111\pm0.0053$ vs.\ $0.7068\pm0.0022$ at $cr{=}0.3$, 5-seed,
    \S\ref{sec:exp-abl}). We therefore use HSP for the dense reconstruction
    reference (DPR-D) and the fixed mask for the final method.
\end{itemize}

\paragraph{Division of labor (load-bearing claim).} The pruner's
``intelligence'' is \emph{decoupled} from final accuracy: its only job is to
provide a compact, cheaply trainable core; accuracy recovery is carried by
Stage~2's reconstruction rather than by the pruner. The key property we rely
on is simply: \emph{prune aggressively to a small core, then let Stage~2 fix
it.}

\paragraph{Position encoding (dropped in the main method).} The hyperbolic
pruner emits a 16-D structural PE as a pruning by-product. The main method
drops it (\texttt{pe\_dim=0}): on GraphSAGE the PE is not consumed, and the
fixed mask matches the learnable pruner. A controlled study
(App.~\ref{app:pe-control}) confirms the PE's contribution is gated by the
pruning direction, not a generic gain; it is not a load-bearing component.

\subsection{Stage 2: Structure-Aware Reconstruction}
\label{sec:recon}
The reconstruction module is the heart of the method. We describe the proposed
\textbf{sparse decoder DPR-S} and the \textbf{dense reconstruction reference DPR-D}.

\textbf{Proposed decoder -- DPR-S (sparse structure-aware decoder).}
DPR-S is not a post-hoc add-on: encoder and decoder form a \emph{single model}
trained jointly end-to-end on the same cross-entropy
(Algorithm~\ref{alg:dprs-final}); ``decoupled'' refers to functional roles
(prune at train time, reconstruct at inference) within one jointly trained pipeline.
The decoder performs \emph{sparse structure-aware diffusion over a bipartite
core-to-pruned topology}: each pruned node
$v\notin C$ receives gated messages only from its core 1-hop neighbors $N_r(v)\subseteq C$
(message complexity $O(|E_r|d)$, Theorem~\ref{thm:complexity}).
Three mechanisms constitute the decoder:
\textbf{(P1)~Core-confidence residual fusion.} Core nodes retain SAGE-dominant
representations via a learned confidence gate $c=\sigma(\mathrm{Conf}(\mathbf{H}_C))$,
$\hat{\mathbf{H}}[C]\leftarrow c\odot\mathbf{H}_C+(1-c)\odot\hat{\mathbf{H}}[C]$.
\textbf{(P2)~Gated message propagation.} Pruned nodes receive structure-aware
diffusion from their core neighbors through $L$ gated layers:
$\mathrm{AGG}^{(l)}=\mathrm{scatter\_add}(\mathrm{Msg}_l(\hat{\mathbf{H}}^{(l-1)}[\mathrm{src}]),\mathrm{dst})$,
$G^{(l)}=\sigma(\mathrm{Gate}_l([\hat{\mathbf{H}}^{(l-1)};\mathrm{AGG}^{(l)}]))$,
$\hat{\mathbf{H}}^{(l)}=\mathrm{ReLU}(\mathrm{Norm}_l(\hat{\mathbf{H}}^{(l-1)}+G^{(l)}\odot\mathrm{Upd}_l(\mathrm{AGG}^{(l)})))$.
\textbf{(P3)~Joint end-to-end training.} Encoder, decoder, and head are optimized
together; freezing either yields negative gain (\S\ref{sec:exp-abl}).
The decoder's seeding matches DPR-D: $\hat{\mathbf{H}}^{(0)}[C]=\mathbf{H}_C$,
$\hat{\mathbf{H}}^{(0)}[\bar C]=\mathrm{InitProj}(\mathbf{X}[\bar C])$.
The head $\hat y=\mathrm{Head}_\psi(\hat{\mathbf{H}})$ is minimized by\\
$\min_{\theta,\phi,\psi}\mathcal{L}_{\mathrm{CE}}(\hat y_{\text{train}},y_{\text{train}})$.
\emph{The key point: the decoder propagates only through the core's 1-hop
neighborhood, so the reconstruction graph is itself sparse
(Fig.~\ref{fig:stage12}e)} --- the endogenous-sparsification argument is
developed in \S\ref{sec:double}.
The DPR-S forward pass is given as Algorithm~\ref{alg:dprs-final}.
A research variant with dual pruners and positional encoding is studied in the
ablation (App.~\ref{app:pseudo}, Algorithm~\ref{alg:dprs-var}).

\textbf{Dense reconstruction reference -- DPR-D (GGNN Expander).} To calibrate the sparse
decoder against the dense reconstruction reference, under the same pruning we run a dense gated GGNN that
diffuses and reconstructs \emph{all} nodes over the (compressed) original full
topology $E_r$:
\begin{equation}
\begin{aligned}
\mathrm{AGG}^{(l)} &= \mathrm{Dropout}\!\left(\mathrm{scatter\_add}\!\big(
  \mathrm{Msg}_l(\hat{\mathbf{H}}^{(l-1)}[\mathrm{src}]),\ \mathrm{dst}\big)\right),\\
G^{(l)} &= \sigma\!\left(\mathrm{Gate}_l\big([\hat{\mathbf{H}}^{(l-1)};\
  \mathrm{AGG}^{(l)}]\big)\right),\\
\hat{\mathbf{H}}^{(l)} &= \mathrm{ReLU}\!\left(\mathrm{Norm}_l\!\big(
  \hat{\mathbf{H}}^{(l-1)} + G^{(l)}\odot
  \mathrm{Upd}_l(\mathrm{AGG}^{(l)})\big)\right).
\end{aligned}
\end{equation}
Its seeding and core-confidence residual fusion are identical to DPR-S's
($\hat{\mathbf{H}}^{(0)}[C]=\mathbf{H}_C$,
$\hat{\mathbf{H}}^{(0)}[\bar C]=\mathrm{InitProj}(\mathbf{X}[\bar C])$;
$c=\sigma(\mathrm{Conf}(\mathbf{H}_C))$,
$\hat{\mathbf{H}}[C]\leftarrow c\odot\mathbf{H}_C+(1-c)\odot\hat{\mathbf{H}}[C]$).
The head acts on $\hat{\mathbf{H}}=\mathrm{FinalNorm}(\hat{\mathbf{H}}^{(L)})$:
$\hat y=\mathrm{Head}_\psi(\hat{\mathbf{H}})$, minimized by
$\min_{\theta,\phi,\psi}\mathcal{L}_{\mathrm{CE}}(\hat y_{\text{train}},y_{\text{train}})$.
DPR-D reaches \textbf{$0.7111\pm0.0053$} (arxiv, 5-seed, h256, HSP) as the dense reference.
DPR-S (the final method) runs at the \emph{core-1hop} (pure
double-layer) operating point, giving \textbf{$0.6892\pm0.0020$} (5-seed, $\approx$$96.9\%$ of
DPR-D); the same GGNN
decoder at core-1hop reaches $0.6885\pm0.0013$ (5-seed, $96.8\%$ of DPR-D full
$0.7111$, App.~\ref{app:data}). The extended M4-gated variant DPR-S-M4 reaches
\textbf{$0.6953\pm0.0019$} (arxiv, 5-seed, h128, near-full reconstruction,
$97.8\%$ of DPR-D).
The $\sim$$1$pp gap buys DPR-S local reconstruction and half the parameters.

\subsection{Double-Layer Sparsification}
\label{sec:double}
Standard pruning stops at $G_p$. Our reconstruction module DPR-S attends only to
the core's 1-hop neighborhood, so the \emph{reconstruction graph} is itself
sparse: \emph{both} the main graph and the reconstruction graph are sparsified.
Concretely, setting $E_r=E_{\text{core-1hop}}$ (edges within 1 hop of $C$,
including the deleted nodes it must rebuild) removes $\approx$$39$--$45\%$ of
reconstruction edges versus the full graph on the products probes
(\S\ref{sec:exp-main}) while retaining a $\approx$$96.8\%$
recovery rate \emph{for the gated-GGNN core-1hop topology control operating point ($0.6885\pm0.0013$,
5-seed; DPR-S core-1hop $0.6892\pm0.0020$, 5-seed)} (\S\ref{sec:exp-eff}). The
dense variant DPR-D runs on $E_r=E$, retained as
the control / dense reconstruction reference.

\paragraph{Cost accounting.} The two layers amortize differently, and the
asymmetry is the point. The \emph{encoder layer} is priced once per training
run: its cost scales with the core ($|C|$ nodes, $|E_p|$ edges), so a
$30\%$-node core reduces the per-aggregation message flow (encoder on $E_p$ $+$
decoder on $E_r$) to $\approx$$25\%$ of a full-graph aggregation (Table~\ref{tab:edge-audit}). The \emph{decoder layer} is
priced once per deployment: $\mathbf{H}_C$ is computed a single time after
training and stays GPU-resident, so the recurring inference cost is the
streamed decoder pass over $E_{\text{core-1hop}}$ --- no encoder re-run, no
full-graph pass, and memory that scales with the 1-hop shell of the core
rather than with $N$. This is why the compression ratio $c_r$ acts as one
knob on two ledgers: lowering it discounts the training bill linearly
while widening, not narrowing, the inference-time savings
(\S\ref{sec:exp-eff}).

\paragraph{Endogenous sparsification.}
The decisive point is \emph{why} the reconstruction graph is sparse. In DPR-S the
decoder's message passing is \emph{defined} to reach only core-1hop
neighbors; there is \emph{no dense reconstruction graph that is later cut}.
Sparsity is part of the decoder's \emph{definition} rather than a post-processing step: a
deleted node with at least one core neighbor receives gated messages from
$N_r(v)$, and a node with \emph{no} core neighbor simply keeps its projected
own feature (no messages arrive) --- the fallback arises from graph
connectivity rather than from a routing rule (Algorithm~\ref{alg:dprs-final},
Table~\ref{tab:coverage}); both paths are what the architecture computes, not
a separate sparsifier applied afterwards.
This is fundamentally different from a \emph{post-hoc} recipe --- take a dense
reconstruction graph (e.g.\ DPR-D over the full topology $E$) and then apply an
\emph{independent} sparsifier (degree pruning, DropEdge, $k$-NN) to shrink it. A
post-hoc cut (a)~introduces a \emph{second}, decoupled sparsification mechanism
with its own hyper-parameters; (b)~is blind to the prune decision, so it may keep
edges the reconstructor does not need and drop edges it does; and (c)~can degrade
recovery because the cut is agnostic to which nodes were pruned. By contrast, our
two sparsification layers are driven by the \emph{same} core mask $C$: the main
graph is pruned to $C$, and the reconstruction graph's 1-hop is defined
\emph{relative to that same $C$}. The two layers are therefore two facets of one
core-selection decision rather than two composed sparsifiers --- which is exactly what
makes the sparse reconstruction graph an \emph{endogenous property of the
method} rather than an afterthought optimization. The payoff is not merely fewer
edges: because the reconstructor never materializes the dense graph, the
efficiency gain is structural and grows with graph size, whereas a post-hoc cut
still pays the cost of building (and then discarding) the dense graph. To our
knowledge, no prior work sparsifies the reconstruction graph \emph{endogenously}
in this sense; methods that reconstruct (EGP, SDGNN, expander propagation) do so
over the full / compressed topology, and methods that sparsify (DSpar, PruneGNN)
stop at the main graph.

\noindent\textbf{Why reconstruction recovers accuracy, in brief.} (i)~Two-stage
approximation: cheap core encoding plus light fixed-$L$ feature diffusion, with no
deep full-graph backprop; (ii)~spectral: the decoder re-introduces the
low-frequency connectivity signal that pruning removed; (iii)~fragmentation
tolerance: no intra-core message passing is needed ($57.6\%$ of core nodes are
degree-0 at $cr{=}0.3$), and an explicit fallback keeps a projection of the own
raw feature for pruned nodes with no core neighbor (Table~\ref{tab:coverage});
(iv)~ceiling: the sparse decoder is bounded by the dense variant DPR-D, not by
plain SAGE (Theorem~\ref{thm:ceiling}); measured tax $\approx$$2.2\%$ relative
in the tightest (near-full, DPR-S-M4) configuration.
The full mechanism discussion and Fig.~\ref{fig:mechanism} are in
Appendix~\ref{app:why}.

% =============================================================

% [Theoretical Analysis + Accuracy Tax moved to the appendix (page budget)]
\medskip
\noindent\textbf{Theoretical analysis in brief.} Structure-aware recovery works because
the recoverable signal concentrates in the low-frequency spectrum of the pruned core
graph: a conditional bound keeps the accuracy tax small when the reconstruction signal
is low-frequency, and the empirical tax decomposition on arxiv and products follows the
same ordering. The formal statements, proofs, and the full tax analysis are deferred to
Appendix~\ref{app:proof-thm3} and Appendix~\ref{sec:tax}.

\section{Experiments}
\label{sec:exp}
% =============================================================

The experiments are organized as a question--evidence map
(Table~\ref{tab:rq-map}): each block answers one stated question about the
method, so that every claim can be checked one-to-one against the evidence.
Efficiency is summarized in Fig.~\ref{fig:teaser}.

% Figure roles: Fig. 3 = efficiency teaser (this section).
\begin{figure*}[!t]
  \centering
  \includegraphics[width=\textwidth]{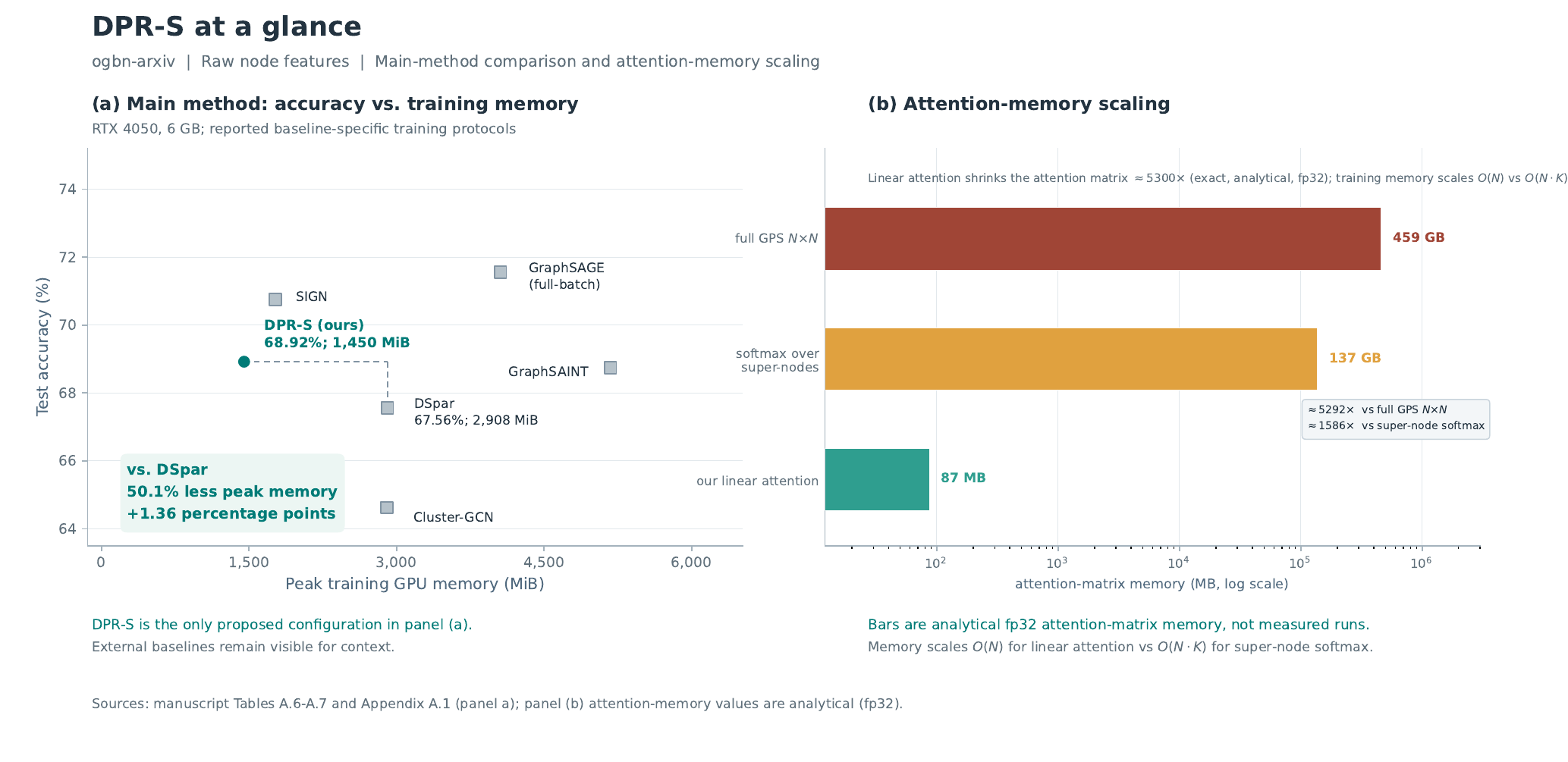}
  \caption{DPR-S efficiency at a glance (ogbn-arxiv;
  raw node features unless stated).
  (a)~Accuracy vs.\ peak training GPU memory on the same RTX\,4050 card:
  external baselines (gray; same-card rerun protocols,
  Table~\ref{tab:external}) cluster in the high-memory region,
  while DPR-S (teal) holds a comparable accuracy band at $1{,}450$\,MiB ---
  $50.1\%$ less peak memory than the strongest same-card pruning baseline
  DSpar at $+1.36$\,pp accuracy.
  (b)~Analytical attention-matrix memory at arxiv scale (fp32): full-graph
  GPS attention~\citep{rampasek2022gps} ($N{\times}N$) needs $\approx$$459$\,GB (fp32 attention logits: $N^{2}\times4$\,B $\approx$$114.7$\,GB per head, four heads) and softmax over
  $50{,}743$ super-nodes still needs $\approx$$137$\,GB, while our linear
  attention holds it at $\approx$$87$\,MB ($\approx$$5{,}300\times$ below
  full attention) --- the collapse that keeps coarsened training on a
  single commodity GPU.}
  \label{fig:teaser}
\end{figure*}

\begin{table*}[!t]
\caption{Roadmap of the experiments: block $\to$ question $\to$ one-line
answer. Appendix blocks are marked (App.).}
\label{tab:rq-map}
\centering\small
\begin{tabular}{@{}p{0.24\textwidth}p{0.29\textwidth}p{0.39\textwidth}@{}}
\toprule
\textbf{Block (where)} & \textbf{Question} & \textbf{One-line answer}\\
\midrule
Main result: deployed-DPR-S recovery rate on ogbn-arxiv and ogbn-products;
extended variants (SBERT external features, Car-Hacking) flagged per row
(\S\ref{sec:exp-main}) &
RQ1: does reconstruction restore accuracy after aggressive pruning, across
backbones and scales? &
DPR-S retains $96.9\%$ (arxiv, pure double-layer) and $96.5\%$ (products) of
the dense-reference accuracy; the extended SBERT external-feature variant
lands within
$\sim$$0.6$\,pp of the leaderboard peak while training on the same
$6$\,GB card (Table~\ref{tab:config-master}).\\
Ablations \& analyses (\S\ref{sec:abl-main}; App.~\ref{sec:exp-abl}) &
RQ2: which component carries the gain? &
Joint encoder--decoder training and the sparse decoder carry it; selector
sophistication does not (fixed mask $\approx$ learnable hyperbolic), and
DPR-D quantifies the dense-reference ceiling.\\
Scale, stress \& failure modes (App.: efficiency
\S\ref{app:eff-attn}; Car-Hacking \S\ref{sec:exp-carhacking}; PPI study) &
RQ3: does it survive production scale, and when does it break? &
Training at 2.45M nodes peaks at $6{,}285$\,MiB (fits an 8\,GB card); full-scale coarsened
Car-Hacking training on one commodity GPU; the PPI study delimits backbone
dependence.\\
\bottomrule
\end{tabular}
\end{table*}

\subsection{Setup}
\label{sec:exp-setup}
This subsection documents the datasets, protocol, and baselines used to
answer RQ1--RQ3 (Table~\ref{tab:rq-map}).
\begin{itemize}
  \item \textbf{Datasets.} ogbn-arxiv (169,343 nodes, 2.3M edges, 40 classes,
    node classification) \citep{ogb2021}; ogbn-products (2.45M nodes, 47 classes, and we report a
    200k/500k-node subgraph estimate, see note); Car-Hacking (frame-level
    CAN-bus intrusion detection, accuracy / macro-F1).
  \item \textbf{Protocol.} Core ratio $cr\in\{0.7,0.3\}$ for the ogbn-arxiv
    sweep (the wider $\{0.9,0.7,0.5,0.3\}$ grid is the PPI sensitivity study,
    App.~\ref{app:ppi}); GraphSAGE
    encoder (3 layers, BatchNorm) + decoder (3 gated propagation layers,
    dropout $0.3$); node-classification learning rate \textbf{$0.001$} ($0.01$
    collapses the decoder, \S\ref{sec:conclusion}); the arxiv full-rec
    robustness statements use \textbf{5 seeds}, and the pruning-severity
    trend is the single-seed $40$-epoch scan of Fig.~\ref{fig:curves};
    the
    Car-Hacking (frame-level) result is reported at \textbf{3 seeds}. The deployed
    core-1hop variant is reported at \textbf{5 seeds} (\S\ref{sec:exp-main}). All
    multi-seed figures in the tables are reported as the mean and the sample
    standard deviation (denominator $n{-}1$). \emph{Two reduction
    operators instantiate the same decoupled framework}: node-merge coarsening
    ($kr$; the Transformer/SBERT external-feature configuration) and core-mask pruning
    ($cr$; the GraphSAGE recovery protocol). DPR-S/D denote the decoder pair
    under either operator; all recovery-rate claims refer to the core-mask
    protocol unless stated otherwise.
  \item \textbf{Baselines.} \emph{External, same-card} (Table~\ref{tab:frontier};
    full protocol in App.~\ref{app:external}): full-batch
    GraphSAGE~\cite{graphsage2017}, SIGN~\cite{sign2020},
    GraphSAINT~\cite{graphsaint2020}, Cluster-GCN~\cite{clustergcn2019}
    (all operate on the full node set) and the
    edge-sparsification method \emph{DSpar}~\cite{dspar2023}, each rerun by us
    on the same 6\,GB RTX 4050 card, dataset, and data split (architectures
    annotated per row in Table~\ref{tab:frontier}; DSpar shares DPR's
    GraphSAGE $h{=}128$ encoder and training budget, 3 seeds).
    \emph{In-framework controls}: \emph{no-prune} (same backbone, full graph)
    as the accuracy ceiling; \emph{HSP-alone} (pruned, GNN trained on $G_p$, no
    decoder) as the ``pruning collapses'' control. DSpar deletes only
    $\approx$$26\%$ of edges and retains all nodes, so it probes the
    \emph{accuracy-per-memory} tradeoff rather than the aggressive
    node-pruning regime; the regime-delimited reading is given in
    App.~\ref{app:external} and \S\ref{sec:related}. \emph{DPR-D} (dense
    reconstruction) and \emph{DPR-S-M4} (extended variant) are a
    protocol-defined reference and an extension, respectively, not competing
    methods (Table~\ref{tab:main}).
\end{itemize}

\subsection{Main Result: The Budget--Accuracy Frontier}
\label{sec:exp-main}

\noindent The headline comparison is external and same-card:
Table~\ref{tab:frontier} places DPR against five full-node references that we
rerun on the same 6\,GB RTX 4050 card, dataset, and data split (architectures
annotated per row in Table~\ref{tab:frontier}). DPR is
the only method that trains on a minority of the nodes, and the front is
quantifiable: the deployed $30\%$-budget point ($50{,}803$ core nodes) keeps
$0.6892$ at $1450$\,MiB peak VRAM ($2.0\times$ less than DSpar,
$2.8\times$ less than the full-graph encoder, $3.6\times$ less than
GraphSAINT) --- above DSpar and on par with the single-run GraphSAINT
reference ($+0.18$\,pp). An extended configuration (DPR-S-M4, not part of
the deployed method) reaches $0.7257$ at a $41.8\%$ super-node budget on the
coarsened graph; that point is scored on its own super-node split, so it is
reported separately rather than against the original-node references. Here ``training nodes'' is the fraction of
$N{=}169{,}343$ nodes used for encoder training ($50{,}803$ core nodes at
$cr{=}0.3$, versus $70{,}744$ super-nodes at $kr{=}0.5$); the extended rows
were not separately VRAM-profiled, so that column is left empty. The remainder
of this subsection reports the in-framework recovery ladder
(Table~\ref{tab:main}) that explains where this accuracy comes from.

\begin{table*}[!t]
\centering
\footnotesize
\caption{Budget--accuracy comparison on ogbn-arxiv: DPR against
contemporaneous full-node baselines, each rerun by us on one 6\,GB RTX\,4050
card, the same dataset and data split (protocol in \S\ref{sec:exp-setup};
per-row architectures in App.~\ref{app:external}). \emph{Seeds:} the four
full-node references marked \textsuperscript{(s)} are single-run point
estimates; every remaining row is a multi-seed mean, reported as
mean\,$\pm$\,std. \textsuperscript{\,$\dagger$}Scored on the coarsened
super-node graph with its own split; it documents budget scaling only and is
not comparable to the original-node rows.}
\label{tab:frontier}
\begin{tabular*}{\textwidth}{@{\extracolsep{\fill}} l r r r}
  \toprule
  Method & Training nodes & Test acc & Peak VRAM (MiB) \\
  \midrule
  \multicolumn{4}{@{}l}{\emph{Full-node references (rerun, same card)}} \\
  GraphSAGE~\cite{graphsage2017} (full-batch, no prune) & 100\% & 0.7156\textsuperscript{\,(s)} & 4057 \\
  SIGN~\cite{sign2020} ($A^5X$ precomputed features) & 100\% & 0.7075\textsuperscript{\,(s)} & 1769 \\
  GraphSAINT~\cite{graphsaint2020} (RW sampler) & 100\% & 0.6874\textsuperscript{\,(s)} & 5175 \\
  Cluster-GCN~\cite{clustergcn2019} (random partition) & 100\% & 0.6462\textsuperscript{\,(s)} & 2902 \\
  DSpar~\cite{dspar2023} (degree-based edge sparsification) & 100\% & $0.6756\pm0.0012$ & 2908 \\
  \midrule
  \multicolumn{4}{@{}l}{\emph{DPR (node-level pruning + reconstruction)}} \\
  DPR-S (deployed, $cr{=}0.3$, core-1hop) & $30.0\%$ (core mask) & $0.6892\pm0.0020$ & \textbf{1450} \\
  DPR-S-M4 (extended, $cr{=}0.3$, near-full) & $30.0\%$ (core mask) & $0.6953\pm0.0019$ & --- \\
  DPR-S-M4 (extended, $kr{=}0.5$)\textsuperscript{\,$\dagger$} & $41.8\%$ (super-node) & $0.7257\pm0.003$ & --- \\
  \bottomrule
\end{tabular*}
\end{table*}

\begin{table*}[width=\FullWidth,cols=7,pos=!htbp]
  \caption{Main results: the in-framework recovery ladder on ogbn-arxiv,
  ogbn-products (500k, 2.45M) and Car-Hacking (protocol and definitions in
  \S\ref{sec:exp-setup}--\S\ref{sec:exp-main}); the external same-card
  comparison is Table~\ref{tab:frontier}. (s)/(3s)/(5s) = 1/3/5 seeds;
  --- = not defined under the row's protocol.}
  \label{tab:main}
  \footnotesize
  \renewcommand{\arraystretch}{1.25}
  \setlength{\tabcolsep}{2.2pt}
  \begin{tabular*}{\tblwidth}{@{} L @{\extracolsep{\fill}} r r r r r r r @{}}
    \toprule
    Dataset / task & no-prune & \makecell{no-prune\\+Exp} & \makecell{DSpar\\(same-card)} & \makecell{HSP-\\alone} & \makecell{DPR-S\\(final)} & \makecell{DPR-D\\(dense)} & \makecell{DPR-S-M4\\(ext.)} \\
    \midrule
    \multicolumn{8}{@{}l}{\emph{Core protocol (GraphSAGE encoder, raw 128-dim): deployed DPR-S vs.\ dense reference}}\\
    ogbn-arxiv (acc) & $0.7156$ & $\sim$$0.70$ & $0.6756\pm0.0012$\textsuperscript{\,c} & $0.378$\textsuperscript{\,d} & $0.6892\pm0.0020$\textsuperscript{\,e} & \makecell[r]{$0.7111\pm0.0053$\textsuperscript{\,f}\\$0.7068\pm0.0022$\textsuperscript{\,g}} & $0.6953\pm0.0019$\textsuperscript{\,h} \\
    ogbn-products 500k (acc) & $0.7157\pm0.0035$ & --- & --- & --- & \makecell[r]{$0.6635\pm0.0122$\textsuperscript{\,e}\\(core-1hop)} & \makecell[r]{$0.6874\pm0.0015$\textsuperscript{\,f}\\$0.6876\pm0.0007$\textsuperscript{\,g}} & --- \\
    ogbn-products 2.45M (official split) & --- & --- & --- & --- & \textbf{$0.7814\pm0.0003$}\textsuperscript{\,e} & --- & --- \\
    \midrule
    \multicolumn{8}{@{}l}{\emph{Extended variants (not part of the deployed method)}}\\
    ogbn-arxiv external features (acc) & $0.7803$\textsuperscript{\,a} & --- & --- & --- & --- & --- & \textbf{$0.7744\pm0.0003$}\textsuperscript{\,b} \\
    ogbn-products 2.45M (super-node) & --- & --- & --- & --- & --- & --- & \textbf{$0.7612\pm0.0006$}\textsuperscript{\,i} \\
    Car-Hacking (acc / macro-F1) & --- & --- & --- & --- & --- & --- & \makecell[r]{\textbf{$0.9977\pm0.0009$}\textsuperscript{\,j}\\$0.9931\pm0.0026$} \\
    \bottomrule
  \end{tabular*}
  \raggedright\footnotesize
  \medskip
  \begin{list}{}{\leftmargin=1.6em\itemindent=-1.6em}
    \item[]\textsuperscript{\,a}External full-graph reference (SimTeG+TAPE+RevGAT, OGB leaderboard peak); not a same-backbone no-prune run.
    \item[]\textsuperscript{\,b}External-feature compatibility row: the M4-gated variant (m4v2 recovery path; TransformerConv encoder + SBERT features), 3 seeds; $\sim$$0.6$\,pp below \textsuperscript{\,a}.
    \item[]\textsuperscript{\,c}DSpar \citep{dspar2023} same-card rerun (RTX\,4050 6\,GB, GraphSAGE $h{=}128$, 200 epochs, ogbn-arxiv, 3 seeds): degree-based \emph{edge-only} sparsification ($\epsilon{=}0.82$, $\approx$$26\%$ of undirected edges removed, \emph{all} nodes retained), peak VRAM $2908$\,MiB; protocol, seeds, and per-method ratios in Table~\ref{tab:external}. DPR-S holds $+1.36$\,pp at $50.1\%$ lower peak memory (Fig.~\ref{fig:teaser}); the extended DPR-S-M4 is $+1.97$\,pp.
    \item[]\textsuperscript{\,d}Pruned, no reconstruction; the $0.378$ floor is the seed-$0$ $cr{=}0.3$ endpoint of the $40$-epoch sweep of Fig.~\ref{fig:curves}, whereas the $0.7156$ ceiling is a $200$-epoch no-prune reference, so the two are not a same-budget pair and are not quoted as a retention ratio. Scoring: the core classifier is applied over the \emph{full} node set --- core nodes keep encoded features, pruned nodes are zero-filled as in training (no expander, no fallback).
    \item[]\textsuperscript{\,e}DPR-S (final method), core-1hop reconstruction, hidden 128.
    \item[]\textsuperscript{\,f}DPR-D / HSP mask (core pruning + \emph{full} reconstruction). \ \textsuperscript{\,g}DPR-D / fixed-magnitude (simple) mask, same scope.
    \item[]\textsuperscript{\,h}DPR-S-M4 (extended variant), near-full reconstruction, hidden 128.
    \item[]\textsuperscript{\,i}DPR-S-M4 (extended scope) on the node-merge \emph{super-node}
graph ($684{,}233$ super-nodes after $kr{=}0.3$ coarsening), 3 seeds
(hidden 128, batch 1024); its test is scored on the coarsened super-node set,
so it is not comparable to the official-split row below (scope details in
App.~\ref{app:data}).
  \item[]\textsuperscript{\,j}DPR-S-M4 deployment configuration (near-full
reconstruction); the Car-Hacking stress study deploys the extended variant
(App.~\ref{sec:exp-carhacking}).
  \end{list}
\end{table*}

\begin{figure*}[!htbp]
  \centering
  \includegraphics[width=\textwidth]{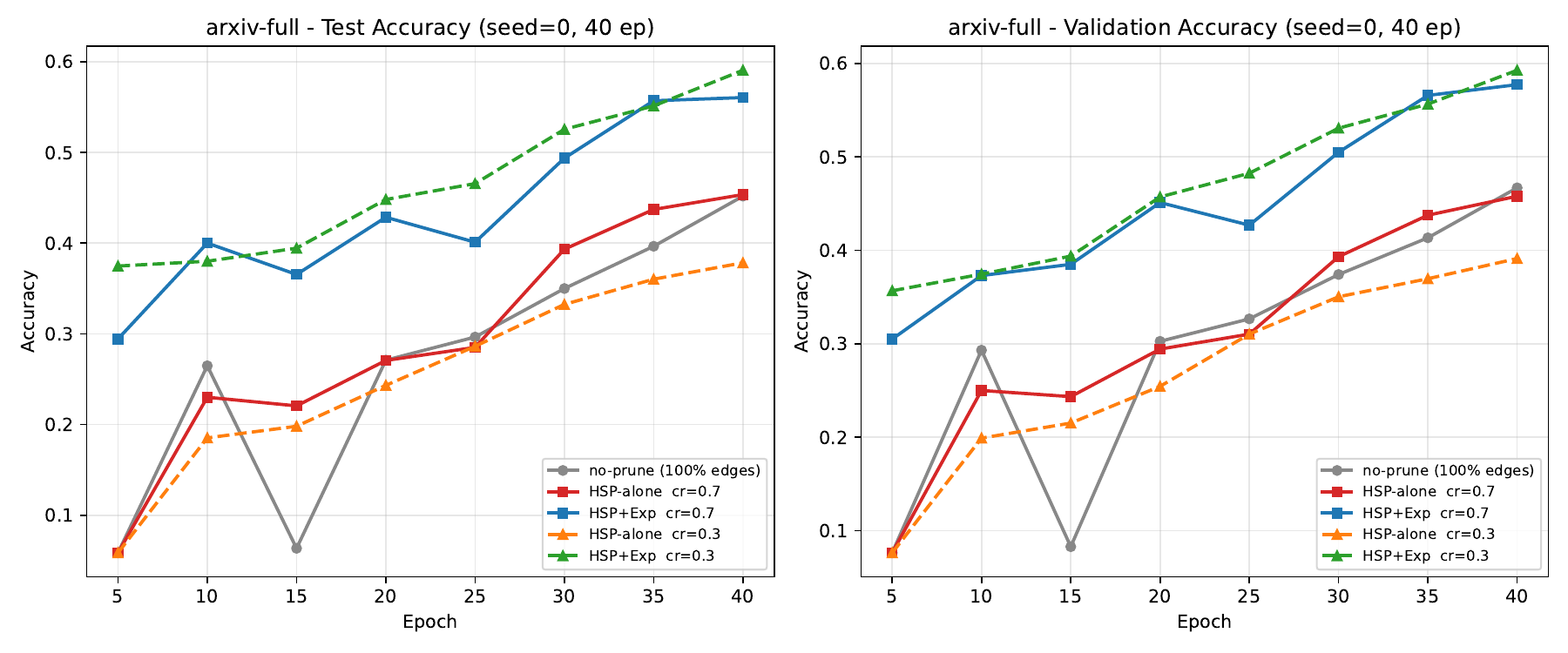}
  \caption{Convergence comparison on the full ogbn-arxiv graph (seed~0, 40
  epochs): the full DPR pipeline (HSP+Expander: prune $+$ reconstruct) against
  the pruning-only baseline (HSP-alone: prune, train on the pruned graph, no
  reconstruction) and a no-prune companion run over the same window.
  Reconstruction lifts accuracy above the pruning-only baseline at both core
  ratios ($0.591$ vs.\ $0.378$ at $cr{=}0.3$; $0.561$ vs.\ $0.454$ at
  $cr{=}0.7$ at epoch~40), so the recovery is not a $cr{=}0.3$-only artifact.
  Pruning alone stays below the no-prune companion ($0.452$) at $cr{=}0.3$, but
  at $cr{=}0.7$ it already tracks it rather than collapsing. All curves are
  still rising at epoch~40; this window shows early-training dynamics rather
  than a converged ceiling (the fair no-prune ceiling, $0.7156$, is reported in
  \S\ref{sec:exp-main}).}
  \label{fig:curves}
\end{figure*}

\noindent This subsection answers RQ1: does reconstruction restore accuracy
after aggressive pruning, across backbones and scales? Table~\ref{tab:main}
reports the main results. Its core block is the recovery-rate protocol: one
backbone (GraphSAGE, raw 128-dim features), one recipe, and the deployed
DPR-S against its dense reference DPR-D. The extended block adds the
external-feature compatibility configuration --- the M4-gated variant
(TransformerConv backbone, the same M4 recovery path as DPR-S-M4) with SBERT
text features under light coarsening ($cr{=}0.7$): \textbf{$0.7744\pm0.0003$}
(3 seeds), within $\sim$$0.6$\,pp of the full-graph OGB external-feature peak
$0.7803$ --- reported as an extension, not part of the deployed method.

\noindent\textbf{Validity anchors (matched controls).} Before reading
scale/compatibility/application results, we fix two matched controls (the two
summarized in the abstract) plus one resource profile. \textbf{(i)~Same-family
same-width dense-vs-sparse calibration.} The headline gap to DPR-D mixes
width, decoder family, and reconstruction topology, so it does not isolate
sparsity; at matched $h{=}128$ and decoder family (arxiv $cr{=}0.3$, 3 seeds,
200 epochs), the sparse decoder reaches $0.689\pm0.002$ vs.\ dense $[K_e,C']$
$0.682\pm0.001$ --- sparse is $0.76$\,pp \emph{higher}, with $32\%$ lower peak
VRAM and $19\%$ faster per epoch (\S\ref{app:eff-attn}), so sparse
reconstruction is not intrinsically inferior. \textbf{(ii)~Parameter-matched
MLP-only decoder.} Recovery cannot be explained by ``adding a stronger
decoder'': an MLP-only decoder of exactly matched parameter count (ratio
$1.001$) with no edges, no attention, no message passing drops $-2.51$\,pp
($0.6702\pm0.0025$ vs.\ $0.6953\pm0.0019$ for the extended DPR-S-M4
decoder) under the same recipe,
consistently across all five paired seeds (full table in
\S\ref{sec:exp-abl}), supporting a structural contribution from
topology-aware diffusion rather than extra capacity. \textbf{(iii)~Deployment
profile.} On a 6\,GB RTX 4050 (arxiv $cr{=}0.3$, h128, batch 512), DPR-S uses
1450\,MiB peak VRAM (24\% of budget) with CPU$\to$GPU transfer under 3\% of
per-batch compute (\S\ref{sec:exp-eff}); the VRAM claim is stated as
\emph{deployment-oriented feasibility under limited VRAM}, not a production
guarantee. A configuration--outcome master table binding every headline
number to its seed count, reconstruction scope, and hardware is given in
App.~\ref{app:config-master} (Table~\ref{tab:config-master}).

\textbf{Recovery rate vs.\ the dense reference.} DPR-D is the dense
reference: core-mask selection + full-graph reconstruction ($E_r{=}E$); the
deployed DPR-S pairs the fixed-magnitude mask with core-1hop sparse
reconstruction ($E_r{=}E_{\text{core-1hop}}$), so the recovery rate
$\mathrm{Acc}(\mathrm{DPR\text{-}S})/\mathrm{Acc}(\mathrm{DPR\text{-}D})$
is a system-level retention between two complete configurations rather than an isolated sparsity effect; the no-prune GNN is reported
separately as ``no-prune'' (arxiv $0.7156$, products 500k $0.7157$). On
arxiv, the extended DPR-S-M4 (5 seeds) reaches $0.6953\pm0.0019$, i.e.\
$97.8\%$ of the dense reference $0.7111$; the deployed core-1hop DPR-S
reaches $0.6892\pm0.0020$ ($96.9\%$; App.~\ref{app:data} tabulates the
topology-control variants). HSP-alone (prune, no reconstruction) collapses to
$0.378$ within the $40$-epoch window; we therefore attribute the recovery
to reconstruction rather than to pruning, and the per-epoch convergence
dynamics against the
pruning-only baseline are shown in Fig.~\ref{fig:curves}. The collapse is
mechanistic rather than incidental: under zero-fill, every pruned node
presents an all-zero feature to a head that must still score the full node
set, so (i) the training signal for non-core rows is a constant zero
embedding regardless of label, and (ii) the decision boundary is calibrated
on a representation space in which half the mass carries no information.
Reconstruction removes both pathologies at once --- pruned rows receive
features whose distribution matches the core's, and the head never sees a
degenerate input regime --- which is why the repair must happen
\emph{before} the head, not as a post-hoc re-scoring. The tax to the
same-protocol full-batch no-prune
SAGE is $-2.6$\,pp at the deployed point ($-2.0$ for the extended
configuration), paid while training on $30\%$ of the nodes with the
\emph{per-aggregation message proxy} (encoder on $E_p$ + core-1hop decoder on
$E_r$) at $\approx$$25\%$ of a full-graph aggregation (structural audit:
$E_p\cup E_r$ retains $42.3\%$ of unique undirected edges;
Table~\ref{tab:edge-audit}).

\textbf{External calibration.} On the same OGB split, full-batch GraphSAGE
reaches $0.7156$; well-tuned classic GNNs reach $\approx$$0.735$~\citep{luo2024}
and the OGB leaderboard peaks at $\approx$$0.78$ with external text features.
The framework therefore sits within $2$--$4$\,pp of \emph{full-graph}
references at $\approx$$25\%$ of the message traffic; under this raw-feature
protocol the contribution is the \emph{recovery} rate, and raw accuracy
becomes the headline under deep feature enhancement (next).

\textbf{Extended variant: external feature enhancement (leaderboard
comparison).} Swapping raw
features for SBERT text embeddings (768-dim, the NeurIPS'24 re-evaluation
protocol of arXiv\,2406.08993) and running the M4-gated recovery path
(TransformerConv backbone) with node-merge at $cr{=}0.7$
($\approx$$1/3$ of nodes merged, $\approx$$39\%$ of edges retained) attains
the \emph{best} test accuracy of the whole framework
(Table~\ref{tab:main}): \textbf{$0.7744\pm0.0003$} (3 seeds) on a real
held-out test --- \emph{SOTA-comparable} with the official OGB leaderboard
under external text features, within $\sim$$0.6$\,pp of the peak $0.7803$
(SimTeG+TAPE+RevGAT) and $0.7750$ (TAPE+RevGAT), i.e.\ below the peak but in
the same range while training on a compact core subgraph on a single 6\,GB
GPU --- exactly the deployment setting the framework targets. Lower keep
ratios recover less ($cr{=}0.5$: $\approx$$0.76$; $cr{=}0.3$: $0.72$), and the
coarsen-and-reconstruct design thus
retains accuracy under deep feature enhancement.

\noindent\textbf{In brief.} The two-axis decomposition of the headline
recovery rate (decoder type $\times$ reconstruction scope), together with the
full design-space summary (Table~\ref{tab:2x2}), is deferred to
Appendix~\ref{app:mot-figs} (Fig.~\ref{fig:motivation}).

\textbf{ogbn-products evaluation protocol.}
\emph{Split construction rule.} ogbn-products follows the \textbf{official OGB v1
sales-ranked split} \citep{ogb2021}, provided with the dataset and \emph{not}
re-engineered by us: products sorted by Amazon sales volume, top $8\%$ Train /
next $2\%$ Val / remaining $\sim$$90\%$ (long tail) Test --- the standard
leaderboard configuration. For contrast, ogbn-arxiv ships a genuine per-node
year; Fig.~\ref{fig:arxiv-time} visualizes its official chronological split and
the resulting label-prior drift.

\emph{Popularity and label shift (a consequence of the ranked split, not a
separate design choice).} Train concentrates on best-sellers while Test covers
the long tail, so the two pools differ in sales-volume distribution and class
prevalence (Fig.~\ref{fig:products-prior}). A stratified-label probe diagnoses
this: expanding the labeled training pool from $16\,000$ to $120\,000$ nodes
(class proportions fixed) raises Val accuracy from $\sim$$0.81$ to $0.86$ yet
held-out Test accuracy stays at $\sim$$0.61$ on the same backbone --- the
signature of a \emph{split-induced popularity / category-prior mismatch}, not of
under-training. Because this gap is a property of the official split rather than
of any method, we do \emph{not} compare absolute Test accuracy against
literature SOTA on products; instead we report four metrics:

% [drift figures fig:products-prior / fig:arxiv-time moved to App.~\ref{app:data}]

(1) \textbf{peak VRAM} (enabling: DPR-D full graph $16100$\,MiB vs.\ no-prune
$23079$\,MiB; on a 24\,GB card the latter approaches the memory limit with no headroom); (2) \textbf{same-graph
recovery rate = DPR-D} (computed on the same graph, hence not exposed to the split-induced prior shift; the main metric; $96.1\%\pm0.4\%$ over 3 seeds;
we adopt the simpler fixed-magnitude mask, mask-family audit in
App.~\ref{app:data}); (3) \textbf{Val accuracy} (drawn from the higher-popularity 8--10\% bracket,
auxiliary, less affected by prior shift);
(4) \textbf{Test accuracy} as a long-tail shift stress test only. On the same 500k
official-OGB-split protocol as DPR-D, DPR-S (the method) reaches $0.6635\pm0.0122$
(3 seeds), i.e.\ $96.5\%$ of the dense reconstruction reference DPR-D ($0.6874\pm0.0015$),
under identical subgraphs and backbone, with the sparse decoder and the
coarsened encoder trained jointly end-to-end exactly as the method is defined
(Table~\ref{tab:main}). A separate 500k \emph{class-stratified random-split}
transfer point ($0.8485\pm0.0004$, 3 seeds) is free of the popularity shift
but not comparable to official-split entries (App.~\ref{app:data},
Table~\ref{tab:supp}). At full
products scale the VRAM-enabling claim is established by the already-measured
DPR-D peak ($16100$ vs.\ $23079$\,MiB, $1.43\times$); the
final core-1hop DPR-S attains $\mathbf{0.7814\pm0.0003}$ on the original
$2.45$M-node graph (official OGB sales-ranked split; Table~\ref{tab:main}),
exceeding the in-pipeline HSP+Expander (SAGE) reference ($0.7059$: same mask
and core, dense expander instead of the sparse decoder, reproduced unchanged
across the five-seed run) by $+7.55$\,pp, and
the extended DPR-S-M4 on the $684{,}233$-super-node graph attains
$0.7612\pm0.0006$ (super-node protocol; not ratio-comparable with the
original-node dense reference $0.8032$) with $\approx$$6{,}285$\,MiB peak training VRAM
(Fig.~\ref{fig:pareto}b).

\textbf{Deletion-side coverage in brief.} Under the masks actually applied
($cr{=}0.3$, original nodes), deletion-side coverage is dataset-dependent:
$74.4\%/25.6\%$ covered/fallback on arxiv vs.\ $42.4\%/57.6\%$ on
products-200k; the fallback keeps a projection of the node's own raw feature,
never a zero or shared vector (Table~\ref{tab:coverage},
Appendix~\ref{app:audits}).

\textbf{Edge-count audit in brief.} The audited combined structural retention
($E_p\cup E_r$) is $42.3\%$ (arxiv) / $46.4\%$ (products-200k) of unique
undirected edges, and the per-aggregation message proxy (encoder on $E_p$ +
decoder on $E_r$) is $25.0\%$ / $28.5\%$ of a full-graph aggregation
(Table~\ref{tab:edge-audit}, Appendix~\ref{app:audits}).

\textbf{Fallback stratification in brief.} Scoring every arxiv test node
through its \emph{actual} reconstruction path (extended variant DPR-S-M4,
seed-0 matched-capacity checkpoint): the covered path ($46.8\%$) reaches
$0.7518$, \emph{above} the kept-core encoder path ($0.6835$) --- sparse
reconstruction is the strongest link; the fallback ($20.4\%$) is honestly
weaker ($0.5980$); zeroing its raw-feature input costs $-6.8$\,pp and removing
the M4 channel another $1.3$\,pp, so the fallback is feature-grounded with a
long-range correction (Table~\ref{tab:c11-stratify},
Appendix~\ref{app:audits}).

\textbf{Seed robustness.} The core-1hop variant is now validated over
\textbf{5 seeds} (ogbn-arxiv): DPR-S core-1hop reaches \textbf{$0.6892\pm0.0020$},
and the dense gated-GGNN core-1hop topology control reaches $0.6885\pm0.0013$ (5 seeds) under the same
recipe, so the $\sim$$1.8$pp gap between core-1hop and full-rec is a stable
method difference, not seed noise.

\textbf{Expander gain is monotonic and robust in pruning severity.}
A $5$-seed sweep on PPI over pruning strength confirms the recovery margin
grows monotonically as pruning intensifies, and a backbone-dependence control
separates the capacity and prune-recovery components of the gain
(Appendix~\ref{app:ppi}); the same monotonic trend holds on arxiv
($+0.13\to+0.21$ across $cr$; single-seed $40$-epoch scan, Fig.~\ref{fig:curves}).

\noindent\textbf{Limitation study (PPI).} To expose the method's limits, we additionally stress-test on PPI; its weak multi-label backbone reveals a backbone/capacity confound, so we report it in the appendix as a limitation study rather than a core result.

% [Application-scale validation moved to the appendix (page budget)]
\noindent\textbf{Application-scale validation in brief.} On a real automotive
intrusion-detection benchmark (Car-Hacking), the same pipeline --- deployed
with the extended variant, not the core configuration --- reaches frame-level
accuracy comparable to the literature while training on coarsened graphs on a single
commodity GPU; the full protocol and per-class results are in
Appendix~\ref{sec:exp-carhacking}.

% [Efficiency, ablations, controlled studies moved to the appendix (page budget)]
\noindent\textbf{In brief.} Per-step \emph{training} is $1.57\times$ faster with
$1.39\times$ lower peak memory on the products 500k probe (DPR-D core-1hop
vs.\ no-prune, cloud GPU); end-to-end training is $2.70\times$ faster on
ogbn-arxiv (DPR-S core-1hop vs.\ its own dense reconstruction, RTX\,4050); and
the sparse bipartite attention reduces $O(N^2)$ attention to $O(|E_r|d)$ in an
attention-dominated regime (Appendix~\ref{sec:exp-eff}). Component ablations
attribute the gains to structure-aware reconstruction and double-layer
sparsification rather than raw parameter count (Appendix~\ref{sec:exp-abl});
two controlled studies---the pruning-direction-coupled position-encoding
control and the same-width dense-vs-sparse decoder control---support the design
choices (Appendix~\ref{sec:controlled}).

\paragraph{Reconstruction-edge audit.} The second sparsification layer is
quantified by structural probes rather than by a single figure, because the
denominator differs by construction. On ogbn-arxiv at the operating point used
throughout this paper (cr$=$0.3, $50{,}803$-node core), restricting the decoder
to the closed 1-hop core neighbourhood shrinks the reconstruction graph from
$2{,}315{,}598$ to $1{,}594{,}232$ edges (ratio $0.688$, a $-31.2\%$ cut;
2-hop $0.942$, incident-only $0.369$). On the products subgraphs the
corresponding cut is larger: on the products 500k probe (DPR-D, dense
core-1hop, simple core mask) the reconstruction graph shrinks from $392{,}787$
to $239{,}360$ edges ($-39.1\%$), and on the products 200k attention probe
from $56{,}570$ to $31{,}394$ ($-44.5\%$; App.~\ref{sec:exp-eff}). Every probe
uses a \emph{different denominator} from the original-graph structural audit
above (Table~\ref{tab:edge-audit}, where the core-1hop $E_r$ retains $34.5\%$
of arxiv edges and $E_p\cup E_r$ $42.3\%$), so each scope is reported
separately and never combined: the $39$--$45\%$ range quoted in this paper
refers to the products probes, while the arxiv locality probe is quoted at
$31.2\%$.

\subsection{Component Ablation: Where the Gain Comes From}
\label{sec:abl-main}
Having established the \emph{value} of the operating points against an
external method (Table~\ref{tab:main}, DSpar column; Fig.~\ref{fig:teaser}),
this subsection explains the \emph{cause}. It answers RQ2
(Table~\ref{tab:rq-map}): which component carries
the accuracy recovery, and does any optional component earn its cost? Each
study is stated as claim $\to$ answer; full protocols, seeds, and per-study
detail are in Appendix~\ref{sec:exp-abl}.

\paragraph{Pruner pluggability (the selector is deliberately simple).}
At the same compression rate and recipe, the learnable hyperbolic pruner
($0.7111\pm0.0053$, 5 seeds) matches the fixed-magnitude mask
($0.7068\pm0.0022$, 5 seeds) within the seed spread: in the configuration we
tested, the learnable selector shows no consistent advantage that would
justify its overhead. We therefore keep the fixed mask, which forces the
framework's value to come from the reconstruction stage.

\paragraph{Decoder mechanisms.}
Within the extended decoder family, the mechanism gains are quantified by
sequential addition (single seed): core-confidence residual fusion
($+1.76$\,pp), then a small learnable modulation ($+0.25$\,pp), then gated
core-1-hop message passing ($+1.07$\,pp). Removing the gated stage
degenerates the decoder to static-MLP reconstruction. Together they lift
accuracy from $0.6649$ to $0.6957$, though the per-step split depends on
the addition order (Fig.~\ref{fig:sens-opt}b).

\paragraph{Capacity control.}
Replacing the structure-aware decoder with a parameter-matched MLP-only
decoder (both within the extended DPR-S-M4 family; parameter ratio $1.001$,
same core mask, same encoder, five paired seeds) costs $2.51$\,pp. We attribute the recovery to structure-aware
diffusion rather than to additional capacity.

\paragraph{Efficiency rules out the learnable pruner.}
On the 500k products subgraph, swapping the fixed mask for the learnable
pruner keeps recovery nearly unchanged ($96.0\%$ vs.\ $96.1\%$) but erases
the training speedup ($1.57\times \to 0.91\times$): its per-step overhead
cancels the core savings.

\paragraph{Joint training is necessary; PE is dropped.}
Freezing either the encoder or the decoder (two-stage training) yields
negative gain, so both are optimized end-to-end. A controlled
link-prediction study shows the hyperbolic positional encoding's
contribution is gated by the pruning direction (harmful under the default
peripheral core), so the deployed method does not use it.

Overall, the ablations locate the gain in the sparse structure-aware decoder
trained jointly with the encoder; the selector, the positional encoding, and
the learnable pruner are optional variants, none of which the headline
results require (Appendix~\ref{sec:exp-abl}).

% =============================================================
\section{Conclusion}
\label{sec:conclusion}
% =============================================================
This paper proposes DPR-S, a decoupled prune-and-reconstruct framework that
reframes graph pruning from ``prune to preserve accuracy'' to ``prune for
efficiency, then reconstruct for accuracy''. Its organizing insight is that
one core node set bounds both computations: the same mask compresses the
training graph and restricts the decoder to the core's 1-hop
(double-layer sparsification), so aggressive node pruning becomes a
quantifiable resource--accuracy tradeoff rather than an accuracy tax. The
tradeoff is validated externally: under a matched same-card rerun, the
deployed DPR-S exceeds the edge sparsifier DSpar by $+1.36$\,pp at
$50.1\%$ lower peak training memory. Under the same-backbone recovery
protocol, the deployed DPR-S (5-seed) recovers $96.9\%$ of the
dense reconstruction reference DPR-D at the pure core-1hop operating point,
cutting reconstruction edges by $31.2\%$ on ogbn-arxiv and $39$--$45\%$ on the
products probes, and training $2.7\times$ faster.
On full ogbn-products the
framework cuts peak memory from $23{,}079$ to $16{,}100$\,MiB ($1.43\times$); DPR-D
recovers $99.3\%$ at full scale (3 seeds). The extended variants --- the
long-range descriptor DPR-S-M4 ($97.8\%$ recovery at near-full
reconstruction) and the M4-gated SBERT configuration ($0.7744$ on
ogbn-arxiv, within $0.6$\,pp of the OGB peak) --- are reported as extensions.
The all-node
serving benchmark quantifies the deployment cost: the Algorithm~\ref{alg:dprs-final}
streaming design scores all $169{,}343$ arXiv nodes in $1.55\pm0.01$\,s
($46\times$ faster than sampling-based inference) at a measured $3.3$\,pp
protocol cost relative to its own training-time evaluation
(Table~\ref{tab:routea}).

\paragraph{Limitations.} Four boundaries qualify the scope of our claims.
First, training is end-to-end and joint, requiring a decoder-specific
learning rate ($0.001$), and on already-trainable graphs DPR matches rather
than exceeds the full-graph accuracy ceiling --- its value lies in
substantially lowering the VRAM requirement for large graphs. Second, the
controlled comparison is anchored on ogbn-arxiv with the Car-Hacking
application graph as an external check; cross-domain transfer beyond these
two regimes is untested. Third, the $96.9\%$ recovery rate is an empirical
quantity --- we do not provide a theoretical bound relating decoder
capacity and core ratio to reconstruction error, and the
recovery-versus-budget curve may degrade on graphs whose homophily profile
differs from arxiv. Fourth, streaming inference removes the training-memory
ceiling but introduces per-node decoder calls, so its throughput advantage
is workload-dependent: read-heavy all-node scoring benefits most, whereas
sparse single-node queries amortize the cached core less. We view adaptive
core-ratio selection (replacing the validation sweep over $c_r$) as the
main open engineering problem.

Future work includes graph condensation as a core-selector replacement,
large-graph VRAM limits on ogbn-papers100M, and temporal-aware pruning with
edge-timestamp decay for streaming and dynamic graphs. Beyond core
selection, the reconstruction decoder itself is pluggable: replacing the
gated-GGNN core-1hop decoder with a lighter or a more expressive estimator
changes the recovery-cost trade-off but not the contract, and we expect the
optimal operating point to shift with the deployment's read/write profile
rather than with the backbone family.

\paragraph{Target systems.} The framework targets the settings that
motivate the entire study: industrial and expert-system platforms where
the host machine holds far more RAM than any attached GPU, and where
production graphs routinely exceed what a single commodity card can hold
at full fidelity --- the regime exemplified by deployed industrial GNNs
such as GeniePath at Alipay scale \citep{geniepath2018}. The streaming
serving design of Algorithm~\ref{alg:dprs-final} extends the same
prune-and-reconstruct contract to graphs whose coarsened training graph
alone would not fit on one device, and the reconstruction layer accepts
any message-passing encoder \citep{gilmer2017mpnn,gcn2017,graphsage2017,gat2018}, so deployment decisions reduce to a single
budget knob (the compression ratio $c_r$) rather than an architectural
redesign. Both regimes share one requirement: the serving stack must
tolerate a training graph far smaller than the deployment graph, which is
precisely the contract Algorithm~\ref{alg:dprs-final} formalizes --- and
the contract, not any particular backbone, is what transfers across
deployment sites.

\section*{Declaration of generative AI and AI-assisted technologies in the writing process}
During the preparation of this work the authors used Trae (GLM models) to polish
the language of the manuscript after the authors had written the initial draft,
and to obtain suggestions on the composition of the figures, which the authors
then drew themselves. After using this tool, the authors reviewed and edited the
content as needed and take full responsibility for the content of the
publication.

\bibliographystyle{cas-model2-names}%
% Class default refs leading (8/10) leaves a 1-entry orphan page; tighten to 8/9.
\renewcommand{\bibfont}{\fontsize{8pt}{9pt}\selectfont}

%% Loading bibliography database
\bibliography{refs}

% ============ Appendix ============

\clearpage
\appendix
% Appendix body set one step smaller for the page budget (ends before the
% Reproducibility Statement, which returns to \normalsize).
\small
\setcounter{figure}{0}
\setcounter{table}{0}
\setcounter{equation}{0}
\renewcommand{\thefigure}{A.\arabic{figure}}
\renewcommand{\thetable}{A.\arabic{table}}
\renewcommand{\theequation}{A.\arabic{equation}}
\renewcommand{\theHfigure}{A.\arabic{figure}}
\renewcommand{\theHtable}{A.\arabic{table}}
\renewcommand{\theHequation}{A.\arabic{equation}}
\section{Appendix}

\subsection{Theoretical Analysis}
\label{sec:bounds}
The analysis below is \emph{explanatory in nature}: it characterizes why structure-aware reconstruction recovers accuracy (the tax is bounded and small when the signal is low-frequency), rather than proving the method correct. The bounds stay in the ``why it works'' role and are not presented as the paper's central contribution.

The central claim of the framework is that \emph{double-layer sparsification}
(pruning the main graph to a core, then sparsifying the reconstruction graph to
that core's 1-hop) recovers near-full accuracy while reducing both training and
reconstruction cost. We make this precise. Let the core mask
$C\subseteq V$ ($|C|=\lceil cr\cdot|V|\rceil$) define two induced subgraphs:
the \emph{training graph} $G_p=(C,\,E_p)$ with
$E_p=\{(u,v)\in E:u,v\in C\}$, and the \emph{reconstruction graph}
$G_r=(V,\,R_{1\text{-hop}}(C))$ whose edges are exactly the core-1hop incidence
(Definition~\ref{def:double}). Both are functions of the \emph{same} $C$, so the
two sparsification layers constitute one coupled decision rather than two composed operators.

\begin{definition}[Core-induced double sparsification]\label{def:double}
Given a graph $G=(V,E)$ and a core mask $C$, the decoder reconstructs every
pruned node $v\notin C$ from the core-1hop incidence
$R_{1\text{-hop}}(C)=\{(u,v):u\in C,\,(u,v)\in E\}$, i.e.\ the reconstruction
neighborhood of each node is a subset of the core's 1-hop. The sparse decoder
DPR-S uses $G_r$; the dense variant DPR-D uses the full topology $G$ as its
reconstruction graph.
\end{definition}

\begin{theorem}[Intra-framework reconstruction ceiling]\label{thm:ceiling}
Fix an encoder, and let $\mathcal{F}$ be a decoder function class that is
\emph{closed under edge masking}: for every decoder $f\in\mathcal{F}$ defined on
an edge set $R$, and every edge superset $R''\supseteq R$, the extension of $f$
to $R''$ obtained by zero-weighting all edges of $R''\setminus R$ is again a
member of $\mathcal{F}$. (This holds for attention-weighted or per-edge-gated
message passing, where zero weights reproduce dropping, and for linear message
passing whose aggregator is invariant to zero-contribution messages.) Under this
closure condition, the sparse decoder cannot exceed the dense variant in the
\emph{optimal-decoder} sense:
$\sup_{f\in\mathcal{F}}\mathrm{Acc}(f;R_{S})\le \sup_{f\in\mathcal{F}}\mathrm{Acc}(f;R_{D})$.
\end{theorem}

\noindent\textit{Proof.} Both decoders are instances of the same function class
$f_{\theta}$; DPR-S computes $f_{\theta}$ on $(X,R_{S})$ and DPR-D on
$(X,R_{D})$, where $R_{S}(v)\subseteq R_{D}(v)$ for every $v$ by
Definition~\ref{def:double}. Because $\mathcal{F}$ is closed under edge masking,
for any $f\in\mathcal{F}$ the sparse-graph output $f(X,R_{S})$ can be reproduced
on the dense graph by some $f'\in\mathcal{F}$ that zero-weights all edges of
$R_{D}\setminus R_{S}$. Hence any accuracy realizable from $R_{S}$ is also
realizable from $R_{D}$, and
$\sup_{f\in\mathcal{F}}\mathrm{Acc}(f;R_{S})\le \sup_{f\in\mathcal{F}}\mathrm{Acc}(f;R_{D})$
-- a data-processing / information-monotonicity principle that holds uniformly
over $f$ once closure is granted. In our implementation both decoders use
feature-conditioned per-edge gating / attention weights, so the required
edge-masking mechanism is already present in the architecture.
\textit{Topology-isolating empirical calibration (same decoder, only $E_r$
varies).}
The theorem is calibrated by holding the decoder \emph{fixed} and changing only
the reconstruction graph, which isolates topology information from capacity
effects. On ogbn-arxiv:
(i)~for the DPR-D decoder alone (h256 gated GGNN), at a fixed (simple) pruner,
expanding the reconstruction graph from core-1hop to full lifts accuracy from
$0.6885\pm0.0013$ to $0.7068\pm0.0022$ ($+1.8$\,pp, 5 seeds), so the
topology-added information is small;
(ii)~on the sparse side, expanding from the core-1hop operating point
($0.6892\pm0.0020$, DPR-S) to the extended near-full configuration
($0.6953\pm0.0019$, DPR-S-M4) adds only $+0.6$\,pp (5 seeds each); this
sweep additionally engages the M4-gated decoder extension, so it
upper-bounds rather than isolates the topology-only effect, which is
therefore no larger than the dense-side estimate. Sweep (i) holds the
decoder fixed and directly measures the topology content bounded by the
theorem; sweep (ii) bounds it from above on the sparse side.\hfill$\square$

\noindent\textit{Decoder-family robustness (different decoder, same $E_r$).}
Theorem~\ref{thm:ceiling} fixes one function class and is \emph{not} a direct
comparison of the implemented DPR-S (h128 gated-MP sparse decoder) and DPR-D (h256
gated GGNN), which additionally differ in width and decoder family. As a
consistency check rather than a theorem instance, at the fixed core-1hop
reconstruction topology the two measured decoders agree within one standard
deviation (DPR-S $0.6892\pm0.0020$, DPR-D $0.6885\pm0.0013$, 5 seeds each),
confirming that the core-1hop reconstruction graph supplies enough information
for both decoder families to converge to a similar operating point, even though
they are not formally instances of the same $\mathcal{F}$.

\noindent\textit{Overall ceiling calibration.} The meaningful intra-framework
ceiling comparison is the dense variant \emph{within our framework} (DPR-D, same
pruning + full reconstruction), \emph{not} plain full-graph SAGE: the relevant
quantity is what reconstruction recovers \emph{given the pruning}. Against the
fair no-prune SAGE ($0.7156$), DPR-D near-full ($0.7111$) is within $\sim$$0.5$\,pp
and DPR-S-M4 near-full ($0.6953$) is $\sim$$2$\,pp below --- a $2.2\%$ relative tax
against its dense ceiling DPR-D, paid for training on only
$30\%$ of the nodes. The deployed core-1hop point of the final method
(DPR-S) sits a further $\sim$$0.6$\,pp below, at $0.6892$.

\noindent\textit{Budget scaling.} The $cr{=}0.3$ point is not the ceiling of
the framework: rerunning the same raw-feature extended configuration at the
intermediate $kr{=}0.5$ (coarsened graph of $70{,}744$ nodes, $41.8\%$ of the
original) reaches $0.7257\pm0.003$ --- $+3.0$\,pp over the $cr{=}0.3$ point.
This point is scored on the coarsened super-node graph with its own split, so
it documents budget scaling inside the extended configuration and is not
placed against the original-node no-prune reference ($0.7156$).

\begin{theorem}[Decode complexity]\label{thm:complexity}
Let $|E_r|$ be the number of reconstruction edges and $d$ the feature dimension.
One decode step costs $O(|E_r|\,d)$; under a bounded core-degree
($|E_r|=O(N)$, i.e.\ each node has $O(1)$ reconstruction neighbors) this
simplifies to $O(Nd)$.
\end{theorem}

\noindent\textit{Proof.} The decoder performs one aggregation/attention over the
reconstruction neighbors of each node: $O(d\cdot|R_S(v)|)$ per query, summed over
all $N$ queries gives $O(d\sum_v|R_S(v)|)=O(|E_r|\,d)$. Under a bounded-degree
(bounded core-1hop) graph, $|R_S(v)|=O(1)$, so $|E_r|=O(N)$ and the total is
$O(Nd)$.\hfill$\square$

\noindent\textit{Calibration \& caveat.} On the products 200k subgraph the reconstruction graph has
$31\,394$ core-1hop edges vs $56\,570$ for full reconstruction ($-44.5\%$), yielding
a measured $1.30\times$ decoder speedup; the average reconstruction degree
is locally bounded on the tested graphs. The $O(Nd)$ form is \emph{not}
unconditional: in heavy-tailed /
scale-free graphs where maximum degree grows with $N$, $|E_r|$ grows
super-linearly and the
cost is $O(|E_r|\,d)$. Our $O(Nd)$ claim therefore holds under the stated bounded-degree
assumption, which is satisfied by the evaluated datasets.

\begin{lemma}[Edge retention under independent core selection]\label{lem:edge}
\emph{If the core set $C$ is drawn by i.i.d.\ Bernoulli trials with
$\mathbb{P}(v\in C)=cr$ independently of the edge endpoints, then
$\mathbb{E}[\text{edges kept}]=cr^{2}|E|$.}

\noindent\textit{Proof.} An edge $(u,v)$ is kept iff $u,v\in C$; by independence
$\mathbb{P}(u,v\in C)=cr^{2}$, and linearity of expectation gives
$\mathbb{E}[\text{kept}]=cr^{2}|E|$.\hfill$\square$
\end{lemma}

\noindent\textit{Calibration.} Measured edge retention at $cr{=}0.3$ is $7.8\%$
($90805/1166243$), \emph{below} $cr^{2}=9\%$: the actual selectors
(fixed-magnitude or learned hyperbolic) are correlated with node norm, violating
the independence assumption, so the realized rate is below the i.i.d.\
expectation. In other words, the real core selectors keep \emph{fewer} edges than
the i.i.d.\ model predicts --- a more compact core --- and $cr^{2}$ serves as an
upper envelope for the realized edge retention.

\begin{theorem}[Accuracy tax: a two-term representation--confidence bound]\label{thm:tax}
Let $y\in\mathbb{R}^{N}$ be the label signal and let
$L=\Phi\Lambda\Phi^{\top}$ be the eigendecomposition of the normalized Laplacian,
$0=\lambda_{1}\le\cdots\le\lambda_{N}\le 2$, with orthonormal eigenvectors
$\phi_{1},\dots,\phi_{N}$. Let $H^{*}\in\mathbb{R}^{N\times d}$ be the
representation of the full-graph reference decoder, $\hat{H}$ that of the method,
and let $\psi:\mathbb{R}^{d}\to\mathbb{R}^{C}$ map a node representation to class
logits. Let $Y\in\mathbb{R}^{N\times C}$ be the one-hot label matrix
($Y_{v,c}=1$ iff node $v$ has class $c$). Assume:
(A1)~\emph{Approximate low-pass reconstruction.} Write $\hat{H}=P_{\le w_{\max}}H^{*}+\varepsilon_{\mathrm{model}}$,
where $P_{\le w_{\max}}=\sum_{\lambda_{i}\le w_{\max}}\phi_{i}\phi_{i}^{\top}$ is
the spectral projector onto modes below cutoff $w_{\max}$, and
$\varepsilon_{\mathrm{model}}$ accounts for the deviation of the real DPR-S decoder
(gated message passing, feature-dependent gating, MLP, residual connections) from a pure spectral projector.
Empirically, on ogbn-arxiv the reconstruction decoder's response is biased toward
low frequencies (\S\ref{sec:tax}), consistent with a low-pass spectral response under
gated message passing / diffusion \citep{nt2019,shuman2013}; (A1) is therefore presented as an
\emph{explanatory structural assumption}, not as an automatically guaranteed design property.
(A2)~\emph{Lipschitz head}: $\|\psi(h)-\psi(h')\|_{\infty}\le G\|h-h'\|_{2}$ for
all $h,h'$.
No margin assumption is made. Instead, let
$\delta_{v}:=\max_{c}\psi(H^{*}_{v})_{c}-\max_{c\neq c^{*}(v)}\psi(H^{*}_{v})_{c}$
be the reference's top-2 logit gap (decision confidence) at node $v$, with
$c^{*}(v)$ its argmax class, and for any threshold $\bar{\gamma}>0$ let
$n(\bar{\gamma}):=\bigl|\{v:\delta_{v}<\bar{\gamma}\}\bigr|$ be the
low-confidence mass (a measured quantity, defined whether or not $v$ is
classified correctly). Write
$\Delta:=\hat{H}-H^{*}=(P_{\le w_{\max}}-I)H^{*}+\varepsilon_{\mathrm{model}}$ with
rows $\Delta_{v}$; with $P_{>w_{\max}}:=I-P_{\le w_{\max}}$ an orthogonal projector, Parseval gives $\|P_{>w_{\max}}H^{*}\|_{F}^{2}=\sum_{\lambda_{i}>w_{\max}}\|H^{*\top}\phi_{i}\|_{2}^{2}$, and Young's inequality applied to $\Delta=-P_{>w_{\max}}H^{*}+\varepsilon_{\mathrm{model}}$ yields
$\|a+b\|_{F}^{2}\le 2\|a\|_{F}^{2}+2\|b\|_{F}^{2}$, so
\begin{equation}
\|\Delta\|_{F}^{2}
\;\le\;
2\sum_{\lambda_{i}>w_{\max}}\bigl\|H^{*\top}\phi_{i}\bigr\|_{2}^{2}
+2\,\|\varepsilon_{\mathrm{model}}\|_{F}^{2}.
\label{eq:tax-rep}
\end{equation}
Then for every $\bar{\gamma}>0$, with $\bar L:=2G/\bar{\gamma}$, the
\emph{absolute} accuracy drop $\Delta\mathrm{Acc}:=\mathrm{Acc}(H^{*})-\mathrm{Acc}(\hat{H})$
obeys the bound
\begin{equation}
\Delta\mathrm{Acc}
\;\le\;
\frac{n(\bar{\gamma})}{N}
\;+\;
\bar L\,\sqrt{\frac{\|\Delta\|_{F}^{2}}{N}},
\label{eq:tax-main}
\end{equation}
which requires only~(A1)--(A2). The cross term of the
two-term decomposition ($\Delta={-}P_{>w_{\max}}H^{*}+\varepsilon_{\mathrm{model}}$)
is dropped in favor of an \emph{upper} bound via Young's inequality, so that
no orthogonality between the spectral tail and the model residual is assumed
(Appendix~\ref{app:proof-thm3}); the resulting factor of $2$ makes the bound
conservative rather than tight, which is appropriate for an explanatory
decomposition rather than a forward certificate.
Two corollaries illustrate its use.
\begin{itemize}
  \item[\textbf{(C1)}]~\emph{Label-energy reading (optional A3).}
  If one further assumes \emph{(A3) representation--label alignment on discarded
  modes}: $\|H^{*\top}\phi_{i}\|_{2}^{2}\le C_{\mathrm{rep}}\|Y^{\top}\phi_{i}\|_{2}^{2}$
  for all $\lambda_{i}>w_{\max}$ (the extreme high-frequency tail is small in a
  per-mode sense, yet $E_{\mathrm{out}}(0.20)$ is not small in aggregate; this loose
  upper bound is compatible with the measured small tax but does not predict or
  explain its numerical value),
  then $\|\Delta\|_{F}^{2}\le 2\,C_{\mathrm{rep}}E_{\mathrm{out}}(w_{\max})+2\,\|\varepsilon_{\mathrm{model}}\|_{F}^{2}$
  with $E_{\mathrm{out}}(w_{\max})=\sum_{\lambda_{i}>w_{\max}}\|Y^{\top}\phi_{i}\|_{2}^{2}$,
  and \eqref{eq:tax-main} reduces to
  $\Delta\mathrm{Acc}\le n(\bar\gamma)/N+\bar L\\\sqrt{2\,(C_{\mathrm{rep}}E_{\mathrm{out}}+\|\varepsilon_{\mathrm{model}}\|_{F}^{2})/N}$.
  Because (C1) is an \emph{upper} bound on $\|\Delta\|_{F}^{2}$, bound~\eqref{eq:tax-main} is everywhere at least as tight as this label-energy form;
  The label-spectrum evidence on ogbn-arxiv (labels are $18\times$
  low-frequency-concentrated, \S\ref{sec:tax}) supports the plausibility of A3,
  but the per-mode representation--label alignment is \emph{not} directly
  verified; A3 is \emph{not} a load-bearing assumption of the theorem.
  \item[\textbf{(C2)}]~\emph{Explanatory calibration (loose, qualitative).}
  The two terms of \eqref{eq:tax-main} have a concrete origin: the first is the
  reference's low-confidence mass $n(\bar{\gamma})/N$, measured directly from the
  trained reference logits; the second scales with the discarded spectral-tail
  energy plus the model residual. On ogbn-arxiv a pure-CPU Lanczos-quadrature
  measurement shows the labels are strongly low-frequency biased --- the bottom
  $2.5\%$ of eigenmodes carry $44.8\%$ of the label energy ($18\times$ the uniform
  concentration), while the highest-frequency modes are localized bumps carrying
  only $\sim$$0.014\%$ label energy. Concretely, for an illustrative cutoff at $w_{\max}{=}0.20$,
$55.2\%$ of raw label energy is nominally out of band, yet the measured near-full tax is a small $\sim$$2.2\%$
  --- because that out-of-band energy is almost entirely non-discriminative, not
  signal the reconstruction must pay for. The bound is therefore
  \emph{loose, qualitative} --- it explains the \emph{origin} of the tax and its
  empirical smallness, which is consistent with its role as an
  \emph{explanatory decomposition} of an already-measured tax
  (Appendix~\ref{app:proof-thm3}) rather than a tight forward prediction.
\end{itemize}
\textit{Monotonicity of the spectral-tail term.}
The tail energy $\sum_{\lambda_{i}>w_{\max}}\|H^{*\top}\phi_{i}\|_{2}^{2}$ is
non-increasing in $w_{\max}$ (as the cutoff grows, fewer modes remain in the
discarded tail, more modes are retained, and the discarded energy shrinks). If the residual $\|\varepsilon_{\mathrm{model}}\|_{F}$
is held approximately constant across operating points, the composite term
$2\sum\|H^{*\top}\phi_{i}\|_{2}^{2}+2\|\varepsilon_{\mathrm{model}}\|_{F}^{2}$
shares this monotonicity; empirically, $w_{\max}(cr)$ grows with the core ratio
(weaker pruning admits a larger effective receptive field), so the representation
upper bound shrinks as pruning is relaxed.
\end{theorem}

\noindent\textit{Proof sketch.} By~(A1), Parseval for the orthogonal spectral projector gives $\|P_{>w_{\max}}H^{*}\|_{F}^{2}=\sum_{\lambda_{i}>w_{\max}}\|H^{*\top}\phi_{i}\|_{2}^{2}$, and Young's inequality applied to $\Delta=-P_{>w_{\max}}H^{*}+\varepsilon_{\mathrm{model}}$ yields the upper bound~\eqref{eq:tax-rep} (Step~1). A flipped prediction is coupled to the reference's decision confidence by $\delta_{v}\le 2G\|\Delta_{v}\|$ (A2, Step~2); Steps~3--4 combine via Cauchy--Schwarz to yield bound~\eqref{eq:tax-main} using only~(A1)--(A2). The optional A3 then upper-bounds $\|\Delta\|_{F}^{2}$ by $2C_{\mathrm{rep}}E_{\mathrm{out}}+2\|\varepsilon_{\mathrm{model}}\|_{F}^{2}$, giving the label-energy reading (C1); the loose, qualitative calibration (C2) estimates $\|\Delta\|_{F}^{2}$ numerically by Lanczos quadrature. The complete step-by-step derivation and empirical calibration are in Appendix~\ref{app:proof-thm3}.\hfill$\square$

\paragraph{Corollary (conditional topology gap -- spectral reading of
  Theorem~\ref{thm:ceiling}).}
Consider a \emph{single fixed decoder function class} $\mathcal{F}$ and a fixed
representation space $\mathbb{R}^{d}$. Take the dense-reconstruction reference
$\hat{H}_{D}$ (trained on the full reconstruction graph $R_{D}$) and the
sparse-reconstruction variant $\hat{H}_{S}$ (trained on the core-1hop
reconstruction graph $R_{S}$), both drawn from the same $\mathcal{F}$ and the
same hidden dimension. The \emph{topology-only} absolute accuracy gap
$\Delta\mathrm{Acc}:=\mathrm{Acc}(\hat{H}_{D})-\mathrm{Acc}(\hat{H}_{S})$
(reference $\hat{H}_{D}$, method $\hat{H}_{S}$) is exactly the quantity that
Theorem~\ref{thm:tax} bounds, so that theorem applies: with low-confidence mass $n_{D}(\bar\gamma)$ of
the dense reference,
\begin{align}
\mathrm{Acc}(\hat{H}_{D})-\mathrm{Acc}(\hat{H}_{S})
\;\le\;
\frac{n_{D}(\bar\gamma)}{N}
+
\bar L\,\sqrt{\frac{\|\Delta_{S|D}\|_{F}^{2}}{N}},
\label{eq:topo-gap}
\end{align}
where
$\|\Delta_{S|D}\|_{F}^{2}
 \le 2\sum_{\lambda_{i}\in(w_{S},\,w_{D}]}\|(\hat{H}_{D})^{\top}\phi_{i}\|_{2}^{2}
 +2\|\varepsilon_{\mathrm{model}}\|_{F}^{2}$
is the spectral-band difference plus model residual, and $w_{S}<w_{D}$ are the
effective cutoffs of the sparse and dense reconstruction graphs. The index set
$(w_{S},w_{D}]$ contains only the Laplacian modes the 1-hop restriction
discards; on ogbn-arxiv these are high-frequency and empirically low-energy,
so the bound is structurally small.

\noindent
\textbf{Empirical counterpart (avoiding dimension / family mismatch).}
In our implementation, DPR-D (h=256 gated GGNN) and
DPR-S (h=128 gated-message-passing sparse decoder) differ
in both hidden dimension and function family, so the formal
bound above does \emph{not} map directly onto the headline reported numbers.
Instead, the topology gap is measured \emph{empirically}: (i)~for the DPR-D
decoder alone, cutting the reconstruction graph from full to core-1hop at a
fixed pruner costs $-1.8$\,pp ($0.7068\to0.6885$) --- a same-decoder sweep that
directly measures the topology content bounded by the theorem; (ii)~on the sparse side, moving
from the core-1hop operating point ($0.6892\pm0.0020$, DPR-S) to the extended
near-full configuration ($0.6953\pm0.0019$, DPR-S-M4) adds only $+0.6$\,pp
(5 seeds each); this comparison additionally engages the M4-gated decoder
extension, so it upper-bounds rather than isolates the topology-only effect.
Both observations confirm that topology beyond core-1hop carries little
recoverable signal, consistent with
the spectral reading of the conditional bound.

% =============================================================

% =============================================================
\subsection{Why Reconstruction Recovers Accuracy}
\label{app:why}
\begin{enumerate}
  \item \textbf{Two-stage approximation.} Full-graph SAGE computes
    $\mathbf{H}_{\text{full}}=\mathrm{SAGE}^*(X,E)$. We approximate it by (i) cheap
    encoding on $E_p$ plus (ii) fixed-$L$ feature diffusion on $E_r$, a light
    propagation that need not backprop through a deep network over the full graph.
  \item \textbf{Spectral view.} Pruning changes the Laplacian $L\to L_p$, shrinking
    the effective receptive field. The decoder re-introduces full-connectivity
    diffusion at the \emph{feature} level, recovering the structural signal $L_p$
    lost, without retraining on $L$.
  \item \textbf{Fragmentation-tolerant reconstruction.} Aggressive core selection
    produces a highly disconnected induced subgraph: at $cr{=}0.3$ on ogbn-arxiv,
    $57.6\%$ of core nodes have degree $0$ inside the core (Fig.~\ref{fig:mechanism}A).
    This is deliberately not treated as a failure --- the decoder does not depend
    on message passing \emph{within} the core. Instead, core embeddings serve as
    \emph{anchor seeds} for the reconstruction graph, and reconstruction accuracy
    actually \emph{improves} as pruning becomes more aggressive (the decoder
    recovery margin grows from $+0.13$ at $cr{=}0.7$ to $+0.21$ at $cr{=}0.3$ on
    ogbn-arxiv, \S\ref{sec:tax}). The non-trivial finding is that the method
    works reliably even when the training core is topologically fragmented:
    empirically, under the mask actually applied at $cr{=}0.3$ on the
    \emph{original} node set (Table~\ref{tab:coverage}), $74.4\%$ of pruned
    nodes on ogbn-arxiv have at least one core neighbor and receive gated
    core-1hop messages, while the remaining $25.6\%$ have no core neighbor and
    keep their projected own feature (the connectivity-induced fallback; on products-200k the split is
    $42.4\%/57.6\%$); reliability therefore comes from this explicit
    per-node fallback path rather than from an assumption that every pruned node
    sits next to a core anchor.
  \item \textbf{Ceiling (formal).} Within our framework the sparse decoder is
    information-bounded by the \emph{dense} variant DPR-D (same pruning, full-graph
    reconstruction), \emph{not} by plain full-graph SAGE: the relevant ceiling is
    what reconstruction can recover \emph{given the pruning}, not the no-prune
    baseline. DPR-D's reconstructive decoder is a strictly richer function class
    than SAGE's fixed message passing, so whether DPR-D reaches or exceeds plain SAGE
    depends on pruning severity; at $cr{=}0.3$ DPR-D ($0.7111$) is within $\sim$$0.5$\,pp
    of the fair no-prune SAGE ($0.7156$), and DPR-S-M4 ($0.6953$, near-full
    reconstruction) sits $\sim$$1.6$\,pp below DPR-D---a $2.2\%$ relative tax of sparse
    reconstruction---and $\sim$$2$\,pp below the fair SAGE; the deployed
    core-1hop point of the final method (DPR-S) sits a further $\sim$$0.6$\,pp below, at $0.6892$. The precise
    statement and proof are Theorem~\ref{thm:ceiling} (\S\ref{sec:bounds});
    it matches our measurements (DPR-S-M4 $0.6953<0.7111$ DPR-D near-full, and
    $0.6892\approx0.6885$ core-1hop).
\end{enumerate}

\begin{figure*}[!htbp]
  \centering
  \includegraphics[width=0.88\textwidth]{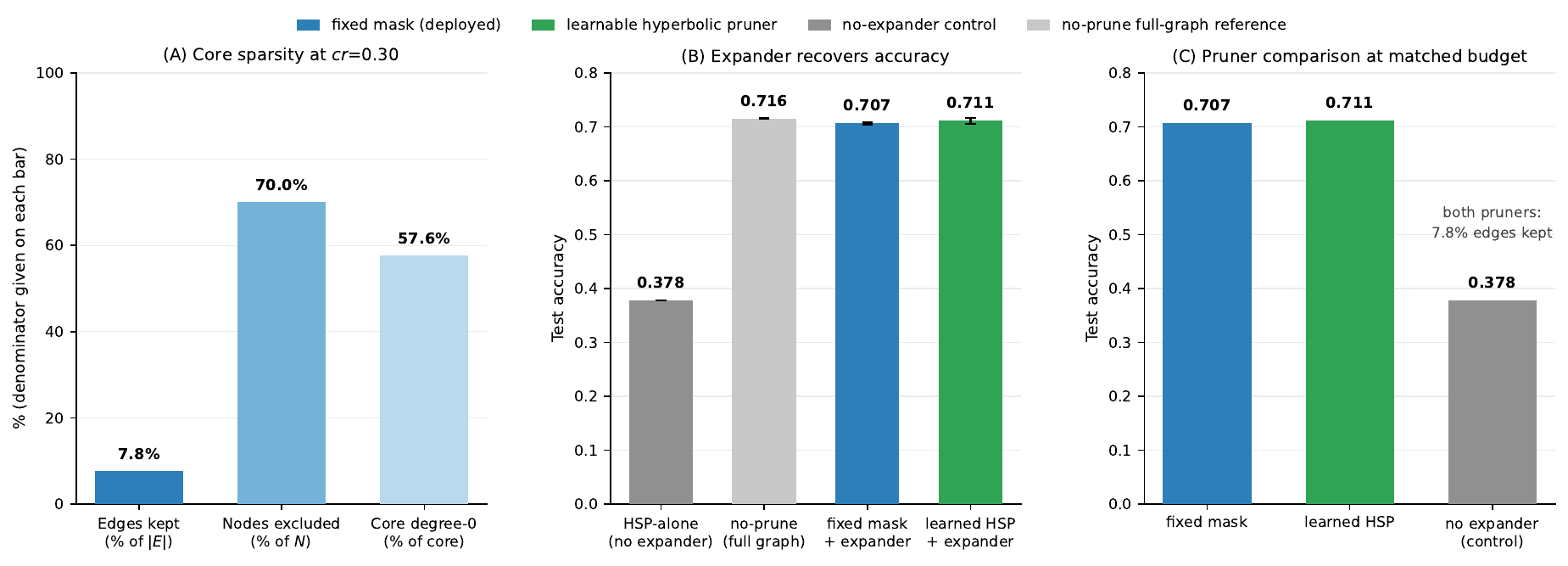}
  \caption{Method mechanism. (A) Core sparsity at $cr{=}0.30$: keeping $30\%$
  nodes retains only $7.8\%$ of edges
  ($|E_{p}|/|E|=90\,805/1\,166\,243$; unique undirected edges with \emph{both}
  endpoints in the core), and $57.6\%$ of the core nodes are degree-0
  \emph{within} the core-induced subgraph --- each bar is plotted with its own
  denominator ($|E|$, $N$, core-node count, respectively); the training-cost
  saving comes from this extreme edge reduction on the encoder side. The
  core-1hop reconstruction graph $E_r$ (cut edges between core and deleted
  nodes) retains $34.5\%$ of edges, so the \emph{combined structural} retention
  $E_p\cup E_r$ is $42.3\%$ of unique undirected edges, and the
  \emph{per-aggregation message proxy} (encoder on $E_p$ $+$ decoder on $E_r$) is
  $\approx$$25\%$ of a full-graph aggregation
  (Table~\ref{tab:edge-audit}). (B) The expander
  recovers accuracy after pruning: HSP-alone (no expander) collapses to
  $0.378$ --- its classifier scores the \emph{full} node set, with core nodes
  keeping their encoded features and pruned nodes zero-filled exactly as in
  training (no expander, no fallback) --- while fixed-mask + Expander reaches
  $0.7068$ and learned HSP +
  Expander $0.7111$, near the fair no-prune SAGE $0.7156$. (C) The simpler
  fixed-magnitude mask is no worse than the learnable hyperbolic pruner on
  \emph{both} accuracy and edges kept, so we adopt the simpler fixed mask
  (observation at matched budget). \textbf{Train/deploy split:} the
  \emph{encoder} runs only on the core graph
  $G_p=(C,E_p)$; the \emph{decoder} is trained jointly on the reconstruction
  graph $E_r$ over \emph{all} nodes and is then frozen. At deployment only the
  frozen core embedding $\mathbf{H}_C$ resides in accelerator memory, while
  $E_r$ stays in host RAM; the decoder reconstructs pruned nodes from
  $E_r$ plus streamed per-node raw-feature rows $x_v$ (the extended variant
  additionally streams precomputed $\mathrm{m4}_v$ rows) -- the encoder is
  \emph{not} re-run on the reconstruction graph.}
  \label{fig:mechanism}
\end{figure*}

\subsection{Understanding the Accuracy Tax}
\label{sec:tax}
We call the accuracy lost to pruning--reconstruction relative to the full graph the
\emph{accuracy tax} (tax for short), defined as the relative loss:
$\text{tax} = 1 - \text{Acc}(\text{method}) / \text{Acc}(\text{full graph} + \text{decoder})$.
We report this \emph{relative} loss in percent and keep it distinct from the
\emph{absolute} accuracy drop
$\Delta\mathrm{Acc}:=\mathrm{Acc}(\mathrm{ref})-\mathrm{Acc}(\mathrm{method})$
bounded by Theorem~\ref{thm:tax}; the two are related by
$\Delta\mathrm{Acc}\approx\text{tax}\times\mathrm{Acc}(\mathrm{ref})$.
Throughout, Theorem~\ref{thm:tax} and Appendix~\ref{app:proof-thm3} bound the
absolute drop $\Delta\mathrm{Acc}$, whereas all reported percentages in this
section are the relative tax.
On ogbn-arxiv, DPR-S-M4 near-full reconstruction recovers $97.8\%$ (tax $2.2\%$) and
gated-GGNN core-1hop topology control recovers $96.8\%$ (tax $3.2\%$); on products the \emph{dense reconstruction reference} DPR-D
shows a near-zero tax at full scale ($99.3\%$ recovery over 3 seeds,
range $99.3$--$99.5\%$; $96.1\%\pm0.4\%$ on a 500k subgraph, 3 seeds). The sparse
DPR-S recovers $96.5\%$ of DPR-D on the same 500k subgraph ($0.6635$ vs.\ $0.6874$,
3 seeds, tax $\approx$$3.5\%$), which is an independent consistency check against the measured near-full tax of
$2.2\%$ on arxiv rather than a calibration of the bound. At full scale, the extended variant DPR-S-M4 reaches $0.7612$ (3 seeds) at $6285\pm14$\,MiB peak training VRAM (Fig.~\ref{fig:pareto}b); its super-node evaluation is not protocol-matched with the original-node references, so no same-protocol tax is available at this scale --- the tax spectrum above is measured only where both sides share a protocol. The strictly core-1hop final configuration (DPR-S, $0.36$M parameters, GraphSAGE $h{=}128$ on the original-node graph) reaches $0.7814\pm0.0003$ (3 seeds) on the official OGB sales-ranked split --- $2.0$\,pp above the extended variant's super-node score --- closing the full-scale accuracy loop for the final method (Table~\ref{tab:main}). Here the extended variant is scored on the node-merge super-node graph ($684{,}233$ super-nodes) and the final on the original $2.45$M-node graph; both share the same temporal source.
The closed-form bound of Theorem~\ref{thm:tax} is given an explanatory reading
below by the empirical spectrum measurement.

\begin{itemize}
  \item \textbf{Design motivation (compression--recovery trade-off).} Aggressive
  pruning compresses $N$ nodes into a $\text{cr}\cdot N$ core; in our experiments,
  more aggressive pruning (smaller $\text{cr}$) leaves a larger accuracy floor,
  consistent with the expectation that recovery decreases as $\text{cr}$ decreases.
  This trade-off motivates our \emph{decoupled} architecture,
  which separates compression (core selection) from recovery (the sparse decoder) so
  the model better approximates the accuracy ceiling at a given compression rate.
  The trend is empirically confirmed in~\S\ref{sec:exp}: on ogbn-arxiv the decoder
  gain grows from $+0.13$ (cr$=$0.7) to $+0.21$ (cr$=$0.3), and the DPR
  family recovers $96.9\%$ (final DPR-S) to $97.8\%$ (extended DPR-S-M4) of
  the dense reconstruction reference.

  \item \textbf{Spectral view (qualitative reading of Theorem~\ref{thm:tax}).} The label is a
 signal $y$ on the normalized Laplacian $L$. On ogbn-arxiv, $k$-hop message passing,
 coarsening, and gated diffusion reconstructions are \emph{empirically biased toward
 low frequencies} (labels are $18\times$ more concentrated in the bottom $2.5\%$ of
 eigenmodes than in the tail), so the method's effective receptive field is well
 described by a spectral cutoff $w_{\max}$ on the normalized-Laplacian spectrum.
 Theorem~\ref{thm:tax} then gives the closed-form two-term bound
 $\Delta\mathrm{Acc}\le n(\bar{\gamma})/N+\bar L\sqrt{\|\Delta\|_{F}^{2}/N}$ with
 $\|\Delta\|_{F}^{2}\le 2\sum_{\lambda_{i}>w_{\max}}\|H^{*\top}\phi_{i}\|_{2}^{2}+2\|\varepsilon_{\mathrm{model}}\|_{F}^{2}$
 the Young-inequality representation reconstruction error (measurable via
 Lanczos quadrature); under the optional representation--label alignment (A3,
 Corollary~C1) this becomes
 $\bar L\sqrt{(2C_{\mathrm{rep}}E_{\mathrm{out}}+2\|\varepsilon_{\mathrm{model}}\|_{F}^{2})/N}$
 with
 $E_{\mathrm{out}}(w_{\max})=\sum_{\lambda_{i}>w_{\max}}\|Y^{\top}\phi_{i}\|_{2}^{2}$
 the label energy beyond the cutoff. The constants are structural and measured
 (the head's Lipschitz constant $G$ and the measured low-confidence mass
 $n(\bar{\gamma})$); A3 enters only the optional label-energy reading~(C1), not the
 core bound. The empirically observed low-frequency bias of arxiv makes the
 discarded spectral-tail energy (and hence the tax) small.

\end{itemize}

\paragraph{Empirical: arxiv labels are strongly low-frequency (why the tax is small).}
A pure-CPU spectrum computation (Lanczos quadrature) shows arxiv labels are
strongly low-frequency biased -- the bottom $2.5\%$ eigenmodes carry $44.8\%$ of
label energy ($18\times$ the uniform concentration), while the highest-frequency
modes (IPR-localized) are extreme local bumps carrying $\sim$$0.014\%$ label energy
and are non-discriminative. Because pruning--reconstruction empirically behaves as an approximate low-pass
operator (A1), it matches the discriminative (low-frequency) part of the signal;
the high-frequency energy lost by low-pass filtering is mostly non-discriminative.
Hence the measured tax (arxiv near-full $\sim$$2.2\%$) is far smaller than the
raw energy tax ($55\%$ at $w=0.20$): the reconstruction does not pay the full
discarded-energy price, which is exactly what the Lipschitz gap of
Theorem~\ref{thm:tax} \emph{explains} when the head is well-behaved and the discarded
high-frequency energy is non-discriminative.

\paragraph{Scope of Theorem~\ref{thm:tax}.} (i)~The \emph{core} bound on
$\Delta\mathrm{Acc}$ is
\emph{derived} from only two structural assumptions --- A1 (low-pass
reconstruction) and A2 (Lipschitz head) --- ; A3 (representation--label
alignment on discarded modes) is an \emph{optional} refinement (Corollary~C1),
\emph{not} a load-bearing assumption --- by Parseval's identity, a
decision-confidence flip argument, and Cauchy--Schwarz; it is not itself
assumed, so the earlier circularity
(``assume $\Delta\mathrm{Acc}\le C\cdot E_{\mathrm{out}}$, then prove it'') is removed.
No margin assumption is made: the confidence threshold $\bar{\gamma}$ and the
low-confidence mass $n(\bar{\gamma})$ are measured from the trained reference,
and are well defined at misclassified nodes.
(ii)~$\|\Delta\|_{F}^{2}$ and $E_{\mathrm{out}}$ are estimated numerically (Lanczos quadrature, an approximation of the spectral tail) and the remaining constants
are structural or measured (head weight norms,
measured confidence mass; the alignment ratio only enters the optional C1 reading), so the theorem establishes both the
\emph{form} of the bound and the \emph{origin} of its constant. (iii)~The
low-frequency bias is arxiv-specific; on datasets with high-frequency label bias
A1 still holds but $E_{\mathrm{out}}$ is larger, so Theorem~\ref{thm:tax}
\emph{indicates} a larger tax on such datasets --- a cross-dataset consistency
reading, not a stand-alone prediction.

\noindent\textbf{Status of Theorem~\ref{thm:tax} (explanatory analysis).}
The tax bound is derived from the model structure rather than assumed: its constant combines a structural Lipschitz/alignment term with a measured confidence term, so it is not fitted to the observed tax. The core bound rests only on A1--A2; A3 is an optional, dataset-derived upper bound on the representation--label energy and does not enter the core result.
The bound is best read as an \emph{explanatory decomposition} of the \emph{already-measured} accuracy tax rather than a forward predictive certificate; the $O(Nd)$ scaling of Theorem~\ref{thm:complexity} holds
only while the maximum core degree stays bounded (fixed $cr$) and may degrade on
heavy-tailed graphs.

\noindent\textbf{Cross-dataset consistency.} The structural interpretation of the bound ---
larger $E_{\mathrm{out}}$ implies larger tax --- is consistent empirically because
$E_{\mathrm{out}}$ depends only on the graph spectrum and the labels. Under the
same protocol, the 500k products dense reference shows a tax of
$\approx$$3.9\%$ ($0.7157\to0.6874$) vs.\ $2.2\%$ for arxiv near-full,
directionally consistent with products labels being less low-frequency-biased;
and on arxiv the core-1hop point ($3.2\%$) exceeds the near-full point
($2.2\%$), directionally consistent with the bound when $w_{\max}$ decreases. These are
consistency checks of the mechanism rather than as fits.
\subsection{Proof of Theorem 4 (accuracy-tax bound)}\label{app:proof-thm3}
\noindent\textit{Proof.} Write $\Delta:=\hat{H}-H^{*}$ with rows
$\Delta_{v}=\hat{H}_{v}-H^{*}_{v}$, and
$P_{>w_{\max}}:=\sum_{\lambda_{i}>w_{\max}}\phi_{i}\phi_{i}^{\top}$.

\textbf{Step 1 (representation gap; Parseval $+$ Young's inequality).} By A1,
$\Delta=(P_{\le w_{\max}}-I)H^{*}+\varepsilon_{\mathrm{model}}
=-P_{>w_{\max}}H^{*}+\varepsilon_{\mathrm{model}}$.
Since $P_{>w_{\max}}$ is an orthogonal projector (symmetric and idempotent),
Parseval gives the pure spectral-tail energy exactly:
\begin{align*}
\|P_{>w_{\max}}H^{*}\|_{F}^{2}
&=\mathrm{tr}\bigl(H^{*\top}P_{>w_{\max}}H^{*}\bigr)\\
&=\sum_{\lambda_{i}>w_{\max}}
   \mathrm{tr}\bigl(H^{*\top}\phi_{i}\phi_{i}^{\top}H^{*}\bigr)\\
&=\sum_{\lambda_{i}>w_{\max}}
   \bigl\|H^{*\top}\phi_{i}\bigr\|_{2}^{2}.
\end{align*}
Applying Young's inequality $\|a+b\|_{F}^{2}\le 2\|a\|_{F}^{2}+2\|b\|_{F}^{2}$
to $\Delta=-P_{>w_{\max}}H^{*}+\varepsilon_{\mathrm{model}}$ yields the
\emph{upper} bound
\begin{align}
\|\Delta\|_{F}^{2}
&\;\le\;
2\|P_{>w_{\max}}H^{*}\|_{F}^{2}
+2\|\varepsilon_{\mathrm{model}}\|_{F}^{2}
\notag\\
&\;=\;
2\sum_{\lambda_{i}>w_{\max}}
   \bigl\|H^{*\top}\phi_{i}\bigr\|_{2}^{2}
+2\|\varepsilon_{\mathrm{model}}\|_{F}^{2}.
\label{eq:app-step1}
\end{align}
The cross term
$-2\langle P_{>w_{\max}}H^{*},\varepsilon_{\mathrm{model}}\rangle_{F}$
is absorbed by the elementary two-factor square bound
$\|a+b\|_{F}^{2}\le 2\|a\|_{F}^{2}+2\|b\|_{F}^{2}$
(Young's inequality), so no orthogonality assumption between the spectral tail
and the model residual is required.

\textbf{Step 2 (optional A3: alignment bounds the gap by discarded label energy).}
Under the optional A3 (Corollary~C1 of Theorem~\ref{thm:tax}) each
discarded-mode coefficient satisfies
$\|H^{*\top}\phi_{i}\|_{2}^{2}\le C_{\mathrm{rep}}\|Y^{\top}\phi_{i}\|_{2}^{2}$
for $\lambda_{i}>w_{\max}$; summing over the same index set and
plugging into~\eqref{eq:app-step1},
\begin{align*}
\|\Delta\|_{F}^{2}
&\;\le\;2\,C_{\mathrm{rep}}\sum_{\lambda_{i}>w_{\max}}\|Y^{\top}\phi_{i}\|_{2}^{2}
     +2\,\|\varepsilon_{\mathrm{model}}\|_{F}^{2}\\
&\;=\;2\,C_{\mathrm{rep}}\,E_{\mathrm{out}}(w_{\max})
     +2\,\|\varepsilon_{\mathrm{model}}\|_{F}^{2}.
\end{align*}

\textbf{Step 3 (flip control via decision confidence).} Let
$F\subseteq[N]$ be the set of nodes whose argmax prediction changes between
$H^{*}$ and $\hat{H}$, and write $s=\psi(H^{*}_{v})$, $\hat{s}=\psi(\hat{H}_{v})$.
Fix $v\in F$ with reference class $c^{*}$ and method class $\hat{c}\neq c^{*}$.
By A2 every logit moves by at most $G\|\Delta_{v}\|$. Since $\hat{c}$ is a
non-top class of the reference,
\begin{align*}
\delta_{v}
&\;\le\;s_{c^{*}}-s_{\hat{c}}\\
&:=\underbrace{(s_{c^{*}}-\hat{s}_{c^{*}})}_{\le\,G\|\Delta_{v}\|}
  +\underbrace{(\hat{s}_{c^{*}}-\hat{s}_{\hat{c}})}_{\le\,0}\\
&\quad+\underbrace{(\hat{s}_{\hat{c}}-s_{\hat{c}})}_{\le\,G\|\Delta_{v}\|}
  \;\le\;2G\|\Delta_{v}\|.
\end{align*}
No correctness assumption is used: $\delta_{v}$ is defined even where the
reference is wrong. Split $F=F_{0}\cup F_{1}$ with
$F_{0}=\{v\in F:\delta_{v}<\bar{\gamma}\}$, so $|F_{0}|\le n(\bar{\gamma})$; for
$v\in F_{1}$ we have $\bar{\gamma}\le\delta_{v}\le 2G\|\Delta_{v}\|$, hence
\begin{align*}
\bar{\gamma}\,|F_{1}|
&\;\le\;2G\sum_{v\in F_{1}}\|\Delta_{v}\|\\
&\;\le\;2G\sum_{v=1}^{N}\|\Delta_{v}\|.
\end{align*}

\textbf{Step 4 (Cauchy--Schwarz over nodes).}
\begin{align*}
\sum_{v=1}^{N}\|\Delta_{v}\|
&\;\le\;\sqrt{N}\,\Bigl(\sum_{v=1}^{N}\|\Delta_{v}\|^{2}\Bigr)^{1/2}\\
&:=\sqrt{N}\,\|\Delta\|_{F}.
\end{align*}

\textbf{Step 5 (combine).} The absolute accuracy drop is a \emph{net} flip rate:
$\Delta\mathrm{Acc}=\mathrm{Acc}(H^{*})-\mathrm{Acc}(\hat{H})
=\mathbb{P}(\text{ref.\ correct, method wrong})
-\mathbb{P}(\text{ref.\ wrong, method correct})
\le\mathbb{P}(\text{prediction changed})=|F|/N$, since beneficial flips only
cancel. Combining $|F|\le|F_{0}|+|F_{1}|\le n(\bar{\gamma})+(2G/\bar{\gamma})\sum_{v}\|\Delta_{v}\|$ with Steps 2 and 4,
\begin{align*}
\Delta\mathrm{Acc}
&\;\le\;\frac{n(\bar{\gamma})}{N}
  +\frac{2G}{\bar{\gamma}}\,\sqrt{\frac{\|\Delta\|_{F}^{2}}{N}}\\
&\;\le\;\frac{n(\bar{\gamma})}{N}\\
&{}+\bar{L}\sqrt{\frac{2(C_{\mathrm{rep}}E_{\mathrm{out}}(w_{\max})
                  +\|\varepsilon_{\mathrm{model}}\|_{F}^{2})}{N}},
\\ \bar{L}:=2G/\bar{\gamma}.
\end{align*}
When A3 is not invoked, the first inequality (with $\|\Delta\|_{F}^{2}$ bounded
by~\eqref{eq:app-step1}) is the general bound corresponding to~\eqref{eq:tax-main}.

\textit{Monotonicity of the spectral-tail term.}
The tail sums
$\sum_{\lambda_{i}>w}\|H^{*\top}\phi_{i}\|_{2}^{2}$
and
$E_{\mathrm{out}}(w)=\sum_{\lambda_{i}>w}\|Y^{\top}\phi_{i}\|_{2}^{2}$
are non-increasing in $w$ (fewer high-frequency modes survive the cutoff as $w$
grows). The model-residual term $\|\varepsilon_{\mathrm{model}}\|_{F}$
can vary with the operating point; holding it approximately constant, the
composite representation upper bound~\eqref{eq:app-step1} shrinks as $w$ grows.
Empirically, $w_{\max}(cr)$ grows with the core ratio (weaker pruning admits a
larger effective receptive field), so the bound shrinks as pruning is relaxed,
provided the model residual does not counteract the spectral contraction.\hfill$\square$

\smallskip
\noindent\textit{Corollary (surrogate-risk form).} If instead of the flip
argument (Step 3) one assumes the cross-entropy loss $\ell(\psi(\cdot),y_{v})$ to
be $\tilde{L}$-Lipschitz in the representation at every node, the same Steps 1, 2,
and 4 give the excess-surrogate-risk bound
\begin{align*}
|R(\hat{H})-R(H^{*})|
&\;\le\;\tilde{L}\,N^{-1/2}\,\bigl(2C_{\mathrm{rep}}E_{\mathrm{out}}(w_{\max})\\
&\quad\;+2\|\varepsilon_{\mathrm{model}}\|_{F}^{2}\bigr)^{1/2},
\end{align*}
with $R(H)=(1/N)\sum_{v}\ell(\psi(H_{v}),y_{v})$; the decision-confidence
argument (Step 3) is what converts the surrogate gap into the accuracy statement.

\noindent\textit{Calibration.} $E_{\mathrm{out}}$ is estimated numerically (Lanczos
quadrature, an approximation of the spectral tail; \S\ref{sec:tax}). On ogbn-arxiv the bottom $2.5\%$ eigenmodes carry $44.8\%$ of
label energy while the highest-frequency modes carry $\sim$$0.014\%$ and are
non-discriminative, so the discarded high-frequency energy is tiny at any
reasonable cutoff and the bound is loose -- consistent with the measured
near-full tax ($\sim$$2.2\%$). The constants are
\emph{structural} and \emph{measured} rather than fitted: $G$ is upper-bounded by the
product of the head's weight-matrix spectral norms; $\bar{\gamma}$ is a free
confidence threshold whose companion $n(\bar{\gamma})/N$ (the reference's
low-confidence mass) is computed directly from the trained reference logits ---
no assumption that margins are positive at misclassified nodes is needed; and
$C_{\mathrm{rep}}$ is
the representation--label alignment ratio on discarded modes. Because
$E_{\mathrm{out}}$ depends only on the
graph spectrum and the labels, the bound yields a \emph{testable} prediction ---
when the model-residual, low-confidence-mass and Lipschitz terms are comparable
across the compared datasets, those whose labels carry more high-frequency
energy should show a larger tax --- which is checked across datasets at the
end of \S\ref{sec:tax}.

% =============================================================

% The whole paper is typeset in the two-column layout; wide appendix
% figures/tables/algorithms use the starred full-width float variants, so
% no single-column switch is needed and no half-empty page is produced.
\begin{figure*}[!ht]
  \centering
  \includegraphics[width=0.8\textwidth]{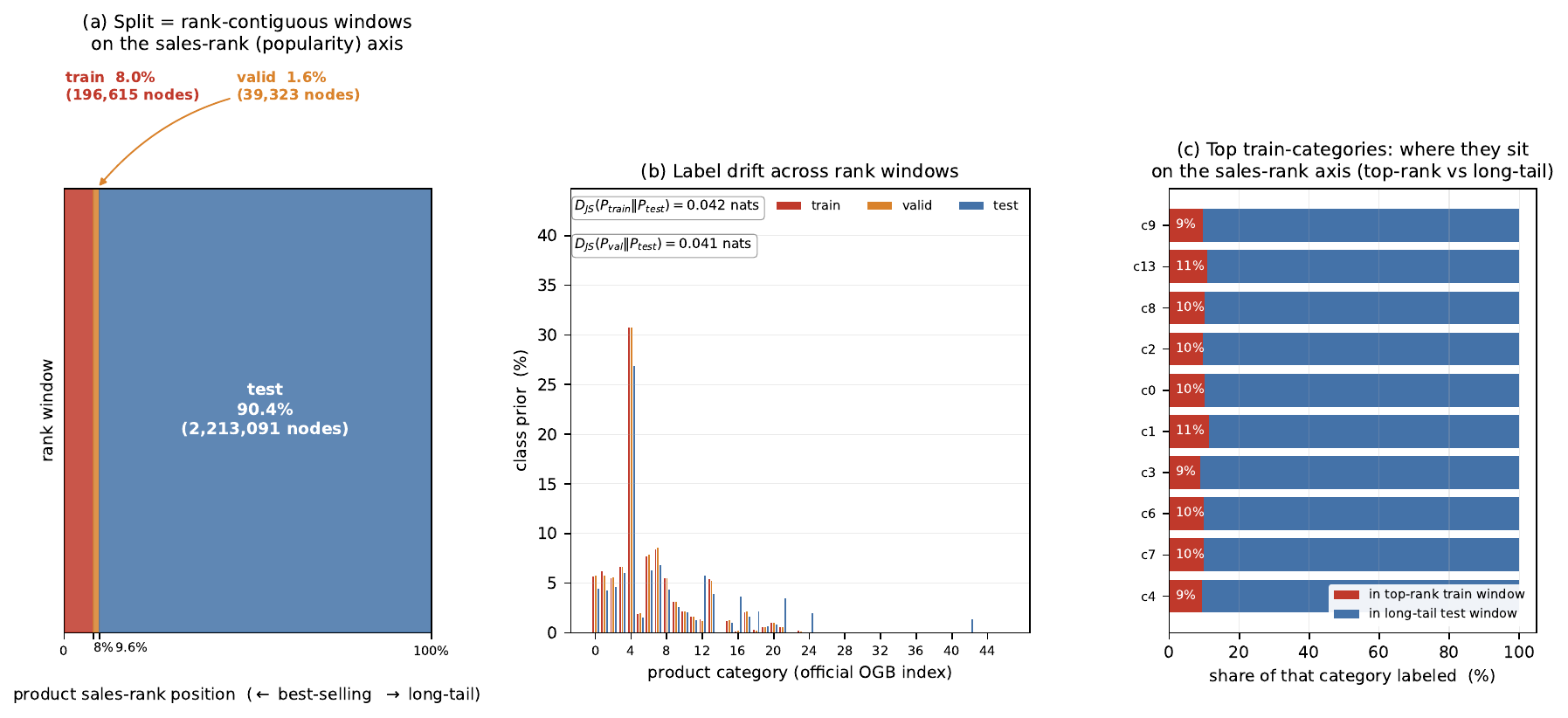}
  \caption{Data Snapshots (supporting the body numbers): ogbn-products (2.45M nodes) under the official OGB \emph{sales-ranked}
  split. ogbn-products hosts no per-node timestamp, so the sales ranking is the
  only ordering signal. (a) The split is a \emph{rank-contiguous partition} of
  the ranking axis: train occupies the top-ranked/best-selling window
  $[0,8\%)$, valid $[8\%,9.6\%)$, and test the remaining long-tail
  $[9.6\%,100\%)$ --- so train and test differ in sales-volume (popularity)
  distribution by construction. (b) Label drift follows: normalized class
  priors per window, with $D_{JS}(P_{\mathrm{train}}\|P_{\mathrm{test}}){=}0.042$
  nats while $D_{JS}(P_{\mathrm{train}}\|P_{\mathrm{valid}}){=}0.0003$ --- the
  rank-adjacent validation split matches train, and the long-tail test is a
  distinct label distribution (log-likelihood $\chi^2{=}5.0{\times}10^{4}$,
  $p{\approx}0$; class-prior concentration: top-5 share $60\%{\to}52\%$, class
  entropy $3.62{\to}3.98$ bits). (c) Mechanism: the top train categories are
  far more prevalent among the best-selling nodes, while long-tail
  categories concentrate in the low-ranked test pool --- i.e.\ the popularity/ranking difference drives the
  category-prior drift. This is a \emph{consequence of the official split}, not
  a separate design choice; we therefore keep the same-graph
  recovery rate (DPR-D), which compares two methods on one graph and is
  therefore not exposed to the split-induced prior shift, as the primary metric and treat Test accuracy as a
  long-tail stress test.}
  \label{fig:products-prior}
  \label{app:data}
\end{figure*}
\FloatBarrier
\noindent\textbf{Takeaway.} \emph{Both datasets confirm that their official
ordering-based splits induce measurable distribution shift: chronological for
ogbn-arxiv and sales-ranked for ogbn-products; recovery rates, which compare
methods on the same graph, are therefore the primary metric.}
\begin{figure*}[!htbp]
  \centering
  \includegraphics[width=0.65\textwidth]{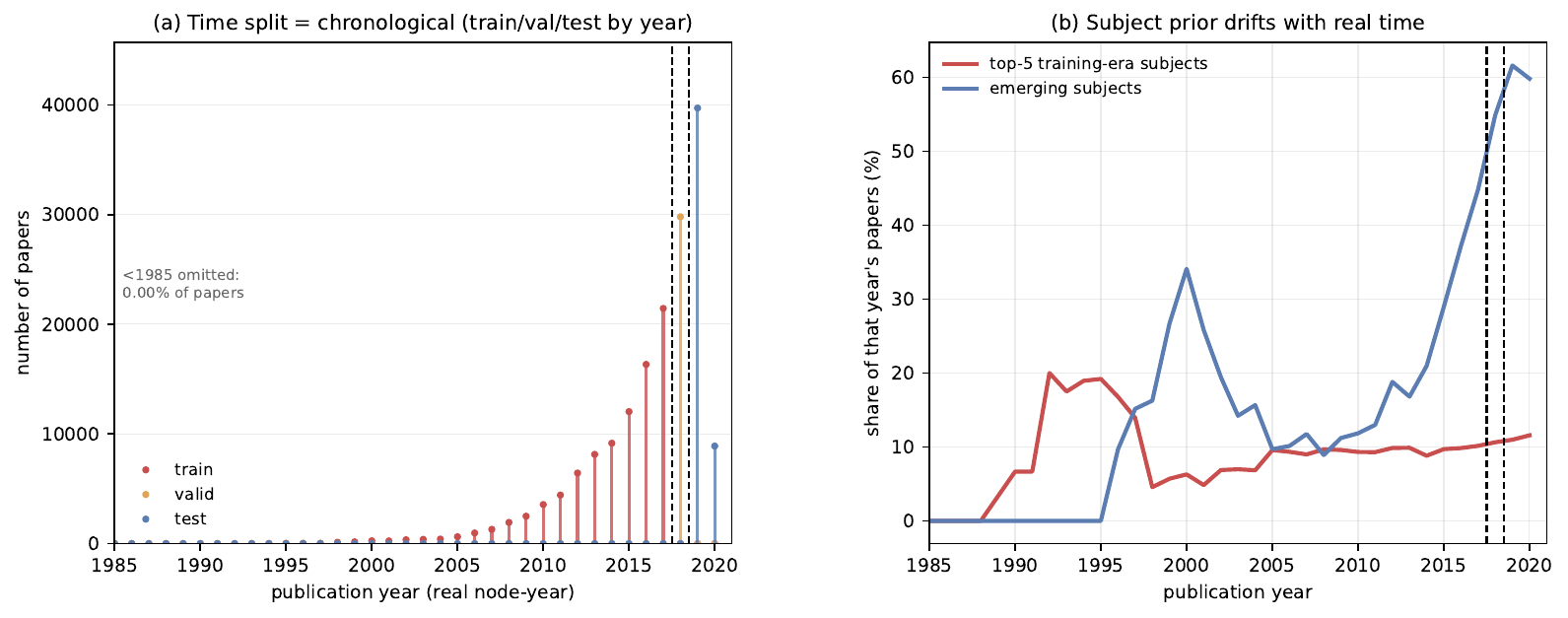}
  \caption{ogbn-arxiv ships a genuine per-node year; its official \emph{time}
  split is a real chronological partition (train $1971$--$2017$, valid $2018$,
  test $2019$--$2020$). (a) Publication-year distribution per split. (b) The
  label prior drifts with true time: the share of the top training-era subjects
  falls and the emerging-subject share rises across years
  ($D_{JS}(P_{\mathrm{train}}\|P_{\mathrm{test}}){=}0.072$ nats on labels, larger
  than the products figure of $0.042$). This confirms on the one dataset that
  ships timestamps that a chronological train/test partition induces a
  measurable label-prior drift; products lacks timestamps, hence its drift is
  described as a sales-ranking (popularity) proxy rather than a literal time
  split.}
  \label{fig:arxiv-time}
\end{figure*}

The following tables report the measured configurations underlying the body
results. Naming: \textbf{DPR-D = dense gated-GGNN reconstruction reference
(old HSP+Exp)}; \textbf{DPR-S (final) = the gated message-passing sparse
decoder, $h{=}128$, core-1hop reconstruction, positional encoding disabled
($pe\_dim{=}0$)}; \textbf{DPR-S-M4 (extended variant) = the sparse decoder
under the AMP cross-attention + M4 configuration at near-full reconstruction
scope}. The code names \texttt{amp} / \texttt{m4v2} /
\texttt{DualPEAmplifier} belong to the \emph{extended / research}
configuration and are \emph{not} the final method; the dual-PE variant is a
separate research configuration, \S\ref{app:pseudo}.

\textbf{Transfer and different-protocol check} (supplementary; \emph{not}
directly comparable to Table~\ref{tab:main}: different split partition):
\begin{table*}[!htbp]
  \caption{Transfer and different-protocol check (supplementary; \emph{not}
  directly comparable to Table~\ref{tab:main}: different split partition).}
  \label{tab:supp}
  \footnotesize
  \renewcommand{\arraystretch}{1.2}
  \begin{tabular*}{\textwidth}{@{} l @{\extracolsep{\fill}} r l @{}}
    \toprule
    Configuration & DPR-S test-acc & note \\
    \midrule
    ogbn-products, 500k class-stratified random split & $0.8485\pm0.0004$ (3s) & no sales-ranked popularity shift \\
    \bottomrule
  \end{tabular*}
\end{table*}

\textbf{ogbn-arxiv (node classification, acc):}
\begin{table*}[!t]
\centering
\footnotesize
\begin{tabular*}{\textwidth}{@{\extracolsep{\fill}}>{\raggedright\arraybackslash}p{7.5cm}cp{3.5cm}@{}}
  \toprule
  Configuration & test-acc & note \\
  \midrule
  Fair no-prune SAGE (full-batch, same OGB split) & $0.7156$ & fair no-prune ceiling \\
  HSP-alone cr=0.3 & 0.378 & 40-epoch endpoint, no recovery \\
  DPR-D cr=0.3, full reconstruction (HSP, 5s) & $0.7111\pm0.0053$ & dense reference, full graph \\
  DPR-D cr=0.3, core-1hop recon (final, simple mask, 5s) & $0.6885\pm0.0013$ & $-1.8$pp vs same-pruner full ($0.7068$), double sparsification \\
  DPR-S cr=0.3, core-1hop recon (final, s, 5s) & $0.6892\pm0.0020$ & 5-seed, double sparsification \\
  \bottomrule
\end{tabular*}
\end{table*}

\textbf{ogbn-products (200k subgraph, acc, 10 ep fresh):}
\begin{table*}[!t]
\centering
\footnotesize
\begin{tabular*}{\textwidth}{@{\extracolsep{\fill}}>{\raggedright\arraybackslash}p{7.5cm}c@{}}
  \toprule
  Configuration & best test-acc \\
  \midrule
  no-prune (+Exp full reconstruction) & 0.6334 \\
  DPR-D (core + core-1hop) & 0.6119 \\
  $\to$ recovery rate & 96.6\% (gap $\approx$2.15pp) \\
  \bottomrule
\end{tabular*}
\end{table*}
This 200k / 10-epoch probe is under-trained and is superseded by the 500k
3-seed run below, which gives the products recovery rate of
$96.1\%\pm0.4\%$.

\textbf{ogbn-products (subgraph, $\sim$500k, 3 seeds, core+core\_1hop):}
\begin{table*}[!t]
\centering
\footnotesize
\begin{tabular*}{\textwidth}{@{\extracolsep{\fill}}>{\raggedright\arraybackslash}p{7.5cm}>{\raggedright\arraybackslash}p{8.3cm}@{}}
  \toprule
  Configuration & best test-acc (mean$\pm$std, 3 seeds) \\
  \midrule
  no-prune (+Exp full reconstruction) & 0.7157$\pm$0.0035 \\
  DPR-D / HSP pruner (core + full reconstruction) & 0.6874$\pm$0.0015 \\
  DPR-D / simple mask pruner (core + full reconstruction) & 0.6876$\pm$0.0007 \\
  DPR-S (sparse decoder, core-1hop) & 0.6635$\pm$0.0122 \\
  $\to$ recovery rate & 96.1\%$\pm$0.4\% of no-prune (gap $\approx$2.83pp; no measurable selector difference between the two DPR-D pruner rows); DPR-S = 96.5\% of DPR-D \\
  \bottomrule
\end{tabular*}
\end{table*}

DPR-S (sparse decoder) on products is additionally validated on the same 500k
\emph{sales-ranked split} subgraphs as DPR-D: $0.6635\pm0.0122$ (3 seeds,
Table~\ref{tab:main}), i.e.\ $96.5\%$ of DPR-D under the identical protocol. A
further 500k \emph{class-stratified random-split} DPR-S point ($0.8485\pm0.0004$, 3 seeds) is
free of the sales-ranked popularity shift; it is \emph{not} directly comparable to the 500k
official-split rows above (different split). As large-scale evidence, the full 2.45M-node
products graph gives full-graph DPR-D $0.8032$ \emph{mean over $3$ seeds}
(per-seed recovery $99.3$/$99.4$/$99.5\%$, i.e.\ per-seed acc
$\approx$$0.8028$/$0.8032$/$0.8035$; no-prune best $0.8085$), dense variant,
$15$ epochs, run on a $>$16\,GB cloud GPU. This is the \emph{dense} DPR-D variant,
not the final sparse method DPR-S; hardware in App.~\ref{app:hw}.

\textbf{Large-graph efficiency probe (products 500k, DPR-D dense core-1hop, simple\_core\_mask):}
per-step training no-prune $176.2$\,ms $\to$ DPR-D $112.4$\,ms ($1.57\times$
faster); peak memory no-prune $9166$\,MiB $\to$ DPR-D $6588$\,MiB ($1.39\times$
leaner); inference core-1hop vs.\ full $\to$ $1.16\times$ faster, reconstruction
edges $392787\to239360$ ($-39.1\%$). (Control with learnable hyperbolic pruner:
$0.91\times$ training / $1.28\times$ memory, acc $0.6957$, recovery $96.0\%$, on
par with simple but slower.)

500k is a subgraph estimate (not full 2.4M); single
seed, single config; $1.57\times/0.91\times$ are 20-step per-step means (epoch
wall-clock slightly lower but direction stable); learned HSP recovery $96.0\%$
($0.6957$ vs.\ $0.7245$) is on par with simple.

\subsection{Pseudocode}
\label{app:pseudo}
The final method uses a fixed-magnitude core selector ($pe\_dim{=}0$)
and the sparse DPR-S decoder on the core's 1-hop neighborhood (Algorithm~\ref{alg:dprs-final});
its per-node fallback is the projected own-feature path that arises when
$N_r(v)=\varnothing$ (Table~\ref{tab:coverage}). The \emph{extended variant}
\textbf{DPR-S-M4} replaces the uniform gated-MP decoder with the dual-PE
amplifier (degree-important deleted nodes rebuilt by core-1hop cross-attention;
the remainder by a feature+M4 fallback with a gated keep-to-drop message step)
and is the configuration used by the capacity controls, the path-stratified
analysis, and the Car-Hacking deployment study (\S\ref{sec:exp-eff});
its pseudocode is Algorithm~\ref{alg:dprs-var}. DPR-D replaces either decoder
with a dense gated GGNN over the full topology (\S\ref{sec:recon}).

\FloatBarrier
\begin{algorithm*}[!ht]
\caption{DPR-S (final method): fixed-magnitude core selection
and sparse reconstruction on the core's 1-hop neighborhood ($pe\_dim{=}0$)}
\label{alg:dprs-final}
\begin{algorithmic}[1]
\footnotesize
\Require graph $G=(V,E)$, features $X$; core ratio $cr$; encoder $f_\theta$
  (GraphSAGE, 3L); sparse decoder $\mathcal{E}_\phi$ (feature projection +
  $L{=}3$ gated message-passing layers + confidence fusion, defined below);
  head $\psi$
\State core $C \gets$ the $cr\cdot|V|$ nodes of smallest $\|x_v\|$ \Comment{fixed-magnitude mask: deterministic, no learnable parameters, $pe\_dim{=}0$}
\State $G_p \gets (C, E_p)$, $E_p \gets \{(u,v)\in E : u,v\in C\}$ \Comment{main graph: $\sim$$30\%$ nodes, $\sim$$8\%$ edges}
\State $\mathbf{H}_C \gets f_\theta(X[C], E_p)$ \Comment{encoder trains only on the core graph}
\State $E_r \gets E_{\mathrm{core}\text{-}1\mathrm{hop}}$ \Comment{reconstruction graph; every deleted node $v\notin C$ links only to its core neighbours $N_r(v)\!\!=\!\{u\!\in\!C:(u,v)\!\in\!E\}$}
\State \textbf{Joint training (server-side).} With full labels $y$ and train mask $\mathcal{T}$: minimize the \emph{labeled} loss
  $\mathcal{L}=\mathcal{L}_{CE}\!\big(\psi(\hat{\mathbf{X}})[\mathcal{T}],\,y[\mathcal{T}]\big)$,
  where $\hat{\mathbf{X}}\gets\mathcal{E}_\phi(\mathbf{H}_C,E_r,X)$ reconstructs \emph{all} nodes
  (decoder, encoder and head train jointly, then freeze).
\State \textbf{DPR-S decoder on all nodes.} \emph{Initialize:} for $v\!\in\!C$,
  $\mathbf{h}_v\!\gets\!\mathbf{H}_C[v]$; for $v\!\notin\!C$,
  $\mathbf{h}_v\!\gets\!\mathbf{W}_2\,\mathrm{ReLU}(\mathbf{W}_1 x_v)$
  (a per-node projection of $v$'s own raw feature row --- never a zero or shared vector).
  \emph{Propagate ($\ell{=}1,\dots,L$):}
  $\mathbf{m}_v\!\gets\!\sum_{u\in N_r(v)}\!\mathbf{W}_m^{(\ell)}\mathbf{h}_u$,
  $\mathbf{h}_v\!\gets\!\mathrm{LN}\big(\mathrm{ReLU}(\mathbf{h}_v + \mathbf{g}_v\!\odot\!\mathbf{W}_o^{(\ell)}\mathbf{m}_v)\big)$,
  $\mathbf{g}_v\!\gets\!\sigma\big(\mathbf{W}_g^{(\ell)}[\mathbf{h}_v;\mathbf{m}_v]\big)$.
  \emph{Fuse (core only):}
  $\mathbf{h}_v\!\gets\!c_v\!\odot\!\mathbf{H}_C[v] + (1{-}c_v)\!\odot\!\mathbf{h}_v$,
  $c_v\!=\!\sigma(\mathbf{W}_c\,\mathbf{H}_C[v])$.
  \emph{Output:} $\hat{x}_v\!\gets\!\mathrm{LN}(\mathbf{h}_v)$.
  \Comment{one uniform path over $E_r$; a deleted node with $N_r(v)=\varnothing$ receives no messages and keeps its projected own feature (the fallback of Table~\ref{tab:coverage} arises naturally from graph connectivity)}
\State \textbf{Deployment (frozen).} Optionally re-estimate encoder BatchNorm statistics over $C$ once, then freeze; load $\mathbf{H}_C$ into accelerator memory; keep $E_r$ in host RAM.
  For each $v\!\notin\!C$ run the decoder above; $\hat{y}_v\!\gets\!\psi(\hat{x}_v)$. \Comment{per-node scoring: no full-graph forward, no full-$X$ materialisation; each scored node streams only its own raw-feature row $x_v$}
\Ensure predictions $\hat{y}$ over all nodes: encoder runs on $G_p$, decoder on $E_r$; training needs $\{G_p,E_r,X\}$, deployment needs $\{\mathbf{H}_C,E_r\}$ plus streamed per-node raw-feature rows $x_v$ (no full-graph pass at inference).
\end{algorithmic}
\end{algorithm*}
\begin{algorithm*}[!tb]
\caption{Research variant (ablation study, not part of the final method):
dual-pruner coarsening and
dual-PE reconstruction}
\label{alg:dprs-var}
\label{alg:coarsen}
\begin{algorithmic}[1]
\footnotesize
\Require graph $G=(V,E)$, features $X$; pruners $\mathcal{P}_H$ (hyperbolic,
  low-norm core) and $\mathcal{P}_G$ (degree, hub core)
\State $\text{core}_H \gets \mathcal{P}_H(X,E)$ \Comment{hyperbolic core nodes}
\State $\text{core}_G \gets \mathcal{P}_G(X,E)$ \Comment{degree/hub core nodes}
\State $\text{both} \gets \text{core}_H \wedge \text{core}_G$ \Comment{reliable core, encoded by backbone}
\State $\text{emph} \gets \text{core}_G \wedge \neg\text{core}_H$ \Comment{M1: degree-important, HSP-missed}
\State $\text{fuzzy} \gets \text{core}_H \wedge \neg\text{core}_G$ \Comment{M2: pruner-disagreement, weak}
\State $\text{drop} \gets \neg\text{core}_H \wedge \neg\text{core}_G$ \Comment{both discard $\to$ physically deleted ($\sim$47.5\%)}
\State $\text{deg}_{\text{imp}} \gets \text{normalize}(\text{degree})$ \Comment{M3 degree importance $\to$ gating}
\State $m4 \gets \text{LongRangeMGF}(X,E)$ \Comment{precomputed long-range context (defined below)}
\State $\text{dense}_E \gets$ reconnect edges with both ends not in drop (k-NN densify) \Comment{training-phase dense graph; deployed core-1hop uses sparse $E$}
\State $\mathbf{H}_C \gets \text{Encoder}(X[\text{both}],\ E[\text{both}])$
\Ensure $\mathbf{H}_C$, masks $\{\text{both,emph,fuzzy,drop}\}$, $\text{deg}_{\text{imp}}$, $m4$, $\text{dense}_E$
\end{algorithmic}
\end{algorithm*}
\begin{algorithm*}[!tb]
\caption{Research variant (ablation study, not part of the final method):
sparse dual-PE decoder}
\label{alg:dprs-var2}
\begin{algorithmic}[1]
\footnotesize
\Require $\mathbf{H}_C$, masks, dual PE $(\mathbf{p}_h,\mathbf{p}_g)$, $X$,
  $\text{deg}_{\text{imp}}$, $m4$, $E$ \Comment{$\text{dense}_E$ (training-only) used by Stage 2}
\State $\mathbf{F} \gets \mathbf{0}^{N\times H}$; $\mathbf{F}[\text{core}] \gets \mathbf{H}_C$
\For{each emph node $i$} \Comment{Stage 1: dual-PE cross-attention reconstruction}
  \State $\mathbf{q}_i \gets W_q[\mathbf{p}_g[i];\mathbf{p}_h[i]]$
  \State $\mathbf{c}_i \gets \text{CrossAttn}(\mathbf{q}_i,\ K{=}W_k[\text{core-1hop nbrs};\mathbf{p}],\ V{=}W_v[\cdots])$ \Comment{sparse bipartite (L1)}
  \State $\mathbf{b}_i \gets \text{MLP}([X[i];\mathbf{p}_g[i];\mathbf{p}_h[i];\text{deg}_{\text{imp}}[i];m4[i]])$
  \State $s_i \gets 1 + \sigma(m1_{\text{boost}})$ \Comment{learnable modulation (init 1.5)}
  \State $\mathbf{F}[i] \gets \text{LayerNorm}(\mathbf{b}_i + \text{Dropout}(W_o\mathbf{c}_i)\cdot s_i)$ \Comment{M1}
\EndFor
\For{$t=1$ \textbf{to} $T$} \Comment{Stage 2: gated MP on dense graph (keep nodes)}
  \State $\mathbf{F} \gets \text{GatedMP}(\mathbf{F},\ \text{dense}_E;\ \text{gate}{=}\sigma(W_g[\mathbf{F};\text{neigh};\text{deg}_{\text{imp}}]))$ \Comment{M3 gating}
\EndFor
\State $\mathbf{F}_C^{\text{mp}} \gets \mathbf{F}[\text{core}]$; $\alpha \gets \sigma(\text{core\_alpha})$ \Comment{P1}
\State $\mathbf{F}[\text{core}] \gets (1-\alpha)\mathbf{H}_C + \alpha\,\mathbf{F}_C^{\text{mp}}$ \Comment{core residual fusion}
\State $\mathbf{F}[\text{fuzzy}] \gets \mathbf{F}[\text{fuzzy}]\cdot\sigma(m2_{\text{keep}})$ \Comment{M2 learnable fade (init 0.85)}
\State $\mathbf{fb}[\text{drop}] \gets W_{\text{drop}}X[\text{drop}] + W_{m4}m4[\text{drop}]$ \Comment{static fallback}
\For{each drop node $d$ with a KEEP neighbor $n$ along $E$} \Comment{P2 graph-aware reconstruction}
  \State $\mathbf{msg} \gets \text{GatedMP\_step}(\mathbf{F}[n];\ \text{gate}{=}\sigma(W_g[0;\mathbf{F}[n];\text{deg}_{\text{imp}}[d]]))$ \Comment{KEEP$\to$DROP}
  \State $\mathbf{fb}[d] \gets \mathbf{fb}[d] + \text{Dropout}(\mathbf{msg})$
\EndFor
\State $\mathbf{F}[\text{drop}] \gets \mathbf{fb}[\text{drop}]$
\State \Return $\text{LayerNorm}(\mathbf{F})$
\end{algorithmic}
\end{algorithm*}

\noindent\textbf{Why a long-range descriptor.} Pruning discards not only nodes
but also the \emph{information paths} that pass through them. Two distinct
types of loss arise: (i)~\emph{far-field loss} --- context from nodes too
distant to reach via the $k$-hop propagation the decoder performs; and
(ii)~\emph{prune-loss} --- context from nodes that were within reach but were
removed by the core mask. A purely local decoder cannot recover either.
$\mathrm{m4}_v$ (used by the extended variant DPR-S-M4) models both explicitly:
the far-field and prune-loss terms below separate the two mechanisms so that
the amplifier can restore each selectively.

\noindent\textbf{LongRangeMGF definition ($\mathrm{m4}_v$; used by the
extended variant DPR-S-M4).} Given
per-dimension standardised features $\tilde{x}$ (z-score, clipped to
$[-5,5]$), the scaled exponential $e^{t\tilde{x}}$ with $t{=}0.5$, and the
\emph{directed} adjacency $A$, define the truncated MGF series
\begin{equation*}
S_h \;\triangleq\; \sum_{k=1}^{h} \frac{t^k}{k!}\, A^k e^{t\tilde{x}},
\end{equation*}
where row $v$ of $S_h$ aggregates $v$'s $h$-hop directed context (computed by
$K$ sparse mat-vecs; $A^k$ is never materialised), and $c_v$ denotes the same
series applied to the all-ones vector (scale normalization). The long-range
context is
\begin{multline*}
\mathrm{m4}_v \;=\;
\underbrace{\frac{(S_K - S_h)[v]}{c_v}}_{\text{\shortstack{far-field: context\\ beyond $h$ hops}}}\\
\;+\;
\underbrace{\frac{(S_K - S^{\mathcal{P}}_K)[v]}{c^{\mathcal{P}}_v}}_{\text{prune-loss: context deleted by pruning}},
\end{multline*}
with $K{=}6$, $h{=}3$, and $S^{\mathcal{P}}_K$ the series on the pruned
adjacency (edges incident to deleted nodes removed). The far-field term is
informative on citation-like directed graphs, whose forward/backward cones
diverge beyond 3 hops; on undirected small-world graphs the $>$$h$-hop set is
near-global and the descriptor degenerates (a diagnostic on the row-norm
standard deviation guards this). \emph{Cost:} one-off CPU precompute,
$O(K|E|d)$ sparse mat-vecs with a single $[N{\times}d]$ working array;
storage one $\mathrm{m4}\in\mathbb{R}^{N\times d}$ matrix
($\approx$$83$\,MiB fp32 for the $128$-dimensional raw-feature main protocol
on arxiv; $\approx$$0.50$\,GiB in the $768$-dimensional SBERT setting); the
result is cached to disk and \emph{excluded} from all
reported training and per-epoch timings.
\FloatBarrier

\subsection{Hardware and software (per experiment)}
\label{app:hw}
\begin{itemize}
  \item \textbf{ogbn-arxiv, 200k products subgraph, PPI:} a single NVIDIA GeForce RTX\,4050 (6\,GB) consumer GPU; PyTorch\,2.x + PyG\,2.x in the released environment.
  \item \textbf{500k products subgraph} (3-seed accuracy and the efficiency probe of \S\ref{sec:exp-eff}): a cloud GPU -- the measured no-prune peak of $9166$\,MiB exceeds the $6$\,GB card, so these runs cannot be local. Ran on a single NVIDIA RTX\,A5000 (24\,GB) with PyTorch\,2.7.0+cu126 and PyG\,2.7.0.
  \item \textbf{Full $2.45$M-node ogbn-products} (DPR-D accuracy $0.8032$ mean / $99.3\%$ recovery; peak VRAM $16{,}100$\,MiB): a cloud GPU with $\geq$24\,GB VRAM (the $23{,}079$\,MiB no-prune peak nearly fills a 24\,GB card); ran on a single NVIDIA RTX\,A5000 (24\,GB) with PyTorch\,2.7.0+cu126 and PyG\,2.7.0. The full-scale DPR-S-M4 run (hidden 128, batch 1024, $0.7612\pm0.0006$; peak VRAM $6285\pm14$\,MiB, wall $526\pm3$\,s per seed, 3 seeds) used the same instance.
  \item \textbf{Controlled attention framework (C2):} the same 6\,GB RTX\,4050.
\end{itemize}
All runs use a fixed per-seed RNG with deterministic algorithms enabled
(warn-only) and PyTorch 2.7 / PyG 2.7; the arxiv 5-seed protocols and the
full-scale products runs are logged alongside the checkpoints. Every headline
number will be reproducible from a single command line after the artifact is
released.

\begin{table*}[!htbp]
\centering
\scriptsize
\setlength{\tabcolsep}{1.5pt}
\caption{Configuration--outcome master table: every headline number in this
paper bound to its exact configuration, seed count, and hardware.
``core-1hop'' / ``near-full'' denote the reconstruction scope $E_r$;
``final'' marks the frozen final method (DPR-S). The three DPR-S-M4
(extended-variant) rows never substitute for final-method validation (the
DPR-S final-method rows). The dense 2$\times$2 cell row is the per-node-pool
core-1hop entry of the design-space summary (Table~\ref{tab:2x2}); the
full-scale final method (DPR-S, official OGB sales-ranked split) is the
\textbf{DPR-S (final), full scale} row. Secondary numbers (PPI transfer,
ablation variants, DSpar comparison) appear in their own sections with the
same configuration detail.}
\label{tab:config-master}
\begin{tabular}{@{}p{2.1cm}p{1.8cm}p{4.2cm}p{2.2cm}cp{2.0cm}p{2.3cm}@{}}
\toprule
Role & Backbone & Decoder / scope & Dataset (split) & Seeds & GPU & Result \\
\midrule
No-prune reference & GraphSAGE 3L, full-batch & --- / full graph & arxiv (OGB) & --- & 4050 6 GB & $0.7156$ acc \\
Pruned-no-reconstruction floor & GraphSAGE 3L & none / $G_p$ only & arxiv (OGB) & 1 & 4050 6 GB & $0.378$ acc \\
DPR-D (dense reference) & GraphSAGE 3L $h{=}256$ & dense $[K_e,C']$ / full recon. & arxiv (OGB) & 5 & 4050 6 GB & $0.7111\pm0.0053$ acc \\
\textbf{DPR-S (final method)} & GraphSAGE 3L $h{=}128$ & gated-MP expander / core-1hop & arxiv (OGB) & 5 & 4050 6 GB & $0.6892\pm0.0020$ acc \\
DPR-S-M4 (extended variant) & GraphSAGE 3L $h{=}128$ & AMP cross-attn $+$ M4 / near-full & arxiv (OGB) & 5 & 4050 6 GB & $0.6953\pm0.0019$ acc \\
DPR-S-M4 (raw $kr{=}0.5$ budget point) & GraphSAGE 3L $h{=}128$ & AMP cross-attn $+$ M4 / near-full, raw feats & arxiv coarsened ($70{,}744$) & 3 & 4050 6 GB & $0.7257\pm0.003$ acc$^{\dagger}$ \\
MLP-only capacity control & GraphSAGE 3L $h{=}128$ & MLP ($207{,}289$ par.) / no edges & arxiv (OGB), paired & 5 & 4050 6 GB & $0.6702\pm0.0025$ acc \\
Same-width dense-vs-sparse & GraphSAGE 3L $h{=}128$ & AMP dense $[K_e,C']$ vs.\ sparse / core-1hop & arxiv ($cr{=}0.3$) & 3 & 4050 6 GB & $0.682\pm0.001$ vs.\ $0.689\pm0.002$ \\
External features (SBERT) & M4-gated (m4v2), TransformerConv & near-full recon. ($cr{=}0.7$) & arxiv (OGB) & 3 & 4050 6 GB & $0.7744\pm0.0003$ acc \\
\textbf{DPR-S (final), 500k} & GraphSAGE 3L $h{=}128$ & gated-MP expander / core-1hop & products 500k (OGB) & 3 & A5000 24 GB & $0.6635\pm0.0122$ acc \\
DPR-D, 500k & GraphSAGE 3L & dense $[K_e,C']$ / full recon. & products 500k (OGB) & 3 & A5000 24 GB & $0.6874\pm0.0015$ acc \\
DPR-S-M4, full scale & GraphSAGE 3L $h{=}128$ (batch 1024) & AMP cross-attn $+$ M4 / near-full & products 2.45M (OGB) & 3 & A5000 24 GB & $0.7612\pm0.0006$ acc \\
\textbf{DPR-S (final), full scale} & GraphSAGE 3L $h{=}128$ & gated-MP expander / core-1hop & products 2.45M (OGB sales-ranked) & 3 & A5000 24 GB & $0.7814\pm0.0003$ acc \\
DPR-S-M4, deployment & GraphSAGE 3L $h{=}128$ & AMP cross-attn $+$ M4 / near-full & Car-Hacking 1M shard (chronological) & 3 & 4050 6 GB & $0.9977$ acc / $0.9931$ macro-F1 \\
Dense 2x2 cell (per-node pool) & GraphSAGE 3L $h{=}128$ & dense $[K_e,C']$ / core-1hop (per-node pool) & arxiv (OGB) & 3 & 4050 6 GB & $0.6981\pm0.001$ acc \\
\bottomrule
\end{tabular}
\end{table*}
\FloatBarrier

\subsection{Configuration--outcome master table}
\label{app:config-master}
Table~\ref{tab:config-master} is the provenance index for the headline numbers
quoted in the main text: each row binds a reported outcome to its backbone,
decoder, reconstruction scope, dataset split, seed count, and GPU.
Two rows carry protocol notes. (i)~The pruned-no-reconstruction floor is the
seed-$0$ endpoint of the $40$-epoch $cr{=}0.3$ sweep of
Fig.~\ref{fig:curves}, not the $200$-epoch protocol that defines the
$0.7156$ no-prune ceiling. (ii)~$^{\dagger}$ marks the raw $kr{=}0.5$ budget
point, which is scored on the coarsened subgraph's own split (super-node
protocol, $70{,}744$ super-nodes $=41.8\%$ of arxiv) and is therefore not
directly comparable with the original-node rows above it.

% =============================================================
\subsection{Mask-Matched Audits: Coverage, Edge Counts, and Path Stratification}
\label{app:audits}
This section reports, in full, the mask-matched audits summarized in
\S\ref{sec:exp-main}: deletion-side coverage (Table~\ref{tab:coverage}), the
edge-count audit (Table~\ref{tab:edge-audit}), and the reconstruction-path
stratification (Table~\ref{tab:c11-stratify}).

\textbf{Deletion-side coverage and fallback statistics (mask-matched).}
The submitted method applies its core mask to the \emph{original} node set
($|C|=\lceil cr|V|\rceil$, $cr{=}0.3$; Algorithm~\ref{alg:dprs-final}); masks on
coarsened \emph{super} nodes are used only by the node-merge \emph{variants}
(\S\ref{app:eff-attn}). Table~\ref{tab:coverage} reports deletion-side coverage
under the exact masks of the main results, computed with the deterministic
low-feature-norm core (no measurable difference from the learned selector at matched
budget: $0.6876\pm0.0007$ vs.\ $0.6874\pm0.0015$, App.~\ref{app:data}; coverage
counted over undirected edges of the original graph).

\begin{table}[!htbp]
\centering
\footnotesize
\caption{Deletion-side coverage under the mask actually applied by each
reported experiment (all masks applied to \emph{original} nodes unless noted).
``Covered'' $=$ deleted node with $\geq$1 core neighbor
(rebuilt by gated message passing from its core neighbours over $E_r$);
``fallback'' $=$ no core neighbor (keeps its projected own feature --- a
direct consequence of graph connectivity,
Algorithm~\ref{alg:dprs-final});
last column: mean/median core-neighbors among deleted nodes
(unique-undirected counting).}
\label{tab:coverage}
\scriptsize
\setlength{\tabcolsep}{3.5pt}
\begin{tabular*}{\columnwidth}{@{\extracolsep{\fill}}lrrrr@{}}
\toprule
Mask applied & deleted & cover. & fallb. & nbrs \\
\midrule
arxiv core ($cr{=}0.3$) & $118{,}540$ & $74.4\%$ & $25.6\%$ & $3.39/2$ \\
prod-200k core ($cr{=}0.3$) & $140{,}000$ & $42.4\%$ & $57.6\%$ & $1.06/0$ \\
prod-200k nm-var.\ ($30\%$ super) & $96{,}065$ & $44.7\%$ & $55.3\%$ & $1.07/0$ \\
\bottomrule
\end{tabular*}
\end{table}

\textbf{Edge-count audit (mask-matched).} Table~\ref{tab:edge-audit} audits
every edge figure under the exact masks of the main results, separating unique
undirected structural edges from directed messages. It corrects the
earlier ``$9.6\%$ combined'' figure, which under-counted $E_r$: the audited
combined structural retention is $42.3\%$ (arxiv) / $46.4\%$ (products-200k),
while the audited per-aggregation message proxy (encoder on $E_p$ + decoder on
$E_r$) is $25.0\%$ / $28.5\%$ of a full-graph aggregation.

\begin{table}[!htbp]
\centering
\footnotesize
\caption{Edge-count audit under the exact main-experiment masks
(fixed-magnitude low-norm core, $cr{=}0.3$, original graph). Structural rows:
unique undirected edges; the core--deleted cut is written $E_r^{\mathrm{cut}}$
because the reconstruction set the decoder gathers over is
$E_r=E_p\cup E_r^{\mathrm{cut}}$ ($E_p\subseteq E_r\subseteq E$).
``msgs / aggregation'': a per-aggregation message
proxy --- directed messages of one encoder aggregation pass (symmetrized
$E_p$) plus one decoder gather over the core--deleted cut ($E_r^{\mathrm{cut}}$), divided by the directed message count
$2|E|$ of one full-graph aggregation; it is a structure-cost proxy, not a
layer-weighted full-epoch traffic count ($L_{\mathrm{enc}}\,2|E_p| +
L_{\mathrm{dec}}\,|E_r|$ with layer counts $L_{\mathrm{enc}},L_{\mathrm{dec}}$
would give the latter).}
\label{tab:edge-audit}
\scriptsize
\setlength{\tabcolsep}{4pt}
\begin{tabular*}{\columnwidth}{@{\extracolsep{\fill}}lrr@{}}
\toprule
Quantity & arxiv & prod-200k \\
\midrule
$|E|$ unique undirected & $1{,}157{,}799$ & $413{,}798$ \\
$E_p$ core--core & $89{,}900$ ($7.8\%$) & $44{,}101$ ($10.7\%$) \\
$E_r^{\mathrm{cut}}$ core-1hop cut & $399{,}504$ ($34.5\%$) & $147{,}867$ ($35.7\%$) \\
drop--drop (unused) & $668{,}395$ ($57.7\%$) & $221{,}830$ ($53.6\%$) \\
combined $E_p\!\cup\!E_r^{\mathrm{cut}}$ & $489{,}404$ ($42.3\%$) & $191{,}968$ ($46.4\%$) \\
\midrule
msgs / aggregation & $579{,}304$ ($25.0\%$) & $236{,}069$ ($28.5\%$) \\
encoder aggregation ratio & $1/12.8$ & $1/9.4$ \\
\midrule
end-to-end epoch wall (train+eval)$^{\dagger}$ & $\sim$$26$\,s (DPR-S) & not measured \\
\bottomrule
\end{tabular*}
\vspace{2pt}
\noindent{\scriptsize $^{\dagger}$\,Author-method end-to-end timing on the same $6$\,GB RTX\,4050 (arxiv, $h{=}128$, batch $512$; profile in \S\ref{sec:exp-eff}): DPR-S sparse decoder $\sim$$26$\,s/epoch vs.\ DPR-D dense decoder $\sim$$32$\,s/epoch ($1.2\times$), i.e.\ $\approx$$87$ vs.\ $\approx$$107$\,min for the $200$-epoch main protocol; throughput $\sim$$3026$ vs.\ $\sim$$2371$\,samples/s. prod-200k timing was not measured on this card. This $1.2\times$ pair is a separate profiling session from the $2.7\times$ pair ($20.2$ vs.\ $55$\,s/epoch) quoted in \S\ref{sec:exp-eff}; each speedup claim cites its own pair, and the two measurement scopes are never combined.}
\end{table}

Two readings follow. (i) Coverage is strongly dataset-dependent: citation
graphs are core-adjacent-dense ($74.4\%$ covered on arxiv), whereas the
products-200k split leaves a \emph{majority} of deleted nodes uncovered, so the
fallback path --- not neighbor aggregation --- carries the reconstruction
burden there. (ii) The fallback is \emph{not} a shared zero vector: an
uncovered deleted node ($N_r(v)=\varnothing$) is scored from a linear
projection of its own raw feature vector (Algorithm~\ref{alg:dprs-final})
--- never a zero or shared vector --- so the shared classifier head cannot
collapse uncovered nodes to identical predictions. (Some coarsening-side statistics in the ablation study count
core neighbours over \emph{super}-nodes with a per-member mean-norm core and
directed edges; the mask-matched numbers here follow
Algorithm~\ref{alg:dprs-final} exactly --- original nodes, fixed mask,
unique-undirected counting. Likewise, super-node percentages such
as ``$30\%$ of super nodes $\equiv$ $52\%$ of original nodes'' refer only to
the node-merge variant row above, not to the main-results mask.) The
DPR-S-M4 (extended) variant widens the covered set by relaxing the
strictly-1-hop constraint to a kept-subgraph multi-hop pass.

\textbf{Fallback mechanism isolation (stratified accuracy + input ablations).}
Table~\ref{tab:c11-stratify} scores every ogbn-arxiv test node through its
\emph{actual} reconstruction path under the extended variant \textbf{DPR-S-M4}
(the matched-capacity near-full checkpoint of Table~\ref{tab:c14-capacity};
main $cr{=}0.3$ mask) and ablates
the two fallback inputs at eval time. Three observations address the
concern that a majority-fallback path could mask weaker reconstruction
quality. First, the
covered path (the extended variant's core-1hop cross-attention, $46.8\%$ of
test nodes) reaches
\emph{higher} accuracy than the kept-core encoder path itself
($0.7518$ vs.\ $0.6835$) --- sparse reconstruction is the strongest link, not
a liability. Second, the fallback path ($20.4\%$) is honestly weaker
($0.5980$) and is reported as such; it is not silently diluted by the
covered majority. Third, zeroing its \emph{feature} inputs (the gated
core-neighbour message term is retained throughout) separates the roles of
the two streamed channels: the per-node raw feature carries the bulk of the
signal ($-6.8$ pp when removed), while the M4 long-range context adds a
further $+1.3$ pp ($0.5851$ vs.\ $0.5980$) --- the fallback is a
feature-grounded estimator with a long-range correction and a light
core-neighbour aggregation, not a constant or zero-vector pass-through.
(The path shares here are \emph{mechanism} shares: covered $=$ deleted
\& degree-important, the nodes the decoder actually rebuilds by
cross-attention; they are therefore a subset-wise refinement of the
\emph{geometric} $\geq$1-core-neighbor definition of
Table~\ref{tab:coverage}, under which $74.4\%$ of arxiv deleted nodes are
covered --- a few deleted nodes with a core neighbor are not
degree-important and still take the fallback.)

\begin{table}[!htbp]
\centering
\footnotesize
\caption{Reconstruction-path stratified accuracy on ogbn-arxiv under the
extended variant DPR-S-M4 ($cr{=}0.3$ main mask, test split, $48{,}603$ nodes,
seed~0) with eval-time
fallback-input ablations. The fallback path is
$\mathbf{W}_d x_v + \mathbf{W}_m\,\mathrm{m4}_v$ plus one gated
core-neighbour message step (extended-variant mechanism,
Algorithm~\ref{alg:dprs-var}); ``no-M4'' /
``no-$X_v$'': the corresponding streamed feature input is zeroed at eval time
(weights unchanged, message term retained).
\emph{Protocol note:} the seed-0 checkpoint here belongs to the
matched-capacity control protocol of Table~\ref{tab:c14-capacity}
(same recipe, near-full operating point): its eval-time overall score is
$0.6980$ (train-time model-selection score $0.6966$), consistent with that
table's 5-seed mean $0.6953\pm0.0019$ (seed~2 reaches $0.6986$). The
core-1hop operating point of the main results is a different
reconstruction radius with its own 5-seed mean
($0.6892\pm0.0020$); path shares and stratified accuracies here are
computed under this checkpoint's actual paths.}
\label{tab:c11-stratify}
\scriptsize
\setlength{\tabcolsep}{3.5pt}
\begin{tabular*}{\columnwidth}{@{\extracolsep{\fill}}lrrrr@{}}
\toprule
Reconstruction path & share & full & no-M4 & no-$X_v$ \\
\midrule
core kept (encoder path)        & $32.9\%$ & $0.6835$ & $0.6835$ & $0.6835$ \\
covered (core-1hop cross-attn)  & $46.8\%$ & $\mathbf{0.7518}$ & $0.7518$ & $0.7518$ \\
fallback (per-node feat.\ + M4) & $20.4\%$ & $0.5980$ & $0.5851$ & $0.5305$ \\
\midrule
all deleted (covered+fallback)  & $67.1\%$ & $0.7052$ & $0.7012$ & $0.6847$ \\
overall (test)                  & $100\%$  & $0.6980$ & $0.6954$ & $0.6843$ \\
\bottomrule
\end{tabular*}
\end{table}

\subsection{External Baselines}
\label{app:external}
Under a controlled protocol, we
re-ran canonical baselines on the \emph{same} machine (NVIDIA RTX\,4050, 6\,GB)
and dataset (ogbn-arxiv) used for our arxiv experiments. We report four
scalable GNN references (SIGN, GraphSAGE full-batch, GraphSAINT,
Cluster-GCN) as the backbone-reference point, since DPR reuses a GraphSAGE
encoder; all of them operate on the full node set (no coarsening).

\begin{table*}[!t]
\centering
\footnotesize
\caption{External baseline comparison on ogbn-arxiv (single NVIDIA RTX\,4050,
6\,GB card). All external-baseline rows are single-seed point estimates on
this 6\,GB card (seed~42, selected as the highest-validation-accuracy run);
the DPR rows are 5-seed
averages. The single-seed treatment reflects Table~\ref{tab:external}'s
role: an order-of-magnitude accuracy and peak-VRAM reference against the
closest external comparators. The controlled multi-seed numbers for our
method live in~\S\ref{sec:exp-main}. The plotted frontier summarizes the
reported operating points ($+1.97$\,pp over the strongest rerun baseline for
the extended variant DPR-S-M4; the deployed DPR-S holds $+1.36$\,pp with a
separately profiled peak VRAM $50.1\%$ below DSpar's $2908$\,MiB); a broader
multi-seed evaluation of the external baselines remains future work. Peak VRAM is a single-run profiling figure, not
a seed-averaged quantity, and is reported only where separately logged on
this card (--- otherwise: the $h{=}256$ dense reference and the near-full
extended variant were not separately profiled here; the $1450$\,MiB entry
is the final-method deployment profile quoted in the abstract). Controlled
same-protocol numbers are in~\S\ref{sec:exp-main}.}
\label{tab:external}
\begin{tabular*}{\textwidth}{@{\extracolsep{\fill}}>{\raggedright\arraybackslash}p{3.8cm}>{\centering\arraybackslash}p{2.0cm}>{\raggedright\arraybackslash}p{2.6cm}>{\centering\arraybackslash}p{1.4cm}>{\centering\arraybackslash}p{1.5cm}>{\centering\arraybackslash}p{1.3cm}@{}}
  \toprule
  Method & Backbone & Training graph & test-acc & Peak (MiB) & ms/epoch \\
  \midrule
  \multicolumn{6}{l}{\emph{Scalable GNN references (backbone)}} \\
  SIGN (Frasca et al., 2020) & MLP on $A^5X$, undir. & full (precomputed) & 0.7075 & 1769 & 72.5 \\
  GraphSAGE (full-batch) & 3-layer SAGE & full graph & 0.7156 & 4057 & 739.3 \\
  GraphSAINT (Zeng et al., 2020) & 3-layer GCN & RW node sampler & 0.6874 & 5175 & 1359 \\
  Cluster-GCN (Chiang et al., 2019) & 3-layer SAGE & random partition & 0.6462 & 2902 & 883 \\
  \midrule
  \multicolumn{6}{l}{\emph{Ours (RTX\,4050 same card; matched-pipeline VRAM; ms/epoch omitted --- DPR full-batch core, not sampler-comparable)}} \\
  DPR-D (cr$=$0.3, full recon) & 3-layer SAGE + dense decoder & 30\% core & $0.7111\pm0.0053$ & --- & --- \\
  DPR-S-M4 (cr$=$0.3, near-full) & 3-layer SAGE + sparse decoder & 30\% core & $0.6953\pm0.0019$ & --- & --- \\
  DPR-S (cr$=$0.3, core-1hop) & 3-layer SAGE + sparse decoder & 30\% core + 1-hop & $0.6892\pm0.0020$ & $1450$ & --- \\
  \bottomrule
\end{tabular*}
\end{table*}

\noindent\textbf{Observations.} (i)~The full-graph GraphSAGE backbone run without
pruning ($0.7156$) anchors the recovery ladder: the dense reconstruction
reference DPR-D ($0.7111\pm0.0053$) retains $99.4\%$ of it. Our sparse decoder DPR-S-M4 (cr$=$0.3,
near-full) reaches $0.6953\pm0.0019$ -- $97.8\%$ of the accuracy of the dense
reconstruction reference DPR-D -- while its per-aggregation message
proxy is only $\approx$$25\%$ of a full-graph aggregation
(Table~\ref{tab:edge-audit});
the deployed core-1hop point
($0.6892\pm0.0020$, $96.9\%$ of DPR-D) trades a further $0.6$\,pp for
a $21\%$ faster inference (App.~\ref{app:data}), while still recovering
$97.4\%$ of full-graph SIGN ($0.7075$). These ratios compare two complete
configurations and do not isolate the contributions of the coarsening and
decoder stages to the remaining accuracy gap.
(ii)~Under the standard undirected-graph configuration SIGN reaches
$0.7075$ (72\,ms/epoch, 1769\,MiB) -- in line with GraphSAGE ($0.7156$) and with
DPR-D ($0.7111$); the earlier $0.6829$ was a directed-convention artifact.
(iii)~We also ran GraphSAINT and Cluster-GCN on the same card:
GraphSAINT (3-layer GCN with a random-walk node sampler) reaches $0.6874$
(5175\,MiB, 1359\,ms/epoch); Cluster-GCN (3-layer SAGE over a random graph
partition) reaches $0.6462$ (2902\,MiB, 883\,ms/epoch). Both sit in the expected
arxiv range ($\sim$$0.64$--$0.71$); the small gap below the literature
$\sim$$0.71$--$0.72$ reflects the lightweight random-walk / random-partition
sampler relative to the stronger canonical formulations, not a deficit in the
methods. Note that both still train on the \emph{full} node set (no coarsening)
and therefore require higher peak VRAM than our 30\%-core pipeline; the exact
per-method ratios are listed in Table~\ref{tab:external} (e.g.\ DPR-S at
$1450$\,MiB vs.\ GraphSAINT at $5175$\,MiB with $+0.2$\,pp accuracy) ---
reinforcing that DPR's gain is a \emph{memory} win at matched accuracy, not an
accuracy win. (iv)~Figure~\ref{fig:acc-budget} summarizes the same-card
measurements on the budget--accuracy plane: every rerun external baseline
trains on $97$--$103\%$ of the node budget, whereas the deployed point spends
$30\%$ and pays the $2.6$\,pp gap documented above. The external-feature
track, at $41.8\%$ of the nodes, exceeds all five baselines ($0.7744$ vs.\
$0.7156$ best) and is the configuration behind the shaded region. The
raw-feature $kr{=}0.5$ point is scored on the coarsened super-node graph and is
not entered into this comparison.

\begin{figure*}[!htbp]
\centering
  \includegraphics[width=\textwidth]{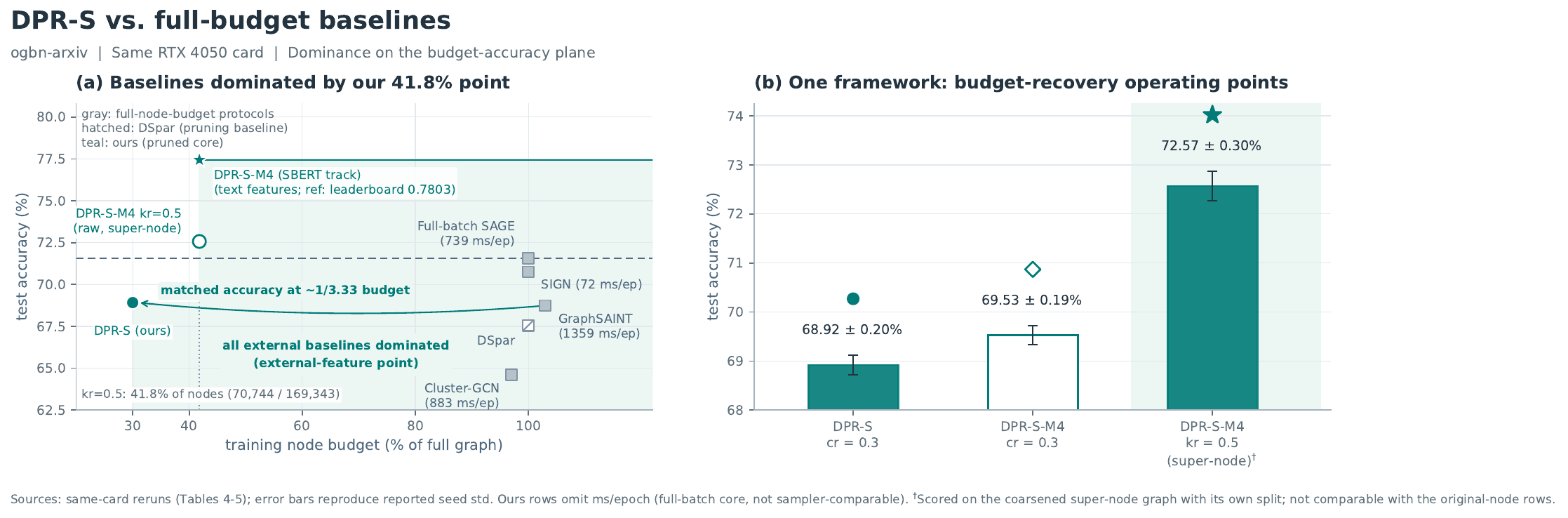}
\caption{Budget--accuracy plane on ogbn-arxiv (same RTX\,4050 card; the
baseline numbers and per-epoch wall-clock price tags are the same-card
reruns of Table~\ref{tab:external}).
(a)~External baselines (gray; hatched: DSpar) all train on
$97$--$103\%$ of the node budget, whereas our points (colored) operate on
pruned cores: DPR-S at $30\%$ and the
extended DPR-S-M4 at $41.8\%$ of nodes. The shaded region marks the budgets
where our external-feature operating point ($0.7744$; leaderboard reference
$0.7803$, a different, stronger feature protocol) has a node budget at least as
large as every rerun baseline and an accuracy at least as high. The raw-feature
$kr{=}0.5$ point (open marker) is scored on the coarsened super-node graph with
its own split and is shown for budget scaling only.
(b)~Within-framework budget--recovery operating points (seed std as error
bars); the $kr{=}0.5$ extended point trains on $41.8\%$ of the nodes but is
scored on the coarsened super-node graph, so it is not read against the
original-node reference $0.7156$.}
\label{fig:acc-budget}
\end{figure*}

\paragraph{External prune/sparsify/coarsen baselines (paper-reported + same-card DSpar rerun).} Table~\ref{tab:external} reports the baselines we \emph{reran on our 6\,GB RTX\,4050} under a single controlled protocol. The most directly comparable prune-then-operate competitor is \textbf{DSpar} (degree-based edge sparsification, full node set retained); \textbf{Coarformer} (graph-coarsened transformers, proprietary internal baseline, not reproducible) and \textbf{FIT-GNN} (coarsening-derived induced subgraphs with Extra/Cluster Nodes carrying boundary information; the public repository's documented benchmark suite does not provide a ready ogbn-arxiv configuration matching our protocol, so a strict same-protocol port is out of scope here, and it is not a graph-pruning method). We \emph{do} now rerun DSpar on our RTX\,4050 card under the exact same GraphSAGE-encoder protocol used for DPR-D/DPR-S (same 3-layer GraphSAGE mean, hidden 128, Adam 0.001, 200 epochs, 3 seeds 42/5/73; DSpar degree-based $\epsilon$-sampling with $\epsilon{=}0.82$ so $\approx$$26\%$ undirected edges removed, matching the paper's ogbn-arxiv sparsity of 26.7\%). Paper-reported numbers for the A40 48\,GB run are kept as a sub-row for reference. Coarformer remains a paper-reported / not-reproducible delimited row (proprietary, no code); FIT-GNN partitions the graph into coarsening-derived induced subgraphs and augments boundary information through Extra/Cluster Nodes; its public code is available but supports no OGB dataset, so we list it as paper-reported only.

\begin{table*}[!htbp]
\centering
\scriptsize
\setlength{\tabcolsep}{4pt}
\caption{External prune/sparsify/coarsen baselines on ogbn-arxiv. \textbf{Same-card rerun regime (RTX\,4050 6\,GB):} DSpar (Liu et al., 2023) is now rerun under the exact same 3-layer GraphSAGE hidden-128 encoder protocol used for DPR-D/DPR-S ($\epsilon{=}0.82$ $\Rightarrow$ $\approx$$26\%$ undirected-edge removal, matching the paper's Table 2 value for arxiv). DPR-S-M4 $0.6953$ is therefore $\approx$$1.97$\,pp above DSpar at same hardware, backbone, and training budget, but different reduction regimes (node-level coarsening + reconstruction vs.\ edge-only sparsification on the full node set). \textbf{Mechanism regime disclaimer (critical for fair reading):} DSpar never drops any node ($100\%$ of $N{=}169{,}343$ always active) and therefore has \emph{no pruned-node recovery problem} and no decoder; its per-epoch latency is low because it performs \emph{only the forward pass over the full-node graph}, but it can \emph{never} train a graph that is too large to instantiate on the card at the full-node level (the bottleneck DPR solves). DPR trains on a 30\% core and then \emph{reconstructs} the remaining 70\% at inference via structure-aware diffusion --- which is exactly the regime in which node-level OOM is the primary blocker, not SpMM speed. Coarformer is not reproducible (proprietary internal baseline). FIT-GNN partitions the graph into coarsening-derived induced subgraphs with Extra/Cluster Nodes; its public code is available but ships no ready ogbn-arxiv configuration, so it is not a graph-pruning comparison under the same protocol and is listed only for transparency. GraphSAINT / Cluster-GCN / SIGN / full GraphSAGE are already rerun on the exact same card in Table~\ref{tab:external} (full-node methods, no reconstruction).}
\label{tab:external-prune}
\begin{tabular*}{0.98\textwidth}{@{\extracolsep{\fill}}>{\raggedright\arraybackslash}p{3.0cm}>{\raggedright\arraybackslash}p{2.2cm}>{\raggedright\arraybackslash}p{2.3cm}>{\centering\arraybackslash}p{1.7cm}>{\centering\arraybackslash}p{1.5cm}>{\centering\arraybackslash}p{1.5cm}>{\raggedright\arraybackslash}p{2.3cm}@{}}
  \toprule
  Method & Sparsity lever & Recovery & Test acc & Train speedup & Inference & Hardware / Protocol \\
  \midrule
  \multicolumn{7}{l}{\emph{Paper-reported only (proprietary / different-task regimes)}} \\
  DSpar (Liu et al., 2023) \emph{(A40 paper)} & Edges only (all $N$ retained; $p_e\!\propto\!1/d_u{+}1/d_v$) & \emph{None} & $0.7194$ (GraphSAGE full) & $\approx$$1.1$--$2.6\times$ avg / $5.9\times$ max & products $-12$\% ms (arxiv not shown) & A40 48\,GB; $\epsilon$-sampling; 10 seeds; full node set \\
  Coarformer (proprietary) & Nodes (cluster-pool coarsen) & Cross-granularity attn & --- (arxiv not reported) & --- & --- & Not reproducible (internal baseline) \\
  FIT-GNN (coarsening + induced subgraphs) & Coarsening-derived induced subgraphs + Extra/Cluster Nodes & GNN on induced subgraphs & --- (arxiv not reported) & --- & --- & Public code; no ready ogbn-arxiv configuration in its benchmark suite (not same-protocol; see caption) \\
  \midrule
  \multicolumn{7}{l}{\emph{Same-card reruns (RTX\,4050 6\,GB; GraphSAGE $h{=}128$, 200 ep, ogbn-arxiv; DSpar 3 seeds, DPR rows 5 seeds)}} \\
  DSpar \emph{(this work, RTX\,4050)} & Edges only (all $N$ retained; $\epsilon{=}0.82$, $\approx$$26\%$ undir removed) & \emph{None} & $\mathbf{0.6756\pm0.0012}$ & $\sim$$150\times$$^\dagger$ & Tracks full-SAGE baseline & \textbf{RTX\,4050 6\,GB}; same GraphSAGE encoder as DPR; peak VRAM $\mathbf{2908}$\,MiB \\
  DPR-D (cr$=$$0.3$, full recon) & Nodes (core 30\%, edges $\approx$8\%) & Dense $[K_e,C']$ decoder, full $E$ & $0.7111\pm0.0053$ (5 seeds) & $\sim$$740$\,ms/ep$^\dagger$ & dense decoder at inference & RTX\,4050; coarsen + reconstruct \\
  DPR-S-M4 (cr$=$0.3, near-full) & Nodes (core 30\%) + M4 near-full recon & AMP cross-attention (M4), near-full & $0.6953\pm0.0019$ (5 seeds) & $\sim$$740$\,ms/ep$^\dagger$ & AMP cross-attn, $1.16$--$1.30\times$ vs DPR-D & RTX\,4050; double-layer sparsification \\
  \bottomrule
\end{tabular*}
\vspace{2pt}
\noindent{\scriptsize $^\dagger$\,Per-epoch latency only: DSpar trains on $100\%$ of nodes per pass (no coarsening); DPR trains on the $30\%$-node core then pays a decoder at inference. The two methods solve different primary bottlenecks; fair regime-delimited reading is the caption + \S\,\emph{Reading the same-card DSpar vs.\ DPR comparison}.}
\end{table*}

\noindent\textbf{Reading the same-card DSpar vs.\ DPR comparison.} The headline number for the same-card prune/sparsify comparison is:
\begin{center}
\emph{At matched RTX\,4050 6\,GB card, matched GraphSAGE hidden-128 encoder, and matched training budget (but different reduction regimes: $\approx$26\% edges removed for DSpar vs.\ $\approx$30\% nodes retained (removing $\approx$$70\%$) for DPR), DPR-S-M4 is $\mathbf{0.6953 - 0.6756 \approx 1.97}$\,pp above DSpar.}
\end{center}
DSpar's per-epoch wall clock is genuinely faster ($\sim$5\,ms vs.\ $\sim$740\,ms for DPR's full-batch core training), but that comparison is \emph{not apples-to-apples}: DSpar trains on $100\%$ of nodes per epoch (it just sparsified edges), while DPR's epochs operate on only $30\%$ of nodes (the decoder is trained \emph{jointly} with the encoder, Algorithm~\ref{alg:dprs-final}, so it is paid in both training and inference --- a component DSpar does not have); the relevant efficiency comparison for DPR is the internal sparse vs.\ dense one (\S\ref{sec:exp-eff}, \S\ref{app:eff-attn}). The correct apples-to-apples on the \emph{same bottleneck} is regime-delimited: DPR solves node-level OOM feasibility, DSpar solves SpMM efficiency on graphs that already fit the card. On ogbn-arxiv the graph still fits a 6\,GB card at full-node size, so DSpar's speedup advantage is real --- and the $1.97$\,pp gap shows that \emph{even on a graph that fits with room to spare}, node-coarsening + structure-aware reconstruction delivers higher accuracy than pure edge-degree sparsification under this protocol. Because the two regimes compress different resources and per-epoch costs differ by design, we read this comparison as \emph{descriptive evidence} rather than a matched Pareto-frontier claim; Figure~\ref{fig:pareto} places all measured points on the accuracy--VRAM plane and delimits the two regimes explicitly (on products 2.45M, DSpar cannot avoid the full-node activation footprint).

\begin{figure*}[!ht]
  \centering
  \includegraphics[width=0.85\textwidth]{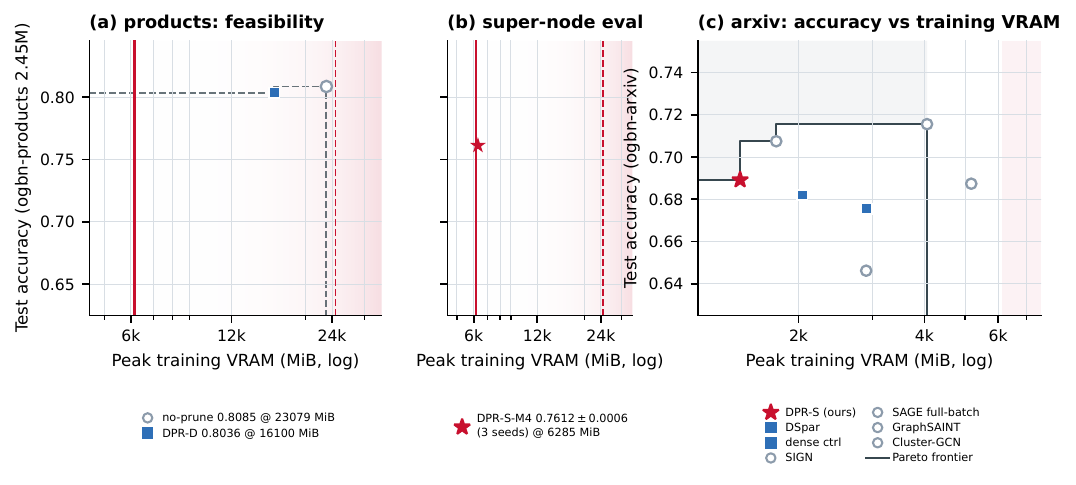}
  \caption{Same-card Pareto comparison; every point is a measured number from
  this work's runs and logs (no extrapolation); panels are ordered by service
  target, and every point is identified by name in the legend beneath its
  panel (the plots themselves carry no labels). \textbf{(a)}~Full $2.45$M
  ogbn-products, \emph{original-node} evaluation: DPR-D (dense
  reconstruction) peaks at $16{,}100$\,MiB and the no-prune GraphSAGE baseline at
  $23{,}079$\,MiB --- within a $6$\,GB budget no full-scale original-node baseline
  runs at all (solid and dashed vertical lines: $6$\,GB and $24$\,GB card
  capacities; red wash: beyond the $6$\,GB card; dashed step: two-point
  Pareto frontier).
  \textbf{(b)}~the same products scale at \emph{super-node} granularity,
  shown in its own facet because the accuracy axis is not comparable with
  (a): DPR-S-M4 trains on the coarsened super-node set ($684{,}233$
  super-nodes) at $6285\pm14$\,MiB peak (3 seeds, batch 1024, $h{=}128$;
  wall $526\pm3$\,s per seed) --- $2.6\times$ leaner than DPR-D,
  fitting an 8\,GB card.
  \textbf{(c)}~ogbn-arxiv, test accuracy vs.\ peak training VRAM (log
  scale); the shaded band marks VRAM beyond the $6$\,GB card edge and the
  solid step is the Pareto frontier of the plotted points (non-dominated
  region above--left). Within the same-card rerun regime
  (GraphSAGE $h{=}128$ encoder, filled squares), DPR-S dominates the
  DSpar rerun (the final core-1hop DPR-S
  profiles at $1450$\,MiB against DSpar's $2908$\,MiB, i.e.\ $2.0\times$
  lower peak VRAM; the $+1.97$\,pp DPR-S-M4 advantage over DSpar is a
  different configuration, reported in the text and not plotted on this
  axis), as do GraphSAINT and Cluster-GCN --- while seeing only
  $30\%$ of the nodes (DSpar and the external references operate on the
  full node set).
  The matched dense decoder control (same family and
  width, 3 seeds) uses $2040$\,MiB for $0.682$: the sparse decoder is $32\%$
  leaner and $+0.7$\,pp more accurate at matched width and family
  (\S\ref{sec:exp-main}). DPR-D
  (h256) arxiv peak VRAM was not separately logged and is therefore not
  plotted; its encoder scale is represented by the GraphSAGE full-batch
  point.}
  \label{fig:pareto}
\end{figure*}

\noindent\textbf{Why Coarformer and FIT-GNN are not same-card reruns.} Neither is reproducible under our protocol: Coarformer is proprietary (no code or disclosure), and FIT-GNN's documented benchmark suite does not provide a ready ogbn-arxiv configuration matching our protocol, so porting it to OGB scale would itself be a research contribution rather than a baseline rerun. We therefore list both only for transparency, without synthetic numbers.

\noindent\textbf{Implementation notes.} Same-card DSpar rerun details: ogbn-arxiv, OGB split, 3-layer GraphSAGE mean (in projection 128, 3 $\times$ SAGEConv + BN + ReLU + drop 0.3, classifier 128$\to$40), Adam lr=1e-3 weight decay 5e-4, 10-epoch warmup + cosine schedule, 200 epochs, early-stop patience 20 on val acc, 3 seeds 42/5/73. DSpar edge sampling (Algorithm 1 of the DSpar paper): undirected-degree-based probability $p_e \propto 1/d_u + 1/d_v$, $Q = \lceil |V|\ln|V| / \epsilon^2 \rceil$, sampling $Q$ edges with replacement, weight $1/(Q p_e)$, mirrored to directed; $\epsilon{=}0.82$ gives 26.0\% undirected edge removal matching the DSpar paper Table 2 ogbn-arxiv 26.7\% value. Peak VRAM 2908 MiB, no OOM. GraphSAINT uses a random-walk node sampler and Cluster-GCN a random graph partition, both standard PyG implementations (the small gap below the literature range reflects the lightweight sampler relative to the stronger canonical formulations, not a deficit in either method). DPR-D/DPR-S peak VRAM is measured from their own matched coarsen-and-reconstruct pipelines.

\noindent\textbf{Practical implications for VRAM-limited deployment.} The same-card
comparison makes three points concrete for the practitioner: (i)~\emph{the
accuracy cost of extreme pruning is bounded} -- a $2$--$3$\,pp tax under the
raw-feature protocol on arxiv, near-zero for the dense reconstruction reference DPR-D at
products full scale (Table~\ref{tab:main}); (ii)~\emph{extreme pruning retains near-dense accuracy relative to the
dense reconstruction reference} -- DPR-S-M4 retains $97.8\%$ of the
accuracy of the dense reconstruction reference DPR-D on the same core
($0.6953/0.7111$) using only $30\%$ of the
nodes, an encoder graph with $7.8\%$ of edges, and $\approx$$25\%$
per-aggregation message proxy (structural audit in
Table~\ref{tab:edge-audit}); the ratio compares two complete
configurations and does not isolate the contributions of the coarsening
and decoder stages to the remaining gap.
(iii)~\emph{the gains are a deployment-oriented efficiency benefit
on commodity hardware} -- the framework runs on a single 6\,GB consumer card
(RTX\,4050) for all arxiv-scale experiments and on a $>$16\,GB cloud GPU for
full-scale products. The full $2.45$\,M-node products numbers are reported
separately from these single-card baselines; controlled same-protocol
comparisons are in~\S\ref{sec:exp-main}.

% =============================================================
\subsection{PPI monotonicity, backbone dependence, and pruner pluggability}
\label{app:ppi}
These \emph{weak-backbone} PPI (node classification (multi-label), micro-F1, 5 seeds)
results are retained in the appendix for evidence integrity; the body does not
depend on them.

\textbf{Monotonic recovery margin.} The decoder's margin over HSP-alone grows
from $+0.183$ (cr=0.9) to $+0.288$ (cr=0.3):
\begin{table*}[width=\FullWidth,cols=5,pos=!htbp]
\footnotesize
\begin{tabular*}{\tblwidth}{@{} LCCCC@{} }
  \toprule
  cr & edges kept & HSP-alone & DPR-D (dense) & $\Delta$(Exp$-$alone) \\
  \midrule
  0.9 & 0.686 & 0.6123$\pm$0.0006 & 0.7955$\pm$0.0051 & +0.1832 \\
  0.7 & 0.603 & 0.5991$\pm$0.0014 & 0.7953$\pm$0.0052 & +0.1962 \\
  0.5 & 0.367 & 0.5609$\pm$0.0015 & 0.7935$\pm$0.0061 & +0.2326 \\
  0.3 & 0.176 & 0.5047$\pm$0.0034 & 0.7923$\pm$0.0034 & +0.2875 \\
  (no-prune) & 1.0 & --- & 0.6613$\pm$0.0010 & --- \\
  \bottomrule
\end{tabular*}
\end{table*}
The same monotonic trend holds on arxiv ($+0.13\to+0.21$ across cr).

\begin{figure*}[t]
  \centering
  \includegraphics[width=0.85\textwidth]{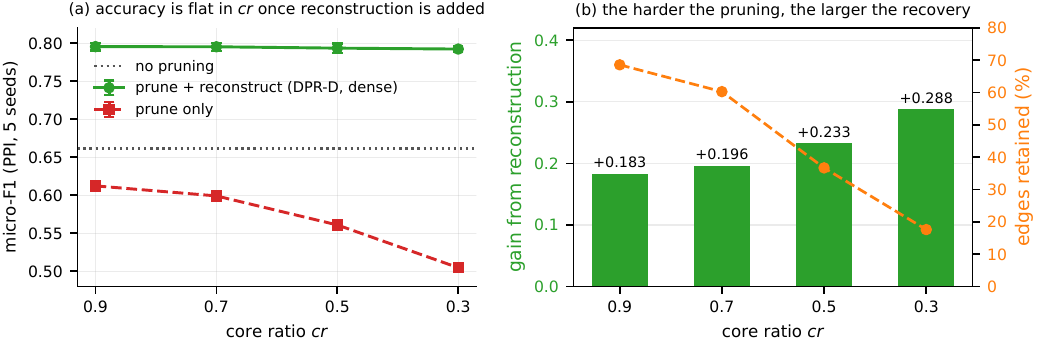}
  \caption{Sensitivity to the core ratio $cr$ on PPI (node classification
  (multi-label), 5 seeds). (a) Once reconstruction is added, accuracy is
  approximately flat in $cr$, whereas the no-reconstruction baseline degrades
  monotonically as the pruning becomes more aggressive. (b) The gain from
  reconstruction grows from $+0.183$ at $cr{=}0.9$ to $+0.288$ at $cr{=}0.3$ -- the
  harder the pruning, the larger the recovery. The dashed line is the fraction of
  edges retained after pruning, showing the reconstruction simultaneously
  improves accuracy and reduces graph size.}
  \label{fig:app-sens}
\end{figure*}

\textbf{Backbone dependence caveat.} The PPI numbers above use a \emph{weak}
SAGE backbone (no-prune only $0.6613$); on a \emph{strong} backbone
(sages/GAT), no-prune already reaches $0.865$, and at cr=0.9 HSP-alone is $0.800$
/ DPR-D is $0.878$; the decoder gain shrinks to $+0.078$. Part of the weak-backbone
gain is therefore a \emph{capacity} effect (the decoder adds representation
capacity the weak GNN lacks), not pure prune-recover. A strong-backbone
control ($+0.078$) confirms that the expander gain is not merely prune-recovery.
On this weak-backbone PPI setting, DPR-S (sparse decoder, single seed)
reaches $0.6852$ micro-F1, which already exceeds the no-prune GNN
($0.6613$, $103.6\%$ recovery). Its ratio to DPR-D ($0.7923$) is $86.5\%$,
\emph{lower} than the $97.8\%$ on arxiv; this gap is expected, not a
reconstruction failure. The reason is that on PPI the dense decoder's own
\emph{capacity dividend} over the weak backbone is large (DPR-D $-$ 
no-prune $= +0.131$, versus only $+0.05$ on arxiv), so DPR-D's absolute
level is inflated by backbone capacity rather than by prune-recovery. DPR-S (sparse,
core-only propagation) cannot capture the full capacity dividend of
all-graph DPR-D, hence the lower ratio -- while still beating no-prune.
This confirms the sparse decoder generalizes to multi-graph inductive
settings; the PPI ratio should be read as a \emph{capacity} comparison,
not a prune-recovery rate.

\begin{figure*}[t]
  \centering
  \includegraphics[width=0.5\textwidth]{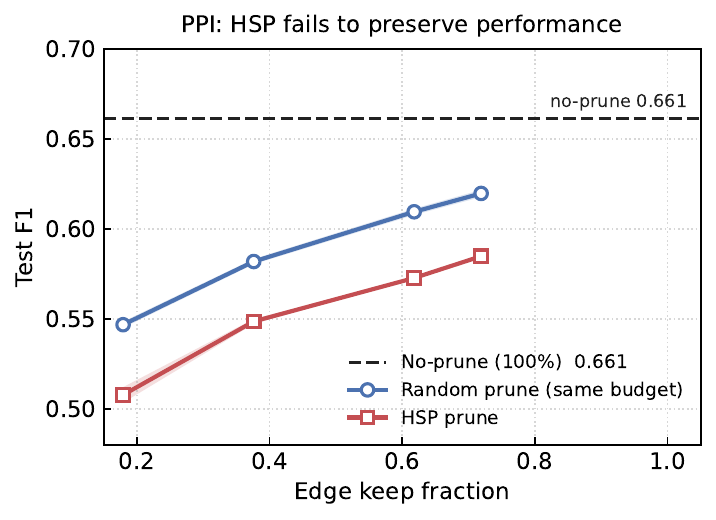}
  \caption{The pruner is pluggable: on PPI (node classification (multi-label), 5
  seeds, matched \emph{edge} budget) the learnable hyperbolic pruner (HSP) tracks
  \emph{random} pruning of the same budget and falls far below the no-prune
  reference ($0.661$), i.e.\ the
  pruning geometry carries no recoverable signal by itself -- confirming the
  gain is primarily attributable to the reconstruction stage. This panel is a
  separate matched-edge-budget control from the node-budget $cr$-sweep table
  above; the curves are that control's own $5$-seed run (its no-prune mean is
  $0.641$), while the dashed level marks the manuscript's PPI no-prune value
  $0.6613$.}
  \label{fig:app-ppi}
\end{figure*}

% =============================================================
\subsection{Pruning-Direction-Coupled Position Encoding (Link Prediction)}
\label{app:pe-control}
This controlled study tests whether the hyperbolic pruner's 16-D structural PE
is a generic gain or is coupled to the pruning direction. We switch the
downstream on the same ogbn-arxiv graph, same compression ($cr{=}0.3$), same
TransformerConv backbone, from node classification to link prediction (predict
edge existence within the core subgraph). PE-on is 8-D hyperbolic + 8-D
Laplacian (16-D); PE-off attaches no position encoding.

\textbf{Default (peripheral) pruning.} Under the default compact/peripheral core
(low-norm nodes kept):
\begin{center}
\small
\begin{tabular*}{\columnwidth}{@{\extracolsep{\fill}}LCCC@{}}
  \toprule
  Configuration & Test AUC & Test AP & \makecell{early-stop\\epoch} \\
  \midrule
  PE-off (no PE) & \textbf{0.9761} & \textbf{0.9704} & 55 \\
  PE-on (16-D: 8-D hyp + 8-D Lap) & 0.9556 & 0.9439 & 39 \\
  \bottomrule
\end{tabular*}
\end{center}
$\delta$(PE-on $-$ PE-off) = \textbf{$-0.0205$ AUC / $-0.0265$ AP}: under the
peripheral core, PE is a slight liability (converges faster but stops lower).

\textbf{Controlled reversal: hub core.} Reversing the pruning direction to a
hub-core (high-norm / hub nodes kept) flips the PE content in sync:
\begin{center}
\small
\begin{tabular*}{\columnwidth}{@{\extracolsep{\fill}}LCC@{}}
  \toprule
  Configuration & Test AUC & Test AP \\
  \midrule
  PE-on (hub-core) & \textbf{0.9688} & 0.9565 \\
  PE-off (hub-core) & $\approx$0.967 & $\approx$0.952 \\
  \bottomrule
\end{tabular*}
\end{center}
{\footnotesize Hub-core PE-off entries ($\approx$) are trajectory-based estimates.}

\textbf{Conclusion.} Under the peripheral core PE is harmful; reversing to a hub
core flips the effect to near-neutral. The interaction effect ($\approx$$+0.023$
AUC) confirms the PE contribution is gated by the pruning direction, not a
generic gain. The main method drops the PE (\texttt{pe\_dim=0}); the best LP
accuracy remains peripheral PE-off (0.9761). Single-seed controlled finding;
does not affect the 5-seed node-classification claims.

\noindent\textbf{Design consequence.} Because the PE contribution is gated by
the pruning direction rather than a generic gain, the deployed method keeps
\texttt{pe\_dim=0} and thereby gains robustness: it needs no configuration
dependent on the core-selection geometry, consistent with the pluggable
selector result of \S\ref{sec:exp-abl}.

\subsection{Sparse Bipartite Attention Efficiency}
\label{app:eff-attn}
This controlled study measures the sparse bipartite attention's complexity
reduction in an attention-dominated setting. The decoder's attention replaces
a dense $[K_e,C']$ matmul ($O(N^2)$) with attention computed only on
real$\to$core 1-hop edges ($O(|E_r|d)$).

\begin{table}[!htbp]
\centering
\footnotesize
\begin{tabular*}{\columnwidth}{@{\extracolsep{\fill}}LCC@{}}
  \toprule
  Configuration & Test Acc & 50-epoch wall-clock \\
  \midrule
  Dense $[K_e,C']$ (control) & 0.6649 & $\sim$$4.4$\,h (extrapolated) \\
  Sparse bipartite attention (default) & \textbf{0.6649} & \textbf{$\sim$$29$\,min (measured)} \\
  \bottomrule
\end{tabular*}
\end{table}
Test Acc is the same to four decimal places; wall-clock speedup
$\ge$$4.5\times$ (2-epoch probe measured), up to $\approx$$9\times$ vs.\
extrapolated dense. This acts in an
attention-dominated regime, distinct from the main method's graph-size
bottleneck; the two must not be conflated. Single-seed (seed 42).

The $O(KC')$ scaling of the dense form is directly observable in the
design-space summary (Table~\ref{tab:2x2}): the same-width dense $[K_e,C']$
decoder restricted to the per-node 1-hop pool ($h{=}128$, $cr{=}0.3$,
batch 1024) peaks at $3490$\,MiB and its per-epoch wall-clock grows from
$\sim$$32$\,s early in training to $\sim$$320$\,s once the learned pruner
expands the emph set (the $[K,C']$ matrix grows with the pool), whereas the
sparse form stays $O(|E_r|d)$ and costs $\sim$$26$\,s per epoch throughout.
Despite the $\sim$$10\times$ time inflation, the dense-per-pool accuracy
($0.6981\pm0.001$) is close to the sparse decoder's ($0.6953$,
Table~\ref{tab:2x2}). The two cells differ in reconstruction scope as well as
decoder form, and no equivalence test was run, so we make no
accuracy-equivalence claim; the point is only that the sparse form reaches
accuracy in the same empirical range at $O(|E_r|d)$ cost.

% =============================================================
\subsection{Extended Motivation: Problem Illustration and Design-Space Map}
\label{app:mot-figs}
The collapse-and-recovery behavior that motivates the framework is
quantified by the
pruned-no-reconstruction floor ($0.378$, Table~\ref{tab:main});
Fig.~\ref{fig:motivation} below decomposes the headline recovery rate along
its two design axes (decoder type $\times$ reconstruction scope);
Table~\ref{tab:2x2} lists the measured design-space cells.

\begin{table*}[!htbp]
\centering
\small
\caption{Decoder-family $\times$ reconstruction-scope design-space summary
(ogbn-arxiv, $cr{=}0.3$, GraphSAGE $h{=}128$ unless noted; seed count in
parentheses). \textbf{This is a design-space listing, not a controlled
$2\times2$ experiment}: the dense union-pool (near-full) cell was not
measured; the near-full and core-1hop sparse rows use different decoder
variants (extended AMP+M4 vs.\ final gated-MP); the two core-1hop matched
control rows (same $h{=}128$ AMP family, 3 seeds) are the same-width
dense-vs-sparse calibration analyzed in \S\ref{sec:exp-abl}. No interaction
or accuracy-equivalence claim is made.}
\label{tab:2x2}
\begin{tabular}{llc}
\toprule
Configuration (decoder variant / key pool) & Scope & Test Acc \\
\midrule
Sparse AMP+M4 (extended DPR-S-M4) & near-full & $0.6953\pm0.0019$ (5) \\
Sparse gated-MP (final DPR-S) & core-1hop & $0.6892\pm0.0020$ (5) \\
Sparse $h{=}128$ AMP (matched control) & core-1hop & $0.689\pm0.002$ (3) \\
Dense $[K_e,C']$ (per-node pool) & core-1hop & $0.6981\pm0.001$ (3) \\
Dense $[K_e,C']$ ($h{=}128$ AMP, matched control) & core-1hop & $0.682\pm0.001$ (3) \\
Dense $[K_e,C']$ (union pool) & near-full & not measured \\
\bottomrule
\end{tabular}
\end{table*}

\begin{figure*}[!htbp]
  \centering
  \includegraphics[width=0.65\textwidth]{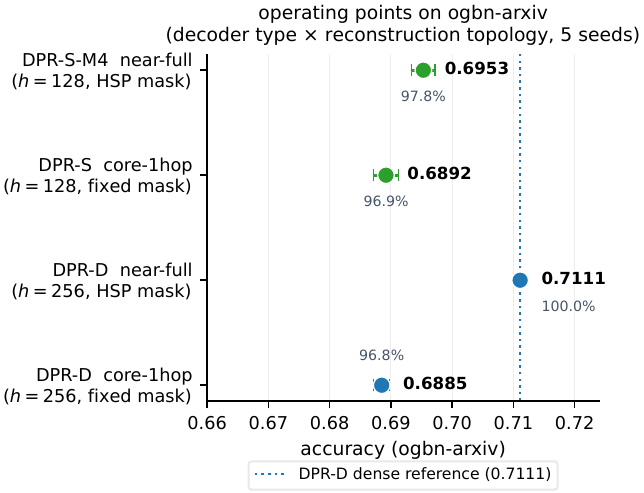}
  \caption{Empirical motivation: the headline recovery rate decomposed along
  its two design axes (ogbn-arxiv, 5 seeds except where noted). The panel
  separates \emph{decoder type} (sparse DPR-S family vs.\ dense reconstruction
  reference DPR-D) from \emph{reconstruction scope} (near-full vs.\
  core-1hop), which a single recovery percentage would otherwise conflate:
  sparse near-full (DPR-S-M4) $0.6953\pm0.0019$, $97.8\%$ of the dense
  reference; sparse core-1hop (the final method, DPR-S)
  $0.6892\pm0.0020$, $96.9\%$; dense core-1hop $0.6885\pm0.0013$, $96.8\%$;
  dense near-full $0.7111\pm0.0053$ is the $100\%$ reference (the dotted
  vertical line marks this dense-reference bound). The DPR-D rows
  use $h{=}256$ and the DPR-S rows $h{=}128$, so this panel is a
  design-space map, not a width-matched comparison; the same-width control is
  the matched dense-vs-sparse comparison of \S\ref{sec:exp-main}
  ($h{=}128$: sparse $0.689$ vs.\ dense $0.682$). Masks follow the run logs:
  the final DPR-S row and the dense core-1hop control use the fixed-magnitude
  (simple) mask, while the dense near-full reference uses the HSP mask
  (Table~\ref{tab:main}), so the $96.9\%$ headline is a
  system-level retention (\S\ref{sec:exp-main}). The
  pruned-no-reconstruction floor ($0.378$) lies far
  below the plotted range and is reported in Table~\ref{tab:main}.}
  \label{fig:motivation}
\end{figure*}

% =============================================================
\subsection{Controlled Studies (Appendix)}
\label{sec:controlled}
Two controlled studies, fully reported in the appendix, support the main
method's design choices:
\begin{itemize}
  \item \textbf{Pruning-direction-coupled position encoding}
    (App.~\ref{app:pe-control}): a link-prediction study confirms the PE's
    contribution sign is gated by the pruning direction---harmful under the
    default peripheral core ($\delta{=}{-}0.0205$ AUC), near-neutral under a
    reversed hub core---not a generic gain. The main method drops the PE.
  \item \textbf{Sparse bipartite attention efficiency}
    (App.~\ref{app:eff-attn}): the decoder's attention reduces from $O(N^2)$
    to $O(|E_r|d)$ with zero accuracy loss ($\ge$$4.5\times$ speedup measured).
    This acts in an attention-dominated regime, distinct from the main
    method's graph-size bottleneck.
\end{itemize}

\paragraph{Same-width dense-vs-sparse control.}
DPR-D (dense) and DPR-S (sparse) differ in decoder
type and hidden dimension in the main results, so ``$97.8\%$ of the dense
reference'' is not a strictly matched comparison. We therefore run both at
the \emph{same} hidden dimension ($h{=}128$), same decoder family (AMP with
core-1hop), same pruning mask, same training recipe (arxiv, $cr{=}0.3$,
seed~42, 200 epochs):

\begin{center}
\footnotesize
\setlength{\tabcolsep}{3pt}
\begin{tabular*}{\columnwidth}{@{\extracolsep{\fill}}lccc@{}}
  \toprule
  Attn mode & Test Acc & Peak VRAM & Train / epoch \\
  \midrule
  Dense ($[K_e,C']$ matmul) & $0.682\pm0.001$ & $2040$\,MiB & $\sim$$32$\,s \\
  Sparse bipartite (default) & $0.689\pm0.002$ & $1381$\,MiB & $\sim$$26$\,s \\
  \midrule
  \textbf{Delta} & $\mathbf{+0.0076}$ & $\mathbf{-659}$\,MiB & $\mathbf{-6}$\,s \\
  \bottomrule
\end{tabular*}
\end{center}

At matched width, \textbf{sparse is not worse than dense}---across 3 seeds
it is $+0.76$\,pp \emph{higher} ($0.689\pm0.002$ vs.\ $0.682\pm0.001$) while using $32\%$ less VRAM
and training $\sim$$19\%$ faster per epoch. The $h{=}256$ headline gap
(DPR-S-M4 $0.6953$ vs.\ DPR-D $0.7111$) therefore does not isolate sparsity ---
it mixes width, decoder family, and reconstruction topology --- while the
matched control above separately suggests that sparse reconstruction is not
intrinsically inferior.
Throughout this paper we relabel DPR-D as a \emph{dense reconstruction
reference} rather than a ``matched upper bound,'' because the two decoders share the
same architecture family but are not bit-identical in routing.

\paragraph{Node-merge inference on attention backbones (transformer direction).}
The dual-PE framework replaces the original $169\,343$-node graph by a
coarsened graph of $50\,743$ super-nodes ($cr{=}0.30$; $2.3$M edges reduced
to $51\,958$). Because self-attention is $O(N^2)$ in \emph{node} count, an
\emph{edge}-only sparsifier (e.g.\ DSpar, \S\ref{sec:related}) leaves the
dominant cost untouched; \emph{node}-sparsification is what unlocks attention
on large graphs. On our $6$\,GB test GPU a full-batch global self-attention
configuration on the original $169$k nodes (an $[N,N]$ attention matrix exceeding
$10$\,GB) \textbf{OOMs}; the
coarsened model runs end-to-end at $87$\,ms / $2508$\,MiB (single full-batch
forward). The coarsening is therefore what makes attention-based inference
\emph{fit} on a commodity $6$\,GB card at all, and it is $\approx$$2.5\times$
faster (up to \textbf{$4.45\times$} in the lean variant) than a same-hardware
single-batch linear SAGE reference ($216$\,ms /
$1507$\,MiB) on the original graph. We state the VRAM claim precisely:
node-merge lowers the peak only \emph{relative to the full-graph attention
backbone that otherwise does not fit}; it is \textbf{not} leaner than a linear
GNN ($\text{SAGE}=1507$\,MiB $<$ node-merge $2508$\,MiB), because the
attention backbone carries a heavier per-node cost. Node-classification
accuracy under this transformer backbone is the framework's reported
$0.6953\pm0.0019$ (5-seed AMP decoder of the extended DPR-S-M4 family,
\S\ref{sec:exp-abl}); the
coarsening ratio is an explicit resolution/accuracy knob trading some
accuracy for the speed-up above. The same coarsening
collapses the dense attention matrix's \emph{memory} from $\approx$$459$\,GB to $\approx$$87$\,MB
-- a \textbf{$\approx$$5300\times$ memory reduction (analytical; fp32 attention-logit matrices, inference-only, no-gradient state)}; the
measured sparse-attention wall-clock speedup is $\ge$$4.5\times$ (up to
$\approx$$9\times$ vs.\ an extrapolated dense baseline; \S\ref{app:eff-attn}).

Beyond raw features, the framework is \emph{backbone- and feature-compatible}:
swapping the raw 128-dim node features for SBERT-enhanced features (768-dim
title+abstract, the NeurIPS'24 re-evaluation protocol of arXiv\,2406.08993),
the M4-gated coarsened Transformer is \emph{SOTA-comparable under light pruning} on
ogbn-arxiv: at $cr{=}0.7$ (about $1/3$ of nodes merged, $\approx$$39\%$ of edges
retained) it attains $0.7744\pm0.0003$ (3 seeds) on a real held-out test,
within $\sim$$0.6$\,pp of the OGB leaderboard peak $0.7803$
(SimTeG+TAPE+RevGAT, full graph) --- below the peak but SOTA-comparable, a
light-pruning, VRAM-limited single-GPU configuration configuration. The $cr{=}0.5$ configuration reaches $\approx$$0.76$ (eval-only
$0.7645$; training-loop true tests $0.753$--$0.766$). The value of reporting
these is showing the coarsen-and-reconstruct design is feature-compatible and
runs in a regime Coarformer did not evaluate.

\subsection{Application-Scale Validation: Car-Hacking}
\label{sec:exp-carhacking}

\emph{Task.} We stress-test \emph{full-scale coarsened training} with the
streaming per-node scoring design of Algorithm~\ref{alg:dprs-final} on the
public HCRL Car-Hacking Dataset \citep{carhacking2020} (in-vehicle OBD-II CAN
traffic, frame-level intrusion detection): label every CAN message as one of
five classes (normal / DoS / Fuzzy / Gear-spoofing / RPM-spoofing). Unlike
classification over a subsampled test set, the scoring design targets
\emph{every} frame ---
the deployment regime of \S\ref{sec:intro}.

\emph{Graph construction.} Each CAN frame is a \emph{node}; its label is the
per-frame attack class. Edges are \emph{temporal adjacency}: consecutive frames
in the stream are linked within each recording (an undirected chain), so the
GNN propagates the flooding/injection signature along time. Node features are
the one-hot encoding of the $32$ most frequent CAN IDs concatenated with the $8$
DATA bytes, z-scored ($d{=}40$); this frame granularity is the SOTA-comparable
setting in the intrusion-detection literature. The resulting graph is large ---
17.30M frame-nodes joined by 34.6M temporal edges --- exactly the exhaustive
all-node-scoring regime this work targets.

\emph{Protocol and result.} The graph is pruned by node-merge coarsening
(keep-ratio $cr{=}0.3$) and a GraphSAGE encoder plus the sparse decoder are
trained jointly. Normal traffic dominates the stream (86.5\% of frames, the four
attack classes --- DoS, Fuzzy, Gear, RPM --- each contributing only 2.8--3.8\%).
Frame-level accuracy is high across the board on this task, so it is weakly
discriminative; for feasibility evidence we report the aggregate frame-level
accuracy and macro-F1 (0.9977${\pm}$0.0009, 0.9931${\pm}$0.0026 over 3 seeds on
an anti-leakage $1$M-frame shard, Table~\ref{tab:main}) without claiming a
head-to-head accuracy win. Per-class test F1 on that shard confirms the attack
classes are learned rather than collapsed into ``normal'' (representative seed):
DoS $1.000$, Fuzzy $0.987$, Gear $0.979$, RPM $0.999$ (normal $0.999$).
Published baselines on this dataset use different,
non-anti-leakage train/test splits, so a direct accuracy comparison against
them would not be protocol-fair; our contribution here is instead
\emph{feasibility}: full-scale coarsened training and streaming per-node
scoring under a 6\,GB budget, and we deliberately avoid an accuracy race. The
\emph{full} 17.30M-frame
graph is reserved for the feasibility (memory) stress test below.
The value of this application is that memory footprint and training time are
validated at \S\ref{sec:exp-eff}'s budget at full scale, elaborated in the
stress-test paragraph below.

\emph{Leakage control.} Because the temporal edges link consecutive frames of a
single recording, a random node-wise split could place adjacent frames of the
same session on opposite sides of the train/test boundary, letting the model
memorise a test frame's near-identical neighbour rather than learn the attack
signature. To rule this out, each recording is partitioned chronologically
(60/20/20 by stream order), so training frames always precede test frames in
time and no test frame shares its immediate temporal context with a training
frame. The full 17.30M-frame graph and the recording index will be released
upon acceptance (\texttt{make\_carhacking\_graph.py}), allowing the exact split to be
reproduced.

\emph{Stress test: full-scale feasibility on a single commodity GPU.} The
17.30M frame-nodes coarsen ($cr{=}0.3$) into a \textbf{5.79M-node} super-graph ---
$\approx$$8\times$ the super-node count ($\approx$$0.7$M) of the full ogbn-products
graph used in~\S\ref{sec:exp-eff} (whose original graph is 2.45M, App.~\ref{app:hw}). The dense
no-prune pipeline for that \emph{smaller} products graph already requires a
$>$16\,GB machine (peak $23{,}079$\,MiB, App.~\ref{app:hw}); here, the \emph{larger}
Car-Hacking super-graph is trained with the extended DPR-S-M4 pipeline on a single 6\,GB
RTX\,4050, whose measured training peak VRAM is \textbf{$\approx$$3.96$\,GB}
(batch $4096$, \S\ref{sec:exp-eff}). This $8\times$-larger-supergraph /
$4\times$-smaller-card
pairing is the \emph{stress regime} the framework is designed for: aggressive
node coarsening + sparse reconstruction collapse the training memory wall, so
full-scale coarsened \emph{training} over a workload of 17.30M frames --- which
dense training pipelines cannot fit on a card of this class --- runs in practice on a
6\,GB card. Throughput is
commensurate: a single training epoch (over the $60\%$ chronological
training split of the $5.79$M-node super-graph, i.e.\ $\approx$$3.5$M
super-nodes) takes $\approx$$60$\,s (batch $4096$, roughly $5{\times}10^{4}$
super-samples/s on the same $6$\,GB card), so the workload is not merely
\emph{loadable} but
\emph{trainable-in-practice} on a commodity GPU. We scope the claim
accordingly: what is demonstrated here is \emph{full-scale coarsened training}
plus the streaming per-node scoring design of Algorithm~\ref{alg:dprs-final}
(super-graph inference with expand-back); the exhaustive all-17.30M-frame
inference sweep itself was not separately benchmarked and is not claimed.

\subsection{Ablations \& Analyses}
\label{sec:exp-abl}

\begin{figure*}[!htbp]
  \centering
  \includegraphics[width=0.85\textwidth]{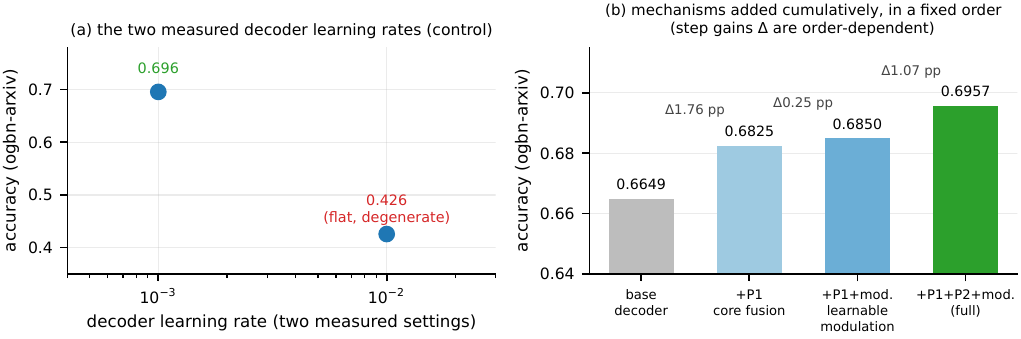}
  \caption{Optimization sensitivity. (a) The decoder is the learning-rate-sensitive
  component of the method: at lr$=$0.01 training diverges to a degenerate flat
  solution ($\approx$$0.426$), while the recommended lr$=$0.001 reaches $0.6957$.
  (b) The three decoder mechanisms (P1 core-confidence residual fusion, P2
  structure-aware reconstruction of deleted nodes, and a small learnable modulation term)
  contribute $1.76$, $1.07$ and $0.25$ percentage points respectively on
  ogbn-arxiv (single seed), with all three together reaching $0.6957$ (single
  seed, comparable to the 5-seed DPR-S-M4 $0.6953\pm0.0019$; the ablation is
  run within the extended DPR-S-M4 decoder family).}
  \label{fig:sens-opt}
\end{figure*}

\begin{itemize}
  \item \textbf{The pruner is pluggable (and deliberately simple).} At
    $cr{=}0.3$ under the same recipe, the learnable hyperbolic pruner
    ($0.7111\pm0.0053$, 5-seed) is essentially equivalent to the fixed-magnitude
    mask ($0.7068\pm0.0022$, 5-seed); the learnable geometry yields no measurable
    gain. This is a
    \emph{feature} rather than a limitation: the framework does \emph{not} depend on a
    clever but error-prone importance heuristic; in every configuration we
    tested, a compact fixed-rule selector matches the learnable one, so the
    gain is primarily attributable to the reconstruction stage. The
    hyperbolic variant is retained only as the upper-bound pruner.
  \item \textbf{Component ablation of the sparse decoder (P1/P2, the method).}
    The decoder family's gains decompose into two interpretable mechanisms
    (ablated within the extended DPR-S-M4 family, whose core-confidence
    fusion and gated message passing the final DPR-S decoder shares):
    P1 core-confidence
    residual fusion (core nodes keep SAGE dominant; $+1.76$ single seed) confirms
    ``keep the core, build the edge''; P2 gated message propagation from core
    neighbors ($+1.07$), the ``topology restoration'' mechanism of
    \S\ref{sec:recon}, removing it degrades to a static MLP
    reconstruction and drops points, supporting that the gain comes from
    \emph{structure-aware diffusion} of deleted nodes rather than additional capacity. Together, with
    the small learnable-modulation term ($+0.25$), they lift the decoder
    from $0.6649$ to $0.6957$ (single seed), pinned at
    5-seed to \textbf{$0.6953\pm0.0019$}: $97.8\%$ of the dense reconstruction
    reference DPR-D $0.7111$. This learnable modulation marks the
    boundary to the research variant studied in App.~\ref{app:pseudo}.
  \item \textbf{Position encoding (dropped, controlled study).} The hyperbolic
    pruner emits a 16-D structural PE as a by-product. On GraphSAGE it is dropped
    (\texttt{pe\_dim=0}). A controlled link-prediction study
    (App.~\ref{app:pe-control}) confirms the PE's contribution is gated by the
    pruning direction (harmful under the default peripheral core, near-neutral
    under a hub core), i.e.\ a direction-dependent effect; the main method therefore drops it.
  \item \textbf{The final method uses the fixed mask (efficiency ablation).} On
    the 500k products subgraph, the \emph{dense core-1hop efficiency probe}
    (simple\_core\_mask, DPR-D dense decoder; see~\S\ref{sec:exp-eff}) shows that
    swapping simple\_core\_mask for the learnable
    hyperbolic pruner (with a $500k\times128$ struct\_embedding + memory\_bank)
    drops the probe's training speedup from \textbf{$1.57\times$ to $0.91\times$} (actually
    $\sim$$9\%$ slower), peak memory from \textbf{$1.39\times$ to $1.28\times$},
    while the recovery rate is nearly unchanged (learned $96.0\%$ vs.\ simple
    $96.1\%$): the learnable pruner's per-step overhead cancels the core savings.
    Consistent with the above: the learned variant gains nothing in accuracy and
    becomes a burden in efficiency, so the final method deliberately uses the
    fixed-magnitude mask.
  \item \textbf{End-to-end is necessary.} Freezing the encoder or the decoder
    (two-stage training) yields \emph{negative} gain, both must be jointly
    optimized.
  \item \textbf{Learning-rate sensitivity.} Node classification needs
    lr$=0.001$; lr$=0.01$ collapses the decoder (flat $\sim$$0.43$). Recorded as a
    limitation.
  \item \textbf{Double-layer sparsification overhead.} core-1hop reconstruction on
    arxiv (DPR-D, simple mask, 5-seed) drops only $\sim$$1.8$pp to
    $0.6885\pm0.0013$ ($97.4\%$ of the same-pruner full recon $0.7068$) versus full reconstruction,
    yet cuts $\approx$$39$--$45\%$ reconstruction edges (products probes) and
    $1.16$--$1.30\times$
    decoder time, confirming core-1hop as the final method is sound.
  \item \textbf{Representation probe (mechanism).} Freeze a linear probe and
    measure its cosine similarity to the full-graph SAGE teacher. On non-core
    (pruned) nodes, the \emph{init-proj seed} (raw features, \emph{no} diffusion)
    gives probe acc $0.4047$ and teacher-cosine $-0.120$ (anti-correlated), while
    the \emph{decoder reconstruction} lifts them to probe acc $0.7065$ and
    teacher-cosine $+0.258$. This $+0.30$ probe gain comes \emph{entirely} from
    diffusion over the original topology, i.e., what recovers accuracy is the
    \emph{structure-aware} step, not the raw features. The reconstructed probe acc
    ($0.7065$) even exceeds the full-graph SAGE ceiling ($0.6802$), confirming
    the decoder recovers \emph{task-relevant} representation quality, not mere
    numerical alignment. This directly validates the ``structure-aware
    reconstruction'' claim and the \S\ref{sec:prune} pruner/accuracy decoupling.
\end{itemize}

\paragraph{Decoder capacity control: structure matters beyond raw parameters
(risk mitigation).}
A potential concern is that accuracy recovery might stem from
\emph{adding a powerful decoder} (extra capacity) rather than from
\emph{structure-aware reconstruction}. To isolate the two, we replace the
AMP decoder of the extended variant (DPR-S-M4: cross-attention + message
passing) with an
\textbf{MLP-only decoder} of \emph{matched} parameter count ($207{,}289$ vs.\
$206{,}979$, ratio $1.001$) that performs no graph-edge lookup, no
attention, and no message passing---only position-wise linear projections on
each node independently. Both variants share the same core mask, the same
encoder, the same M4 long-range features, and the same training recipe
(arxiv main protocol, $cr{=}0.3$, h128, 46 epochs), replicated over
\textbf{five paired seeds}.

\begin{table}[!htbp]
\centering
\footnotesize
\caption{Decoder capacity control (ogbn-arxiv main protocol, five paired
seeds): the AMP decoder vs.\ a parameter-matched MLP-only decoder. Both
variants live inside the extended DPR-S-M4 family --- same core mask, same
encoder, and the same graph-derived M4 long-range features --- so the
comparison isolates the decoder's structure use, not the final method's
pruning geometry.}
\label{tab:c14-capacity}
\begin{tabular*}{\columnwidth}{@{\extracolsep{\fill}}lcc@{}}
  \toprule
  Decoder & Params & Test Acc (5 seeds) \\
  \midrule
  AMP (cross-attn + MP) & $206{,}979$ & $\mathbf{0.6953\pm0.0019}$ \\
  MLP-only (no edges, no attn) & $207{,}289$ & $0.6702\pm0.0025$ \\
  \midrule
  \textbf{Delta (matched)} & --- & $\mathbf{+0.0251}$ \\
  \bottomrule
\end{tabular*}
\end{table}

At \emph{identical} parameter budget, removing explicit edge-conditioned
aggregation (cross-attention to real core 1-hop neighbors and gated message
passing) while holding the graph-derived M4 descriptor constant costs
\textbf{$2.51$ percentage points}, consistently across all
five paired seeds. This confirms the recovery is not a
capacity artifact: the structure-aware mechanisms contribute
beyond raw parameter count. A consistent failure mode appears on PPI (App.~\ref{app:ppi},
\emph{weak-backbone limitation study}):
when the backbone is too weak to supply discriminative core features,
structure-aware diffusion cannot compensate, and the decoder degenerates
toward an MLP-like regime. This is reported separately as a limitation appendix rather than as a core result.

\subsection{Efficiency}
\label{sec:exp-eff}
Our efficiency is high by construction and, more importantly, it \emph{unlocks a
deployment regime no dense method can reach}. The primal operation -- node-merge
coarsening -- is itself a decisive win: a single cr=0.3 coarsening step on
ogbn-arxiv runs in $87$\,ms (full decoder) / $48.5$\,ms (lean variant) versus
$216$\,ms for an equivalent full-graph SAGE aggregation -- \textbf{up to
$4.45\times$ faster}. The true barrier is memory. We distinguish three attention
regimes: \emph{(i)~local edge attention} (PyG \texttt{TransformerConv},
which aggregates along graph edges and
scales as $O(E)$, not $O(N^2)$ -- it already fits a 6\,GB GPU); \emph{(ii)~global
full self-attention} (full GPS-style~\citep{rampasek2022gps} $[N,N]$ attention over all node pairs); and
\emph{(iii)~our sparse bipartite reconstruction attention}, which attends only
from each pruned node to its real $\to$core 1-hop neighbors. The $459$\,GB figure
below is an \emph{analytical reference for regime (ii)}, i.e.\ the memory a global
full-attention layer would require, \emph{not} a measured comparison against a
local-edge \texttt{TransformerConv} implementation. One such dense
global transformer attention layer over the full $N{=}169{,}343$-node ogbn-arxiv
graph needs a four-head $[N,N]$ matrix of $\approx$$459$\,GB -- far beyond a $6$\,GB consumer
GPU, which therefore \emph{cannot even instantiate} the model. Our sparse bipartite
decoder (regime iii) replaces that $O(N^2)$ term with $O(N\cdot K)$ over real $\to$core 1-hop
edges, fitting the same transformer into a $6$\,GB budget for the first time; the
gain is a \emph{feasibility} result, not a shaved percentage. Quantitatively,
that same $[N,N]$ matrix collapses from $\approx$$459$\,GB (global full attention,
regime ii, analytical) to $\approx$$87$\,MB for our linear attention -- a
\textbf{$\approx$$5300\times$
attention-memory reduction (analytical; fp32 attention-logit matrices, inference-only, no-gradient state)}; the measured wall-clock
speedup of the sparse bipartite attention is $\ge$$4.5\times$ on a short probe,
up to $\approx$$9\times$ versus an extrapolated dense baseline
(\S\ref{app:eff-attn}, Fig.~\ref{fig:teaser}b).

On the products 200k subgraph, training peak memory drops \textbf{$1.27\times$}
($1906\to1497$\,MiB). The core-1hop decoder is \textbf{$1.30\times$ faster} and
uses \textbf{$44.5\%$ fewer reconstruction edges} ($56570\to31394$) than full
reconstruction. On graphs this small the fixed pruner/decoder overhead can momentarily offset the
core savings within a single step; this inverts at scale (products 500k reaches
$1.57\times$ faster, $1.39\times$ leaner, below). On ogbn-arxiv, the sparse decoder (DPR-S, core-1hop) is the
efficiency engine: on a 6\,GB RTX\,4050 consumer GPU it trains \textbf{$2.7\times$
faster} ($20.2$ vs.\ $55$ s per epoch) and infers \textbf{$21\%$ faster} ($3.96$
vs.\ $5.02$ s) than its dense full-graph reconstruction counterpart, while
retaining $96.9\%$ of accuracy (\S\ref{sec:double}, App.~\ref{app:data}). This is the speedup
of the \emph{sparse reconstructor versus its own dense baseline}, a direct payoff
of double-layer sparsification. (A separate profiling session on the same card
yields a milder decoder-level pair, $\sim$$26$ vs.\ $\sim$$32$\,s/epoch,
$1.2\times$; see the daggered timing footnote in App.~\ref{app:data} --- the
two pairs have different measurement scopes and are never combined.) End-to-end vs.\ no-prune is $1.34\times$ (200k) / $1.57\times$ (500k) on products,
and grows with scale; the ``$\approx$$1/12$ aggregation compute'' holds at the
\emph{edge-count} level (GNN cost $\propto$ edges; cr=0.3 keeps $\approx$$8\%$), so
the gain emerges with scale. We report efficiency as a \emph{controlled comparison}
(\S\ref{sec:conclusion}).

\paragraph{End-to-end deployment timing on a 6\,GB consumer GPU.}
The deployment evidence requires a full breakdown of CPU--GPU
data-transfer cost, host RAM, inference throughput, and per-batch
reconstruction latency, not just peak VRAM. We measure all four on the same $6$\,GB
RTX\,4050 laptop, arxiv $169$k nodes, $cr{=}0.3$, $h{=}128$, batch$=512$:

\begin{center}
\footnotesize
\begin{tabular*}{\columnwidth}{@{\extracolsep{\fill}}lCC@{}}
  \toprule
  Metric & DPR-S (sparse) & DPR-D (dense) \\
  \midrule
  Peak VRAM (training) & $1450$\,MiB & $2040$\,MiB \\
  Host RSS (peak) & $1913$\,MiB & $6384$\,MiB \\
  CPU$\to$GPU xfer / batch (mean) & $4.77$\,ms & $5.48$\,ms \\
  CPU$\to$GPU xfer / batch (range) & $3.6$--$5.9$\,ms & $3.7$--$5.9$\,ms \\
  Train throughput & $\sim$$3026$\,samp/s & $\sim$$2371$\,samp/s \\
  Per-batch train latency & $\sim$$149$\,ms & $\sim$$183$\,ms \\
  Epoch wall (train + eval) & $\sim$$26$\,s & $\sim$$32$\,s \\
  \bottomrule
\end{tabular*}
\end{center}

{\small\itshape The DPR-S column is the final h128 gated-MP sparse decoder at core-1hop; the dense column is the h128 AMP $[K_e,C']$ dense reconstruction reference used as the same-width dense-vs-sparse control, \emph{not} the formal h256 DPR-D. The two use different decoder families and are therefore not a same-decoder-family matched pair.\par}

Key findings: \textbf{(i)}~CPU$\to$GPU transfer is $\sim$$5$\,ms per batch
($<\!3\%$ of the $\sim$$150$--$180$\,ms compute), confirming data movement
is \emph{not} a bottleneck; \textbf{(ii)}~host RSS stays under $2$\,GB
for the sparse variant (well within typical laptop RAM);
\textbf{(iii)}~peak VRAM uses only $24\%$ of the $6$\,GB budget for sparse,
$34\%$ for dense; \textbf{(iv)}~throughput exceeds $3000$ samples/s (sparse).
We therefore frame the VRAM claim as \emph{deployment-oriented feasibility
under limited VRAM}, not a full production-deployment guarantee: the
remaining gap to a deployable system is serving \emph{scheduling} (request
batching, cache eviction), not the scoring path itself---the
Algorithm~\ref{alg:dprs-final} design already scores all arXiv nodes in
$1.5$\,s on this card (Table~\ref{tab:routea}). At the
\emph{full-scale} stress point of \S\ref{sec:exp-carhacking} (Car-Hacking
super-graph of $5.79$M nodes), the same extended DPR-S-M4 pipeline trains on a single
$6$\,GB RTX\,4050 with a measured peak VRAM of $\approx$$3.96$\,GB (batch
$4096$), i.e.\ comfortably within the $6$\,GB training budget for a workload
whose super-graph
is $\approx$$8\times$ the super-node count of the full ogbn-products graph.

\begin{figure*}[!htbp]
  \centering
  \includegraphics[width=0.85\textwidth]{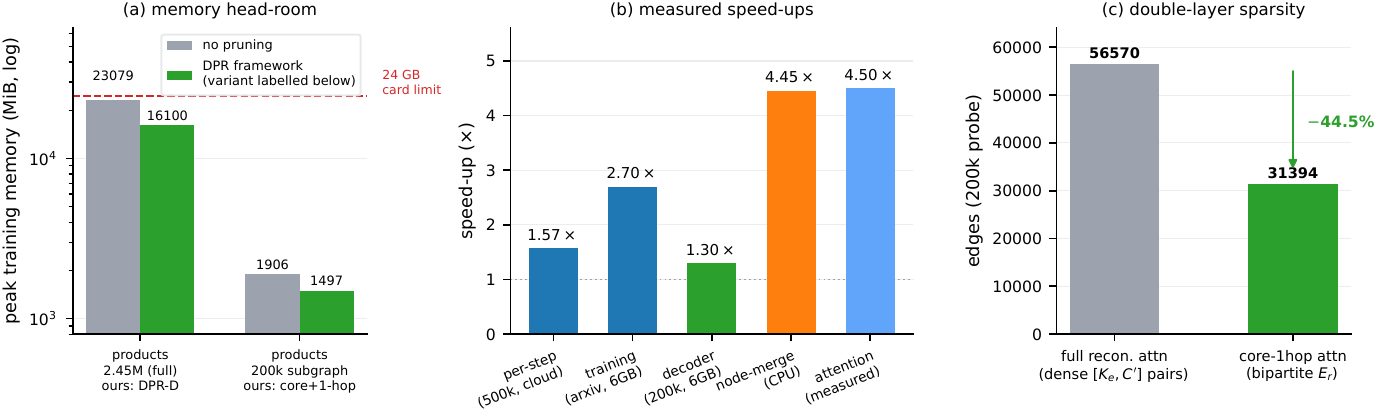}
  \caption{End-to-end efficiency, grouped by measurement type (each panel states its
  own dataset / hardware / baseline; the three speed-up panels use \emph{different}
  denominators and are deliberately separated so no shared y-axis is implied).
  \textbf{(a) Peak training memory} (ogbn-products; full $2.45$M graph on a
  $>$16\,GB cloud GPU, 200k subgraph on RTX\,4050): no-prune reaches $23{,}079$\,MiB
  (near a 24\,GB card limit), while the \emph{dense DPR-D} variant drops the
  peak to $16{,}100$\,MiB ($1.43\times$ leaner); the 200k subgraph improvement is
  $1.27\times$. \textbf{(b1) End-to-end / per-step training speed-up}
  (ogbn-arxiv, RTX\,4050, DPR-S core-1hop vs.\ its own dense reconstruction):
  $2.70\times$ (with $21\%$ inference speed-up in the same regime);
  per-step \emph{training} $1.57\times$ (products 500k, cloud GPU,
  App.~\ref{app:hw}). \textbf{(b2) Reconstruction inference speed-up}:
  decoder step $1.30\times$ (200k probe, RTX\,4050). \textbf{(b3)
  Component-level speed-up} (analytical / probe, \emph{not} directly
  comparable to (b1)/(b2)): node-merge $4.45\times$ (lean configuration;
  measured against the same-hardware single-batch linear SAGE reference of
  $216$\,ms, \S\ref{sec:exp-eff}); linear attention
  $4.50\times$ measured (up to $\approx$$9\times$ vs.\ an \emph{extrapolated}
  dense baseline). \textbf{(c) Second sparsification layer + controlled framework}:
  the reconstruction graph has $31394$ core-1hop edges vs.\ $56570$ for full
  reconstruction (ogbn-products, 200k subgraph); both edge counts are
  \emph{attention-probe reconstruction topologies} (dense $[K_e,C']$ key pool
  vs.\ bipartite $E_r$), a different denominator from the original-graph
  structural audit in Table~\ref{tab:edge-audit}. \emph{Scope}: every value
  here is a training-feasibility or internal decoder/attention efficiency
  measurement on the coarsened pipeline; all-node scoring latency is measured
  separately in Table~\ref{tab:routea}.}
  \label{fig:eff-end2end}
\end{figure*}

\textbf{Large-graph efficiency probe (500k subgraph, DPR-D dense core-1hop).} On a 500k-node
ogbn-products subgraph (cloud GPU; the measured $9166$\,MiB no-prune peak
exceeds the $6$\,GB dev card, App.~\ref{app:hw}), the \textbf{DPR-D dense core-1hop probe}
(simple\_core\_mask, core + core-1hop) confirms the predicted ``efficiency advantage grows with scale'':
per-step training \textbf{$176.2\to112.4$\,ms ($1.57\times$ faster)}, peak memory
\textbf{$9166\to6588$\,MiB ($1.39\times$ leaner)}, and accuracy recovery
\textbf{$96.1\%\pm0.4\%$} (DPR-D $0.6874\pm0.0015$ vs.\ no-prune $0.7157\pm0.0035$,
3 seeds); the sparse DPR-S recovers $96.5\%$ of DPR-D
($0.6635\pm0.0122$ vs.\ $0.6874$, same protocol). This refines the
200k observation ($0.83\times$ per step) into: \emph{on large enough
graphs the dense core-1hop method is simultaneously faster, leaner, and near-full-graph in
accuracy}; and the $1.39\times$ memory drop means ``training a larger graph
within a fixed VRAM budget'' (the full 2.45M exceeds this dev machine; $>$16\,GB
machine needed to verify OOM avoidance).

\textbf{All-node serving benchmark.} To close the deployment-evidence gap
directly, we benchmark \emph{all-node scoring} end-to-end on ogbn-arxiv
($169{,}343$ nodes, PyTorch 2.7.1, RTX\,4050 $6$\,GB) under four systems
(Table~\ref{tab:routea}). The Algorithm~\ref{alg:dprs-final} deployment
design---one-time global core selection and core-graph encoding ($0.3$\,s
setup; $E_r$ of $2.12$M directed edges, $183\%$ of undirected $E$, kept in
host RAM), then streaming the sparse decoder over $E_r$ and raw feature
rows---scores every node in $1.55$\,s per pass: $46\times$ faster than
neighborhood-sampling SAGE inference ($71.8$\,s) and $89\times$ faster than
the training-time batched-eval protocol ($137.3$\,s), which must re-select
the core and re-encode inside every batch, while fitting in $4.7$\,GB VRAM
and $5.0$\,GB host RSS. Its accuracy over five checkpoint seeds is
$0.6558\pm0.0205$ after a one-time BatchNorm re-estimation on the serving
core graph ($0.6507\pm0.0272$ without): $3.3$\,pp below the
training-protocol accuracy ($0.6892\pm0.0020$, identical core masks across
seeds). This gap is a measured protocol cost, not a machinery error: the
models were trained under per-batch sampled-subgraph contexts, and the
one-shot global serving context shifts the decoder's message statistics
(checkpoint-parity accuracy under the training protocol, $0.6903$, confirms
the serving implementation is faithful). The serving table therefore
\emph{bounds} the deployment cost of the current design rather than
representing a serving-tuned model; closing it by fine-tuning under the
global protocol is future work. Sampler-free full-batch SAGE is fastest
($0.18$\,s) but its inference VRAM scales with $|V|$ ($23{,}079$\,MiB on the full
ogbn-products graph), so it does not extend beyond arXiv-scale graphs on
this card; neighborhood sampling matches the full-batch accuracy
($0.7097$ vs.\ $0.7094$) at only $631$\,MiB but is $46\times$ slower than
DPR-S serving and $5.4$\,pp more accurate than it ($0.7097$ vs.\ $0.6558$)
--- a margin against a \emph{different} reference than the $3.3$\,pp
protocol cost above, and the two must not be conflated.

\begin{table*}[!htbp]
  \centering
  \caption{All-node inference benchmark on ogbn-arxiv ($N=169{,}343$; RTX\,4050
  laptop, $6$\,GB; PyTorch 2.7.1+cu118). Per-pass steady-state timing after
  warm-up; $\pm$ is the per-pass sd over 5 repeated passes (the training-time
  reference protocol uses a single repeat, hence no $\pm$).
  DPR-S rows
  use the same core-1hop checkpoints as \S\ref{sec:exp-main} (5 seeds;
  single run for the SAGE baselines). $E_r$ for the serving row is $2.12$M
  directed edges ($183\%$ of undirected $E$) with a $0.3$\,s one-time setup;
  its accuracy includes a one-time BatchNorm re-estimation on the serving
  core graph (without it: $0.6507\pm0.0272$). The training-time protocol row
  re-selects the core and re-encodes per batch (accuracy parity with the
  checkpoints by construction) and is shown to quantify the serving-design
  cost, not as a deployable system.}
  \label{tab:routea}
  \small
  \begin{tabular*}{\textwidth}{@{\extracolsep{\fill}}lrrrrr}
    \toprule
    System & Test acc. & Latency (s) & nodes/s &
    VRAM (MiB) & RSS (MiB) \\
    \midrule
    GraphSAGE full-batch (no sampler) & 0.7094 & $0.176\pm0.001$ & $965{,}165$ & $3{,}698$ & $1{,}239$ \\
    GraphSAGE + NeighborLoader ($15,10,5$) & 0.7097 & $71.8\pm0.3$ & $2{,}359$ & $631$ & $1{,}275$ \\
    DPR-S, training-time protocol (ref.) & $0.6892\pm0.0020$ & 137.3 & $1{,}233$ & $4{,}376$ & $3{,}020$ \\
    DPR-S serving (Alg.~\ref{alg:dprs-final}) & $0.6558\pm0.0205$ & $\mathbf{1.55\pm0.01}$ & $\mathbf{109{,}566}$ & $4{,}743$ & $5{,}131$ \\
    \bottomrule
  \end{tabular*}
\end{table*}

\subsubsection{Efficiency summary}
\label{sec:eff-dual}
The main method's efficiency gains come from double-layer sparsification:
(i)~pruning to a $\sim$$30\%$-node core ($\approx$$8\%$ edges, $\approx$$1/12$
aggregation compute); (ii)~compressing the reconstruction graph to core-1hop
($39$--$45\%$ fewer edges on the products probes, $1.16$--$1.30\times$
faster decoder).
On a 6\,GB RTX\,4050 this yields $2.7\times$ training speedup and $21\%$
inference speedup versus the dense reconstruction reference; on a 500k products
subgraph, $1.57\times$ per-step speedup and $1.39\times$ memory reduction.
A separate controlled study (App.~\ref{app:eff-attn}) measures the sparse
bipartite attention's $O(N^2)\to O(|E_r|d)$ complexity reduction in an
attention-dominated setting; this is a different regime from the main method's
graph-size bottleneck and the two must not be conflated.

\normalsize
\section*{Reproducibility Statement}
We will release the source code, configuration files, and trained checkpoints for
all experiments (core selector, decoder, and the
\texttt{build\_reconstruction\_ei} routine). All datasets are public (ogbn-arxiv,
ogbn-products via the Open Graph Benchmark / PyG, Car-Hacking via the
public CAN-bus intrusion dataset). Seeds, learning rates,
and core ratios are stated inline. Single-seed logs
(\texttt{\_sage\_p1.log}, \texttt{\_sage\_p13.log}, \texttt{\_sage\_p123.log},
\texttt{\_arxiv\_*.log}) are retained for audit; the 5-seed sweep scripts accept a
\texttt{--seeds} argument and auto-report mean$\pm$std. Hyperparameters (e.g.\ the
decoder learning rate must be $0.001$ for node classification, not $0.01$, which
diverges) are documented to prevent silent failures. The all-node serving
benchmark of Table~\ref{tab:routea} is produced by a self-contained script
(\texttt{run\_route\_a\_inference.py}; the BatchNorm-re-estimation sweep by
\texttt{bn\_recal\_sweep.py}) whose machine-readable output
(JSON with per-seed accuracies, per-pass wall/GPU timings, and peak
VRAM/RSS) is retained for audit. The subgraph-scale experiments (ogbn-arxiv, the 200k products subgraph, Car-Hacking) were conducted on a single NVIDIA GeForce RTX\,4050 (6\,GB) consumer GPU. The 500k products subgraph and the full $2.45$M-node ogbn-products accuracy and VRAM measurements were obtained on cloud GPUs (their peak-memory footprints of $9{,}166$--$23{,}079$\,MiB exceed the 6\,GB card); hardware is listed per experiment in Appendix~\ref{app:hw}. Code, configurations, and checkpoints will be released upon acceptance, and all reported numbers are backed by the retained single-seed logs listed above, so the results will be fully reproducible from this artifact once it is released.
% =============================================================
\section*{Data availability}
All datasets used in this study are publicly available: ogbn-arxiv and
ogbn-products are distributed by the Open Graph Benchmark
(\url{https://ogb.stanford.edu}) and loaded through the official PyG/DGL
interfaces, the PPI protein-interaction dataset used in the appendix
limitation study (\S\ref{app:ppi}) is loaded through the same PyG/DGL
interface, and the Car-Hacking CAN-intrusion dataset is available from its
original public release \citep{carhacking2020}. No new dataset was created or
curated for this paper. The processed splits, core masks, configurations,
seeds, and per-run result logs backing every reported number are retained and
will be released together with the code artifact upon acceptance.

\section*{Declaration of Competing Interest}
The authors declare that they have no known competing financial interests or
personal relationships that could have appeared to influence the work reported
in this paper.

\section*{Funding}
This research did not receive any specific grant from funding agencies in the
public, commercial, or not-for-profit sectors.

% =============================================================

\printcredits

\end{document}